\documentclass[10pt]{article} 
\usepackage[accepted]{tmlr}

\renewcommand{\headrulewidth}{2pt}   %

\newcommand\stackequal[2]{%
  \mathrel{\stackunder[2pt]{\stackon[4pt]{=}{$\scriptscriptstyle#1$}}{%
  $\scriptscriptstyle#2$}}}

\usepackage{stackengine}
\newcommand\ledot{\mathrel{\ensurestackMath{%
  \stackengine{-.5ex}{\lessdot}{-}{U}{c}{F}{F}{S}}}}

\usepackage{xcolor}
\usepackage{adjustbox}
\usepackage{tikz}
\usetikzlibrary{arrows.meta,
                chains,
                positioning}

\usepackage{hyperref} 
\usepackage{url}
\usepackage{amsmath}
\usepackage{cleveref}
\usepackage{mathrsfs}
\usepackage{algpseudocode}
\usepackage[inline]{enumitem}
\usepackage{comment}
\usepackage{cancel}
\usepackage{algorithm}
\usepackage{amsthm}
\usepackage{thm-restate}
\usepackage{amssymb}
\usepackage{bm,bbm}

\makeatletter
\newlength{\trianglerightwidth}
\algnewcommand{\LineComment}[1]{\Statex \hskip\ALG@thistlm $\triangleright$ #1}
\algnewcommand{\LineCommentCont}[1]{\Statex \hskip\ALG@thistlm%
  \parbox[t]{\dimexpr\linewidth-\ALG@thistlm}{\hangindent=\algorithmicindent \hangafter=1 \strut\makebox[\algorithmicindent][l]{$\triangleright$}#1\strut}}
\makeatother

\usepackage{calrsfs}
\DeclareMathAlphabet{\pazocal}{OMS}{zplm}{m}{n}
\newcommand{\norm}[1]{\left\lVert#1\right\rVert}

\DeclareMathOperator*{\argmax}{arg\,max}
\DeclareMathOperator*{\argmin}{arg\,min}
\DeclareMathOperator{\Tr}{Tr}

\DeclareMathOperator{\Diag}{Diag}

\usepackage{booktabs}
\usepackage{multirow}
\usepackage{graphicx}
\usepackage{subfig}
\usepackage{makecell}

\DeclareSymbolFont{myletters}{OML}{ztmcm}{m}{it}
\DeclareMathSymbol{\uplambda}{\mathord}{myletters}{"15}

\algnewcommand{\ElseIIf}[1]{\algorithmicelse\ #1}
\algnewcommand{\IIf}[1]{\State\algorithmicif\ #1\ \algorithmicthen}
\algnewcommand{\EndIIf}{\unskip\ \algorithmicend\ \algorithmicif}

\makeatletter

\usepackage{etoolbox}
\newcommand{\StatePar}[1]{%
  \State\parbox[t]{\dimexpr\linewidth-\ALG@thistlm}{\strut #1\strut}%
}
\makeatother

\title{\textsc{AdaCubic}: An Adaptive Cubic Regularization Optimizer for Deep Learning}

\affilone{Aristotle University of Thessaloniki, Greece \\ }
\affiltwo{University of Applied Sciences Northwestern Switzerland, Muttenz, Switzerland}
\correspondence{tsingalis@csd.auth.gr}

\author{
Ioannis Tsingalis$^{1}$, Constantine Kotropoulos$^{1}$, and Corentin Briat$^{2}$}

\def\month{02}  
\def\year{2026} 
\def\openreview{\url{https://openreview.net/forum?id=pZBQ7J37lk}} 

\definecolor{darkblue}{rgb}{0.0,0.0,0.55}
\usepackage[strict]{changepage}

\usepackage{framed}

\definecolor{formalshade}{rgb}{0.95,0.95,1}

\begin{document}

\maketitle

\begin{abstract}
A novel regularization technique, \textsc{AdaCubic}, is proposed that adapts the weight of the cubic term. The heart of \textsc{AdaCubic} is an auxiliary optimization problem with cubic constraints that dynamically adjusts the weight of the cubic term in Newton's cubic regularized method. We use Hutchinson’s method to approximate the Hessian matrix, thereby reducing computational cost. We demonstrate that \textsc{AdaCubic} inherits the cubically regularized Newton method's local convergence guarantees. Our experiments in Computer Vision, Natural Language Processing, and Signal Processing tasks demonstrate that \textsc{AdaCubic} outperforms or competes with several widely used optimizers. Unlike other adaptive algorithms that require hyperparameter fine-tuning, \textsc{AdaCubic} is evaluated with a fixed set of hyperparameters, rendering it a highly attractive optimizer in settings where fine-tuning is infeasible. This makes \textsc{AdaCubic} an attractive option for researchers and practitioners alike. To our knowledge, \textsc{AdaCubic} is the first optimizer to leverage cubic regularization in scalable deep learning applications.

\vspace{0.5em}
\begin{center}
\centering
\texttt{\url{https://github.com/iTsingalis/AdaCubic}}
\end{center}

\end{abstract}

\section{Introduction}\label{sec:intoduction}

Deep Neural Networks (DNNs) have demonstrated strong performance across a variety of machine learning tasks~\citep{pouyanfar2018,dargan2020}. DNN models are non-convex~\citep{jin2021,danilova2022,pooladzandi2022}. Accordingly, saddle points may arise during the optimization procedure~\citep{bedi2021escaping}. In~\citet{dauphin14}, it is shown that the saddle points affect the efficiency of a DNN. Therefore, methods that avoid saddle points are necessary, as discussed next.

The Cubic Regularized (CR) Newton's method was introduced in~\citep{nesterov2006}. This method effectively circumvents saddle points in a non-convex setting. The first research direction focuses on carefully selecting the regularization parameter for the cubic term. In~\citet{cartis2011}, an Adaptive Regularized Cubic (ARC) method is presented where the cubic regularization term is adapted dynamically, similarly to the radius in the Trust Region methods~\citep{conn2000}. To mitigate the computational burden of deriving the Hessian matrix and the gradient in ARC,~\citet{carmon2019} solves the CR sub-problem using gradient descent. Alternatively, one solves the cubic sub-problem using a subsampled gradient and a Hessian-vector product~\citep{tripuraneni2018}. In~\citet{kohler2017}, a subsampled scheme for the gradient and the Hessian matrix is exploited, achieving the same convergence rate as ARC. In~\citet{wang2020}, momentum information is utilized to improve the convergence rate of CR. Inspired by~\citet{fang2018}, a recursive stochastic variance reduced CR method is proposed in~\citet{zhou2020}, yielding a better convergence rate than that reported in~\citep{tripuraneni2018}. In~\citet{huang2022}, the CR method was applied to solve unconstrained convex-concave saddle point problems. 

In a second research direction, it has been demonstrated that injecting a random perturbation whenever a saddle point is encountered can facilitate escape from saddle points. In~\citet{ge2015,jin2017}, both negative curvature and random perturbation are applied to Stochastic Gradient Descent (\textsc{SGD}) to escape saddle points. Within the same scope, in~\citet{allenzhu2018, royer2018}, it is shown that negative curvature and random perturbation can be used to find an $(\epsilon_g, \epsilon_H)$-stationary point faster than the first-order methods. A drawback of these methods is the need to compute the smallest eigenvalue of the Hessian matrix and the corresponding eigenvector. Several methods have been proposed to address this limitation. In~\citet{li2019}, it is shown that a perturbed version of Stochastic Recursive Gradient Descent, without using the Hessian matrix information, also converges to an $(\epsilon_g, \epsilon_H)$-stationary point. In~\citet{allen2018,zhang2021}, a robust Hessian matrix power method is proposed to compute the negative curvature near saddle points, yielding faster convergence than the standard perturbed gradient descent methods. In~\citet{chen2022}, to achieve a better convergence rate, the average movement of the iterates is controlled by a step-size shrinkage scheme~\citep{li2019}. 

In a third research direction, escaping saddle points relies on momentum information~\citep{wang2021distributed}. First-order methods with random initialization and momentum information are shown to be able to escape saddle points in~\citep{sun2019}. A greater momentum in \textsc{SGD} enlarges the projection to an escape direction, leading to a fast saddle point escape~\citep{wang2020a}. In~\citet{levy2021}, a parameter-free recursive momentum method is proposed for non-convex optimization. In~\citet{wang2020a}, it is shown that acceleration can be achieved for non-quadratic functions under Polyak-{\L}ojasiewicz condition and non-convexity. In~\citet{wang2021}, it is presented that the momentum term accelerates the training of a one-layer-wide ReLU network.

A fourth research direction employs variance reduction to escape saddle points. In~\citet{allen2016}, the minimization of the sum of smooth functions is studied, where variance reduction is applied to speed up convergence in both the stochastic and the deterministic case. A general variance-reduction estimation method is introduced that is not restricted to gradients~\citep{fang2018}. This method has been applied to numerous problems and has achieved convergence rates superior to those reported in~\citep{allen2016}. In~\citet{nguyen2017a}, a recursive gradient estimator for convex optimization is introduced. The latter estimator is then extended to non-convex problems in~\citep{nguyen2017b}. In~\citet{ge2019}, the first variance reduction technique not based on a separate negative curvature search subroutine is proposed. 

Last but not least, a second-order optimizer, called \textsc{AdaHessian}, has been introduced in~\citep{ya2021}. \textsc{AdaHessian} is based on the Adaptive Moments Estimation (\textsc{Adam}) optimizer~\citep{kingma2014} and leverages Hutchinson's method to approximate the curvature information with low computational cost~\citep{bekas2007}. The convergence rate of \textsc{AdaHessian} for a strongly convex and smooth loss function can be found in~\citep{ya2021,pooladzandi}. Based on the \textit{Hessian power}, the convergence rate of \textsc{AdaHessian} for a strongly convex and smooth loss function matches that of either gradient descent or Newton's method~\citep{jahani2021,sadiev2022}.

In this paper, we focus on the CR Newton method~\citep{nesterov2006} and propose a novel algorithm that dynamically adapts the weight of the cubic term in the cubic subproblem. 
The adaptation of the cubic term is achieved by utilizing an auxiliary cubically constrained optimization problem. The proposed algorithm, \textsc{AdaCubic}, leverages the advantages of CR theory and Hutchinson’s estimation technique. In more detail, the \textbf{contributions} of this paper are:
\begin{itemize}
\item A novel method is proposed that automatically adapts the regularization parameter $M$ in the cubic sub-problem and avoids saddle points. The primary theoretical contributions concerning the adaptation of $M$ are encapsulated in Lemma~\ref{lemma:minLmaxL}, Theorems~\ref{thm:minLmaxL2} and~\ref{thm:probequiv}, as well as the methodologies detailed in Algorithms~\ref{alg:mainAlg} and~\ref{alg:newton}. Figure~\ref{fig:outlineTheory}, in Appendix~\ref{appendix:theorySummary}, depicts how the key lemmata, theorems, and corollaries are logically connected throughout~\Cref{sec:proposed,sec:covAnalysis,sec:probSolution} and 
Appendices~\labelcref{appendix:globalSol,appendix:minLmaxL,appendix:minLmaxL2,appendix:probequiv,supplement:gradientDeviationBound,supplement:hessianDeviationBound,appendix:nursolution,supplement:diagHessian,supplement:diagHessianANDlemma1Nesterov2006,supplement:algDetails,supplement:diagmodelminRelatedLemmata,supplement:diagmodelmin,supplement:vectormatrixBernstein}.

\item The proposed optimizer does not need the computation of Krylov sub-space \citet{wang2020,zhou2020,kohler2017} or the calculation of the smallest eigenvalue \citet{allen2018, allenzhu2018, park2020} to obtain an optimal solution. The optimal solution is obtained by leveraging Hutchinson’s method that approximates the diagonal of the Hessian matrix~\citep{bekas2007}. In this way, the proposed method exhibits low memory complexity. 
\item The convergence rate of \textsc{AdaCubic} is established by exploiting the diagonal structure of the approximate Hessian matrix, which is computed using data batches. This property makes \textsc{AdaCubic} particularly appealing for deep learning applications.
\item \textsc{AdaCubic} is tested on Computer Vision, Natural Language Processing, and Signal Processing tasks, demonstrating a competitive or better performance when compared to \textsc{SGD} \citet{robbins1951}, \textsc{Adam} \citet{kingma2014}, and \textsc{AdaHessian} \citep{ya2021} optimizers. It should be noted that the parametrization of \textsc{AdaCubic} is performed by employing a well-known set of parameters used in Trust Region algorithms~\citep[Section 17.1]{conn2000}. These parameters are used universally in experimental evaluations, thereby casting \textsc{AdaCubic} as an attractive optimizer when fine-tuning is prohibitive.
\end{itemize}

The paper is organized as follows. Section~\ref{sec:proposed} details the proposed optimization framework. The convergence analysis of the proposed optimization framework is demonstrated in Section~\ref{sec:covAnalysis}. Section~\ref{sec:probSolution} presents the algorithms that compute the optimal solution of the proposed optimization framework in Section~\ref{sec:proposed}. Experimental results, computational complexity, and conclusions are presented in Sections~\ref{sec:expeval},~\ref{sec:complexity}, and~\ref{sec:conclusions}, respectively.

\section{Proposed Optimization Framework}\label{sec:proposed}
\noindent \textbf{Outline.} Section~\ref{sec:preliminaries} introduces the fundamental definitions used throughout the paper, including the basic formulation of the CR method, which serves as a core building block of the proposed framework.
Section~\ref{sec:problemForm} then introduces an auxiliary constrained optimization problem that forms the foundation of the \textsc{AdaCubic}. The key intuition is to reformulate the classical CR method as a constrained problem in which the cubic regularization term appears explicitly as a constraint. By leveraging Lagrange multiplier theory, this reformulation yields an adaptive update mechanism that automatically adjusts the strength of the cubic regularization term in the CR method during optimization. To derive this update mechanism Lemmata~\ref{lemma:globalSol},~\ref{lemma:minLmaxL}, Theorem~\ref{thm:minLmaxL2}, Corollary~\ref{cor:strongduality}, and Theorem~\ref{thm:probequiv} are introduced.

Lemma~\ref{lemma:globalSol} establishes that the auxiliary constrained problem admits a global minimizer and ensures that each optimization step is well defined.  Lemma~\ref{lemma:minLmaxL} is used to establish Theorem~\ref{thm:minLmaxL2}, which in turn is used to derive Corollary~\ref{cor:strongduality}. Corollary~\ref{cor:strongduality} shows that the auxiliary optimization problem is characterized by \emph{strong duality}~\citep[Section~5.4]{boyd2004}. The latter theoretical results are then combined to derive Theorem~\ref{thm:probequiv}, which provides the basis to replace the fixed cubic regularization parameter of the CR method with an adaptive one and finally derive the \textsc{AdaCubic} optimizer presented in Section~\ref{sec:probSolution}.

\subsection{Preliminaries}\label{sec:preliminaries}

To simplify notation, the iteration index $k$ in $\mathbf{x}_k \in \mathbb{R}^d$ will be explicitly denoted when necessary. Otherwise, it will be suppressed. Let $\nabla^2_{\mathbf{x}}f(\mathbf{x}_k)$ and $\nabla_{\mathbf{x}}f(\mathbf{x}_k)$ be the Hessian matrix and the gradient of the function $f(\mathbf{x}_k)$ with respect to (w.r.t.) $\mathbf{x}$. In the following, the subscript $\mathbf{x}$ in $\nabla^2_{\mathbf{x}}$ and $\nabla_{\mathbf{x}}$ is omitted for simplicity, resulting in $\nabla^2f(\mathbf{x}_k)$ and $\nabla f(\mathbf{x}_k)$, respectively. The spectrum of the symmetric $d\times d$ matrix $\nabla^2 f(\mathbf{x}_k)$ is denoted by $\lambda (\nabla^2 f(\mathbf{x}_k))= \{\lambda_i(\nabla^2 f(\mathbf{x}_k))\}_{i=1}^{d}$. Suppose that the eigenvalues are sorted in descending order, i.e.,
\begin{equation}\label{eq:lambdaorder}
\lambda_1(\nabla^2 f(\mathbf{x}_k))\geq \dots \geq \lambda_d(\nabla^2 f(\mathbf{x}_k)) = \lambda_{\min} (\nabla^2 f(\mathbf{x}_k)).
\end{equation} 
If $\nabla^2 f(\mathbf{x}_k)$ is indefinite, i.e., 
\begin{equation}\label{eq:indefiniteH}
\lambda_d(\nabla^2 f(\mathbf{x}_k)) < 0 \quad\mathrm{and}\quad \lambda_i(\nabla^2 f(\mathbf{x}_k)) > 0, \quad i<d, 
\end{equation} 
then $f(\mathbf{x})$ is non-convex. The notations $\nabla^2 f(\mathbf{x}_k)\succeq 0$ or $\nabla^2 f(\mathbf{x}_k)\succ 0$ indicate that the Hessian matrix is positive semi-definite or positive definite, respectively. Let $\partial_\tau \mathscr{F} $ be the partial derivative of a function $\mathscr{F} \colon \mathbb{R} \to \mathbb{R}$ w.r.t. the real-valued variable $\tau$. Moreover, let $\odot$ and $\oslash$ denote the element-wise product and division, respectively. In addition, let $\operatorname{diag}(\nabla^2 f(\mathbf{x}_k)) = \bigl[[\nabla^2 f(\mathbf{x}_k)]_{11},\dots, [\nabla^2 f(\mathbf{x}_k)]_{dd}\bigr]^T\in \mathbb{R}^{d}$ be a column vector containing the diagonal elements of the Hessian matrix and $\operatorname{Diag}(\nabla^2 f(\mathbf{x}_k))= \nabla^2 f(\mathbf{x}_k) \odot \mathbf{I}$ a $d \times d$ stand for a diagonal matrix retaining the diagonal elements of the Hessian matrix, where $\mathbf{I}$ is the identity matrix.  $\norm{\cdot}_2$ refers to the vector $\ell_2$ norm or to the spectral norm of a matrix. The $ d$-dimensional vector of ones is denoted by $\mathbbm{1}_d$.

A non-convex optimization problem is defined by
\begin{equation}\label{prob:orgProb}
\begin{aligned}
& \underset{\mathbf{x}\in \mathbb{R}^d}{\min}
& & f(\mathbf{x}) \stackrel{\Delta}{=} \frac{1}{n} \: \sum_{\ell=1}^n f_{\ell}(\mathbf{x}), 
\end{aligned}
\end{equation} 
where $f \colon \mathbb{R}^d \to \mathbb{R}$ and $f_\ell \colon \mathbb{R}^d \to \mathbb{R}$ are non-convex functions. Solving~(\ref{prob:orgProb}) is generally NP-Hard~\citep{murty1987,hillar2013}. As a result, a reasonable goal is to find an $\epsilon$-stationary point, i.e., an approximate local minimum, by checking $\norm{\nabla f(\mathbf{x})}_2 \leq \epsilon$, where $\nabla f(\mathbf{x}) \in \mathbb{R}^d$ is treated as a column vector. However, $\epsilon$-stationary points can be non-degenerate saddle points (i.e., the Hessian matrix at all saddle points has negative eigenvalues) or even local extrema in non-convex optimization. To avoid saddle points, second-order methods are used to find an $(\epsilon_g, \epsilon_H)$-stationary point by checking
\begin{equation}\label{eq:soc}
\norm{\nabla f(\mathbf{x})}_2 \leq \epsilon_{g} \quad \text{and} \quad \lambda_{\min}\Bigl(\nabla^2 f(\mathbf{x})\Bigr) \geq -\epsilon_{H},
\end{equation} where $\epsilon_{g},\epsilon_{H} > 0$, and $\lambda_{\min}\bigl(\nabla^2 f(\mathbf{x})\bigr)$ denotes the minimal eigenvalue of the Hessian matrix. CR technique is designed to avoid saddle points~\citep{nesterov2006}. Starting from an arbitrary point $\mathbf{x}_0$, the update rule of CR that solves~(\ref{prob:orgProb}) is written as 
\begin{equation}\label{prob:cubicNesterov2}
\begin{aligned}
 \mathbf{s}_{k+1} = \underset{\mathbf{s} \in \mathbb{R}^d}{\argmin}~m_{M}(\mathbf{s})
\end{aligned} 
\end{equation} 
where
\allowdisplaybreaks
\begin{equation}\label{eq:basicCubicProblem}
m_{M}(\mathbf{s})  \stackrel{\Delta}{=} f(\mathbf{x}_k) + \nabla f(\mathbf{x}_k)^T \mathbf{s}   + \frac{1}{2} \:\mathbf{s}^T \: \nabla^2 f(\mathbf{x}_k) \: \mathbf{s} + \frac{M}{6}\:\norm{\mathbf{s}}_2^3, 
\end{equation} 
$\mathbf{x}_{k+1} \stackrel{\Delta}{=} \mathbf{x}_{k}+\mathbf{s}_{k+1}$, and $M > 0$ is the regularization parameter that can be fixed or adaptive~\citep{nesterov2006,cartis2011}. 
In the following sections, the problem formulation and its solution are presented. 

\subsection{Problem Formulation}\label{sec:problemForm}

\noindent \textbf{Auxiliary Problem.} We are interested in developing an adaptive method for selecting $M$ in~(\ref{prob:cubicNesterov2}). To do so, we introduce the auxiliary constrained optimization problem
\begin{equation}\label{prob:cubicOursEquiv}
\begin{aligned}
& \underset{\mathbf{s} \in \mathbb{R}^d}{\argmin}
&& \hat{m}(\mathbf{s}) \stackrel{\Delta}{=}  f(\mathbf{x}_k) + \nabla f(\mathbf{x}_k)^T \mathbf{s} + \frac{1}{2}\: \mathbf{s}^T \: \nabla^2 f(\mathbf{x}_k) \: \mathbf{s} \quad \\ & \text{subject to} &&  g_{\xi} (\mathbf{s}) \stackrel{\Delta}{=} \frac{1}{6}\:\Bigl(\norm{\mathbf{s}}_2^3 - \xi \Bigr)\leq 0,
\end{aligned}
\end{equation} 
for $\xi \geq 0$. The Lagrangian function of~(\ref{prob:cubicOursEquiv}) is
\begin{equation}\label{eq:lagrangeCubicOursEquiv}
\pazocal{L}_{\xi}(\mathbf{s}, \nu) =  f(\mathbf{x}_k)  + \nabla f(\mathbf{x}_k)^T \mathbf{s}  + \frac{1}{2}\: \mathbf{s}^T\:\nabla^2 f(\mathbf{x}_k) \: \mathbf{s} + \frac{\nu}{6}\: \Bigl(\norm{\mathbf{s}}_2^3 - \xi \Bigr),
\end{equation} 
where $\nu$ is the Lagrange multiplier. Let
$\Omega = \{ \mathbf{s} \mid g_{\xi} (\mathbf{s}) \leq 0 \}$. The minimizer we are seeking in~(\ref{prob:cubicOursEquiv}) lies either within the interior of $\Omega$ (i.e., $g_{\xi} (\mathbf{s}) < 0$) or lies on the boundary of $\Omega$ (i.e., $g_{\xi} (\mathbf{s}) = 0$). Lemma~\ref{lemma:globalSol} is an immediate result of the previous discussion.

\begin{restatable}{lem}{globalSol}
A vector $\mathbf{s}^*$ is a minimizer of $\hat{m}(\mathbf{s})$ subject to $\norm{\mathbf{s}^*}_2^3 \leq \xi$ if and only if satisfies
\begin{equation}\label{eq:firstOrderCond}
\left(\nabla^2 f(\mathbf{x}_k) + \frac{\nu^*}{2} \norm{\mathbf{s}^*}_2~ \mathbf{I} \right)\: \mathbf{s}^* = - \nabla f(\mathbf{x}_k),
\end{equation} 
\begin{equation} \label{eq:possDefglobalSol}
\nabla^2 f(\mathbf{x}_k) + \frac{\nu^*}{2} \norm{\mathbf{s}^*}_2~ \mathbf{I} \succeq 0,
\end{equation} and $\nu^* \:(\norm{\mathbf{s}^*}_2^3 - \xi) = 0$, where $\nu^* \geq 0$. If $\nabla^2 f(\mathbf{x}_k) + \frac{\nu^*}{2} \norm{\mathbf{s}^*}_2~ \mathbf{I} \succ 0 $, then the minimizer $\mathbf{s}^*$ is unique.
\label{lemma:globalSol}
\end{restatable}  
The condition $\nu^* \:(\norm{\mathbf{s}^*}_2^3 - \xi) = 0$ in Lemma~\ref{lemma:globalSol} is called Complementary Slackness (CS) condition. The proof of Lemma~\ref{lemma:globalSol} can be found in Appendix~\ref{appendix:globalSol}.

\begin{restatable}{defin}{defD} 
For some $\nu \geq 0$, denote 
\begin{equation}
\pazocal{D}_\nu \stackrel{\Delta}{=} \left\{ r \mid \nabla^2 f(\mathbf{x}_k) + \frac{\nu \: r}{2} \:\mathbf{I} \succ \bm{0},\quad r > 0 \right\}.
\end{equation}
\label{def:defD}
\end{restatable} 
Next, it is proven that problem~(\ref{prob:cubicOursEquiv}) is characterized by strong duality. To do so, Lemma~\ref{lemma:minLmaxL} and Theorem~\ref{thm:minLmaxL2} are introduced. Lemma~\ref{lemma:minLmaxL} is used as a preliminary result to prove Theorem~\ref{thm:minLmaxL2}. Corollary~\ref{cor:strongduality} establishes the strong duality of problem~(\ref{prob:cubicOursEquiv}) as an immediate outcome of Theorem~\ref{thm:minLmaxL2}.

\begin{restatable}[Proof in Appendix~\ref{appendix:minLmaxL}]{lem}{minLmaxL}
For $r \in \pazocal{D}_\nu$ we have
\begin{equation}\label{eq:minLmaxL}
\min_{\mathbf{s} \in \mathbb{R}^d}\: \pazocal{L}_{\xi} (\mathbf{s}, \nu) = \max_{r \in \pazocal{D}_\nu}\: \mathscr{L}_{\xi}(\mathbf{s}(\nu, r), \nu, r),
\end{equation} 
where
{\allowdisplaybreaks
\begin{equation}\label{eq:lagrangianscr}
\mathscr{L}_{\xi}(\mathbf{s}(\nu, r), \nu, r) =  -\frac{1}{2}\:\nabla f(\mathbf{x}_k)^T \Bigl(\nabla^2 f(\mathbf{x}_k) + \frac{\nu \: r}{2} \:\mathbf{I}\Bigr)^{-1} \nabla f(\mathbf{x}_k) - \frac{\nu}{6}\xi - \frac{\nu}{12} r^3.
\end{equation}} 
For $r \in \pazocal{D}_\nu$ the direction 
\begin{equation}\label{eq:hx}
\mathbf{s} (\nu, r) = - \left(\nabla^2 f(\mathbf{x}_k) + \frac{\nu \: r}{2} \:\mathbf{I}\right)^{-1}\nabla f(\mathbf{x}_k), 
\end{equation} 
satisfies 
{\allowdisplaybreaks
\begin{equation} 
\pazocal{L}_{\xi} (\mathbf{s}(\nu, r), \nu) = \mathscr{L}_{\xi}(\mathbf{s}(\nu, r), \nu, r) + \frac{4}{3\nu} \frac{\Bigl(r+2\norm{\mathbf{s}(\nu, r)}_2 \Bigr)}{\Bigl(r+\norm{\mathbf{s} (\nu, r)}_2 \Bigr)^2}  \Bigl(\partial_{r} \mathscr{L}_{\xi}(\mathbf{s}(\nu, r), \nu, r) \Bigr)^2.
\end{equation}} 
For $r^* \in \pazocal{D}_\nu$ that maximizes $\max_{r \in \pazocal{D}_\nu}\: \mathscr{L}_{\xi}(\mathbf{s}(\nu, r), \nu, r)$, 
\begin{equation}\label{eq:solutionhopt}
\mathbf{s}(\nu, r^*) = - \left(\nabla^2 f(\mathbf{x}_k) + \frac{\nu}{2} \norm{\mathbf{s}(\nu, r^*)}_2 \:\mathbf{I}\right)^{-1}\nabla f(\mathbf{x}_k)
\end{equation} 
is the minimizer of $\min_{\mathbf{s} \in \mathbb{R}^d}\: \pazocal{L}_{\xi} (\mathbf{s}, \nu)$ in~(\ref{eq:minLmaxL}).
\label{lemma:minLmaxL}
\end{restatable}

\begin{restatable}[Proof in Appendix~\ref{appendix:minLmaxL2}]{thm}{minLmaxL2} 
We have
\begin{equation}\label{eq:minLmaxL2}
\min_{\mathbf{s}\in \mathbb{R}^d}\max_{\nu \geq 0} \pazocal{L}_{\xi} (\mathbf{s}, \nu) = \max_{\nu \geq 0, ~r \in \pazocal{D}_\nu} \mathscr{L}_{\xi}(\mathbf{s}(\nu, r), \nu, r),
\end{equation} 
where $\mathscr{L}_{\xi}(\mathbf{s}(\nu, r), \nu, r)$ is defined in~(\ref{eq:lagrangianscr}). For $r \in \pazocal{D}_\nu$, the direction 
\begin{equation}\label{eq:solutionh}
\mathbf{s} (\nu, r) = - \left(\nabla^2 f(\mathbf{x}_k) + \frac{\nu \: r}{2} \:\mathbf{I}\right)^{-1}\nabla f(\mathbf{x}_k), 
\end{equation} 
satisfies
{\allowdisplaybreaks
\begin{equation}
\pazocal{L}_{\xi} (\mathbf{s}(\nu, r), \nu) =  \mathscr{L}_{\xi}\bigl(\mathbf{s}(\nu, r), \nu, r\bigr) - \nu~\partial_{\nu} \mathscr{L}_{\xi}\bigl(\mathbf{s}(\nu, r), \nu, r\bigr) +  \frac{4}{3\nu} \frac{\left(r+2~\norm{\mathbf{s}(\nu, r)}_2 \right)}{\left(r+\norm{\mathbf{s}(\nu, r)}_2 \right)^2}  \Bigl(\partial_{r} \mathscr{L}_{\xi}(\mathbf{s}(\nu, r), \nu, r) \Bigr)^2.
\end{equation}} 
For the optimal values $\nu^*$ and $r^* \in \pazocal{D}_\nu$ that maximize $\max_{\nu \geq 0, ~r \in \pazocal{D}_\nu} \mathscr{L}_{\xi}(\mathbf{s}(\nu, r), \nu, r)$,
\begin{equation}\label{eq:solutionhopt1}
\mathbf{s}^*(\nu^*, r^*) =  - \Bigl(\nabla^2 f(\mathbf{x}_k) + \frac{\nu^*}{2} \norm{\mathbf{s}(\nu^*, r^*)}_2 \:\mathbf{I}\Bigr)^{-1}\nabla f(\mathbf{x}_k), 
\end{equation} 
is the minimizer of $\min_{\mathbf{s}\in \mathbb{R}^d}\max_{\nu \geq 0} \pazocal{L}_{\xi} (\mathbf{s}, \nu)$  in~(\ref{eq:minLmaxL2}), i.e., the optimal $\mathbf{s}^*$ in Lemma~\ref{lemma:globalSol}. 
\label{thm:minLmaxL2}
\end{restatable}

\begin{restatable}{cor}{strongduality} 
The constrained optimization problem~(\ref{prob:cubicOursEquiv}) is characterized by \textit{strong duality}, i.e.,
\begin{equation}\label{eq:strongDuality}
\min_{\mathbf{s}\in \mathbb{R}^d} \max_{\nu \geq 0} \pazocal{L}_{\xi}(\mathbf{s}, \nu) =\max_{\nu \geq 0} \min_{\mathbf{s}\in \mathbb{R}^d} \pazocal{L}_{\xi} (\mathbf{s}, \nu).
\end{equation} 
\label{cor:strongduality}
\begin{proof}
See the proof of Theorem~\ref{thm:minLmaxL2}.  
\end{proof}
\end{restatable}

Given Corollary~\ref{cor:strongduality}, the equivalence between problems~(\ref{prob:cubicNesterov2}) and~(\ref{prob:cubicOursEquiv}) is established in Theorem~\ref{thm:probequiv}. The equivalence implies that both problems have the same optimum.

\begin{restatable}[Proof in Appendix~\ref{appendix:probequiv}]{thm}{probequiv} 
Let $\nu^*$ be the optimal dual variable of the constrained optimization problem~(\ref{prob:cubicOursEquiv}). The following optimization problems
\begin{equation}\label{eq:equivalence}
\min_{\mathbf{s} \in \mathbb{R}^d}\: m_{M} (\mathbf{s}) \: \: \mathrm{and} \: \: \min_{\mathbf{s}\in \mathbb{R}^d}\: \hat{m} (\mathbf{s}) \: \mathrm{subject \: to} \: \mbox{ $g_{\xi} (\mathbf{s}) \leq 0$}
\end{equation} are equivalent w.r.t. the optimal solution $\mathbf{s}^*$, when $M = \nu^*$ and $\xi = \norm{\mathbf{s}^*}_2^3$.
\label{thm:probequiv}
\end{restatable}

\section{Local Convergence Analysis}\label{sec:covAnalysis}

\noindent \textbf{Outline.} This section provides the local convergence analysis of Algorithm~\ref{alg:mainAlg}. It begins with Assumption~\ref{assum:continuityAssum}, which defines the Lipschitz continuity constants for $f_i(\mathbf{x})$, $\nabla f_i (\mathbf{x})$, and $\nabla^2 f_i(\mathbf{x})$. Subsequently, Theorem~\ref{thm:diagmodelmin} establishes the local convergence of Algorithm~\ref{alg:mainAlg} when using the exact gradient and Hessian matrix.

Adequate agreement between the exact gradient $\nabla f(\mathbf{x}_k)$ and the approximate gradient $\mathbf{g}_k$ is established in Assumption~\ref{assum:gradAgreementAssum}. This assumption is grounded on~\citet[Assumption 2]{wang2019} and facilitates the approximation of the gradient using a sampling scheme in Lemma~\ref{lemma:gradientDeviationBound}, akin to the one outlined in~\citep[Theorem 7]{kohler2017}. 

A sufficient agreement between the exact diagonal  Hessian matrix $\operatorname{Diag} (\nabla^2 f(\mathbf{x}_k))$ and the approximate diagonal Hessian matrix $\mathbf{B}_k$ in~(\ref{eq:approxHessianDiag}), is established in Assumption~\ref{assum:hessAgreementAssum}.
Assumption~\ref{assum:hessAgreementAssum} is a direct application of~\citep[Assumption 2]{wang2019}. Additionally, Assumption~\ref{assum:hessAgreementAssum} supports Lemma~\ref{lemma:hessianDeviationBound}, while Lemma~\ref{lemma:aproxmodelmin} is pivotal for establishing Lemma~\ref{lemma:localConvCond}. Lemma~\ref{lemma:localConvCond} is used in Lemmata~\ref{lemma:gradientDeviationBound},~\ref{lemma:hessianDeviationBound}, and Corollary~\ref{cor:gradienthessianSampling}. Lemmata~\ref{lemma:gradientDeviationBound} and~\ref{lemma:hessianDeviationBound} provide the deviation bounds for the gradient and Hessian matrix, along with the corresponding conditions required for these bounds to hold. These conditions are consolidated in Corollary~\ref{cor:gradienthessianSampling}, which ensures the validity of both deviation bounds.

The analysis concludes by discussing the local convergence of the sub-sampled case, where the exact gradient $\nabla f(\mathbf{x}_k)$ and diagonal Hessian matrix $\Diag(\nabla^2 f(\mathbf{x}_k))$ are replaced with their sub-sampled approximations $\mathbf{g}_k$ and $\mathbf{B}_k$, respectively.

\noindent \textbf{Convergence Analysis.} Next, we begin with the main results of the analysis. Assumption~\ref{assum:continuityAssum} is commonly used in previous works~\citep{nesterov2006,cartis2011,cartis2011B,kohler2017} and is applied here in combination with Remark~\ref{rmrk:continuityAssum1}.

Let $\pazocal{F} \subseteq \mathbb{R}^d$ be a closed convex set with a non-empty interior. Let $\mathbf{x}_0 \in \mathrm{int}\: \pazocal{F}$ be a starting point of the iterative optimization scheme in the interior of $\pazocal{F}$.

\begin{restatable}[Continuity]{assum}{continuityAssum} 
The convergence analysis is based on the following assumptions:
\begin{itemize}
\item The functions $f_i (\mathbf{x})$ are twice-continuously differentiable and bounded from below by $f_i^{\mathrm{low}}$.
\item The functions $f_i(\mathbf{x})$, $\nabla f_i (\mathbf{x})$, and $\nabla^2 f_i(\mathbf{x})$ are Lipschitz continuous in $\pazocal{F}$ with Lipschitz constants $L_f$, $L_{g}$, and $L_{H}$, respectively.
\end{itemize}
\label{assum:continuityAssum}
\end{restatable} 

\begin{restatable}{rmrk}{continuityAssum1} 
Due to the triangle inequality, it follows that the Lipschitz continuity also holds for $f(\mathbf{x})$, $\nabla f(\mathbf{x})$, and $\nabla^2 f(\mathbf{x})$, with Lipschitz constants $L_f$, $L_{g}$, and $L_{H}$, respectively. In addition, given Assumption~\ref{assum:continuityAssum}, $f(\mathbf{x})$ is also lower bounded by some $f^{\mathrm{low}}$. 
\label{rmrk:continuityAssum1}
\end{restatable}

By leveraging Theorem~\ref{thm:probequiv}, the iteration complexity of Algorithm~\ref{alg:mainAlg} is equivalent to that performed by the cubic regularization method in~\citep{nesterov2006}. Theorem~\ref{thm:diagmodelmin} analyses the iteration complexity of Algorithm~\ref{alg:mainAlg} by adapting the analysis from \citet[Theorem 1]{nesterov2006}, when $\Diag(\nabla^2 f(\mathbf{x}))$ replaces $\nabla^2 f(\mathbf{x})$.

\begin{restatable}[Proof in Appendix~\ref{supplement:diagmodelmin}]{thm}{diagmodelmin} 
Suppose Assumption~\ref{assum:continuityAssum} holds. Also, let the sequence $\mathbf{x}_{i}$, with $i\geq 0$, be generated by Algorithm~\ref{alg:mainAlg} when $\Diag (\nabla^2 f(\mathbf{x}_i))$ is used. Then, after $k$ iterations, the sequence $\{ \mathbf{x}_i \}_{i \geq 1}$ satisfies
\begin{equation}\label{eq:normgradfUpperExactDiag}
\min_{1 \leq i\leq k} \norm{\nabla f(\mathbf{x}_i)}_2 \leq 
\pazocal{O} \left(\frac{1}{k^{2/3}}\right).
\end{equation} 
\label{thm:diagmodelmin}
\end{restatable} 
If we want to find the iteration $k$ that satisfies $\min_{1 \leq i\leq k} \norm{\nabla f(\mathbf{x}_i)}_2 \leq \epsilon$, we upper bound~(\ref{eq:normgradfUpperExactDiag}) by $\epsilon$ and we conclude that 
\begin{equation}\label{eq:kIterComplexExactDiag}
k \geq \pazocal{O} \left(\frac{1}{\epsilon^{3/2}}\right).
\end{equation} 

\noindent \textbf{Deviation Bounds.} Rather than utilizing deterministic gradient and Hessian information, we can employ estimates of the gradient, the Hessian matrix, and the loss function, which are derived from an independent set of points $\pazocal{B}_k$, i.e.,
\begin{equation}\label{eq:approxGrad}
    \mathbf{g}_k = \frac{1}{|\pazocal{B}_k|} \sum_{i \in \pazocal{B}_k} \nabla f_i (\mathbf{x}_k),
\end{equation}
\begin{equation}\label{eq:approxHessian}
    \mathbf{H}_k = \frac{1}{|\pazocal{B}_k|} \sum_{i \in \pazocal{B}_k} \nabla^2 f_i (\mathbf{x}_k),
\end{equation} and 
\begin{equation}\label{eq:approxLoss}
   F (\mathbf{x}_k) = \frac{1}{|\pazocal{B}_k|} \sum_{i \in \pazocal{B}_k} f_i (\mathbf{x}_k).
\end{equation}

\begin{restatable}[Sufficient agreement of $\mathbf{g}_k$ and $\nabla f(\mathbf{x}_k)$]{assum}{gradAgreementAssum} There is a constant $C_g>0$ such that the inexact gradient $\mathbf{g}_k$ satisfies, for all $k \geq 0$,
\begin{equation}\label{eq:gradAgreementAssum}
\norm{\mathbf{g}_k - \nabla f(\mathbf{x}_k)}_2 \leq C_g \norm{\mathbf{s}_k}_2^2.
\end{equation}
\label{assum:gradAgreementAssum}
\end{restatable} 
For some $\mathbf{x}_k$, the computation of the Hessian matrix $\mathbf{H}_{k}\in \mathbb{R}^{d \times d}$ in~(\ref{eq:approxHessian}) is expensive due to the large size $d$ of $\mathbf{x}_k$. Only the Hessian-vector product can be calculated at a reasonable computational complexity~\citep{pearlmutter1994}. Let $\pazocal{H}_{k}  \colon \mathbb{R}^d \to \mathbb{R}^d$ be a function such that $\pazocal{H}_{k} (\mathbf{v}) \stackrel{\Delta}{=} \mathbf{H}_{k} \mathbf{v}$, where $\mathbf{H}_{k} $ is not accessible. Given the Hessian-vector product operator $\pazocal{H}_{k}$, the diagonal of $\mathbf{H}_{k}$, i.e., $\mathbf{h}_{k}\stackrel{\Delta}{=}\operatorname{diag}(\mathbf{H}_{k})$, is approximated by the Hutchinson's method as~\citet{bekas2007}
\allowdisplaybreaks
\begin{equation}\label{eq:hutchinsosHessian}
\mathbf{b}_{k} = \left[\sum_{s=1}^{\pazocal{S}} \pazocal{H}_{k}(\mathbf{v}_s) \odot \mathbf{v}_s \right]\oslash \left[ \sum_{i=1}^{\pazocal{S}}  \mathbf{v}_s \odot \mathbf{v}_s \right] =  \frac{1}{\pazocal{S}}\sum_{s=1}^{\pazocal{S}} \pazocal{H}_{k}(\mathbf{v}_s) \odot \mathbf{v}_s \in \mathbb{R}^d,
\end{equation} 
where $\mathbf{v}_s \sim \mathtt{Rademacher (0.5)}$ and $\pazocal{S}$ is the number of random vectors used in the approximation. Thus, the diagonal approximate Hessian matrix $\mathbf{B}_{k} \in \mathbb{R}^{d \times d}$ is given by
\begin{equation}\label{eq:approxHessianDiag}
\mathbf{B}_{k} \stackrel{\Delta}{=} \operatorname{Diag}(\mathbf{b}_{k}) = \frac{1}{\pazocal{S}} \sum_{s=1}^\pazocal{S} \operatorname{Diag} \left( \pazocal{H}_{k}(\mathbf{v}_s) \odot \mathbf{v}_s \right).
\end{equation} 
The approximate Hessian matrix (\ref{eq:approxHessianDiag}) is used in the description of Algorithms~\ref{alg:mainAlg} and~\ref{alg:newton}, in Section~\ref{sec:probSolution}. It is worth noting that in the code implementation of Algorithms~\ref{alg:mainAlg} and~\ref{alg:newton} only the diagonal of $\mathbf{B}_{k}$ is computed, which reduces the memory cost from $d \times d$ to $d$.

\begin{restatable}{assum}{hessAgreementAssum} There is a constant $C_B>0$ such that the inexact Hessian $\mathbf{B}_k$ satisfies, for all $k \geq 0$,
\begin{equation}\label{eq:hessAgreementAssum}
\norm{\mathbf{B}_k - {\operatorname{Diag} (\nabla^2 f(\mathbf{x}_k))}}_2 \leq C_B \norm{\mathbf{s}_k}_2.
\end{equation}
\label{assum:hessAgreementAssum}
\end{restatable} 
By replacing $\nabla f (\mathbf{x}_k)$ and $\nabla^2 f (\mathbf{x}_k)$ with $\mathbf{g}_k$ and $\mathbf{B}_k$ we get
\begin{equation}\label{eq:basicCubicProblemApprox}
\mathfrak{m}_{M}(\mathbf{s}) \stackrel{\Delta}{=} F(\mathbf{x}_k) + \mathbf{g}_k^T \mathbf{s} + \frac{1}{2} \:\mathbf{s}^T \: \mathbf{B}_k \: \mathbf{s} + \frac{M}{6}\:\norm{\mathbf{s}}_2^3.
\end{equation} 
Note that the conditions under which $\nabla f (\mathbf{x}_k)$ and $\nabla^2 f (\mathbf{x}_k)$ can be substituted with $\mathbf{g}_k$ and $\mathbf{B}_k$ are detailed in Lemmata~\ref{lemma:gradientDeviationBound} and~\ref{lemma:hessianDeviationBound}, respectively. 

Let $(\mathbf{s}_{k+1}, \nu_{k+1})$ be the output of Algorithm~\ref{alg:newton} for $\mathbf{B}=\mathbf{B}_k$, $\mathbf{g}=\mathbf{g}_k$, and so on. Algorithm~\ref{alg:newton} is called in line~\ref{line:innerRootFinder} of Algorithm~\ref{alg:mainAlg}. Then, recall that $(\mathbf{s}_{k+1}, \nu_{k+1})$ is a minimizer of~(\ref{prob:cubicOursEquiv}) and according to Theorem~\ref{thm:probequiv} it is also a minimizer of problem~(\ref{prob:cubicNesterov2}) for $M=\nu_{k+1}$. 
The first- and second-order optimality conditions
\begin{equation}\label{eq:skplus1Nablafrakm}
\mathbf{s}_{k+1}^T\:\nabla_{\mathbf{s}} \mathfrak{m}_{M=\nu_{k+1}} (\mathbf{s}_{k+1})=0
\end{equation} 
and
\begin{equation}
\quad \mathbf{s}_{k+1}^T\:\Bigl(\nabla^2_{\mathbf{s}} \mathfrak{m}_{M=\nu_{k+1}}(\mathbf{s}_{k+1})\Bigr)\:\mathbf{s}_{k+1}\geq 0,
\end{equation}
get us to Lemma~\ref{lemma:aproxmodelmin}. Lemma~\ref{lemma:aproxmodelmin} is exploited to prove Lemma~\ref{lemma:localConvCond}. Lemma~\ref{lemma:localConvCond} is used by Lemmata~\ref{lemma:gradientDeviationBound},~\ref{lemma:hessianDeviationBound}, and Corollary~\ref{cor:gradienthessianSampling}.

\begin{restatable}[Approximate model minimizer]{lem}{aproxmodelmin}
Let 
\begin{equation}\label{prob:cubicNesterovAprox}
\begin{aligned}
 \mathbf{s}_{k+1} = \underset{\mathbf{s} \in \mathbb{R}^d}{\argmin}~\mathfrak{m}_{M}(\mathbf{s}).
\end{aligned} 
\end{equation} Then, the following statements hold
\begin{equation}\label{eq:firstOptCubicIter}
\mathbf{g}_k + \mathbf{B}_{k} \: \mathbf{s}_{k+1} + \frac{M}{2}\:\norm{\mathbf{s}_{k+1}}_2\mathbf{s}_{k+1} = \bm{0},
\end{equation}
\begin{equation}\label{eq:secondOptCubicIter}
\mathbf{B}_{k} + \frac{M}{2}\:\norm{\mathbf{s}_{k+1}}_2\mathbf{I} \succeq \bm{0}, 
\end{equation} 
and
\begin{equation}\label{eq:OptCubicIter}
\mathbf{g}_k^T \mathbf{s}_{k+1}+ \frac{1}{2} \mathbf{s}_{k+1}^T\mathbf{B}_{k} \: \mathbf{s}_{k+1} +  \frac{M}{6}\:\norm{\mathbf{s}_{k+1}}_2^3 \leq  -\frac{M}{12}\: \norm{\mathbf{s}_{k+1}}_2^3.
\end{equation} 
Recall that $\mathbf{x}_{k+1} = \mathbf{x}_{k} + \mathbf{s}_{k+1}$ and from Theorem~\ref{thm:probequiv}, $M=\nu_{k+1}$.
\begin{proof}
The reader is referred to~\citep[Lemma 3]{wang2019}.  
\end{proof}
\label{lemma:aproxmodelmin}
\end{restatable}

\begin{restatable}{lem}{localConvCond} 
Let $\{F(\mathbf{x}_k)\}$ be bounded from below by $F^{\mathrm{low}}$. Also, let $\mathbf{s}_{k+1}$ satisfy the first two conditions in Lemma~\ref{lemma:aproxmodelmin} and let $M$ be bounded from below by some $M^{\mathrm{low}}$. Then
\begin{equation}\label{eq:stozero}
    \norm{\mathbf{s}_{k+1}} \to 0, \text{ as}~k \to \infty.
\end{equation}

\begin{proof}
First, note that by Assumption~\ref{assum:continuityAssum}, $F(\mathbf{x})$ is also bounded from below by some $F^{\mathrm{low}}$. Additionally, since $M$ is bounded from below and $M=\nu_{k+1}$, as indicated in Theorem~\ref{thm:probequiv}, $\nu_{k+1}$ is also bounded from below. The lower bound of $M$ is further discussed in Lemma~\ref{lemma:lemma2Nesterov2006} in Appendix~\ref{supplement:diagmodelminRelatedLemmata}.

Following similar lines to~\citet[Lemma 5.1]{cartis2011}, we focus on the sub-sequence of successful iterations, as in~\citep{cartis2011,conn2000}. Thus, from the successful iteration in Algorithm~\ref{alg:mainAlg}, i.e., when $\rho_k \in [\eta_1, \eta_2)$, we have
\allowdisplaybreaks
\begin{multline}
    F(\mathbf{x}_k) - F(\mathbf{x}_{k+1}) \geq  \eta_1 (F(\mathbf{x}_k) - \mathfrak{m}_{M=\nu_{k+1}} (\mathbf{s}_{k+1})) \\ \stackrel{(\ref{eq:basicCubicProblemApprox})}{\geq}  \eta_1 \Biggl(- \mathbf{g}_k^T \mathbf{s}_{k+1} -  \frac{1}{2} \:\mathbf{s}_{k+1}^T \: \mathbf{B}_k \: \mathbf{s}_{k+1} - \frac{M^{\mathrm{low}}}{6}\:\norm{\mathbf{s}_{k+1}}_2^3\Biggr),
\end{multline} 
which by applying~(\ref{eq:OptCubicIter}) yields
\begin{equation}\label{eq:localConvCond}
    F(\mathbf{x}_k) - F(\mathbf{x}_{k+1}) \geq \eta_1 \frac{M^{\mathrm{low}}}{12} \norm{\mathbf{s}_{k+1}}_2^3.
\end{equation} Summing over all iterates from $0$ to $k-1$ in~(\ref{eq:localConvCond}) we obtain
\begin{equation}
    F(\mathbf{x}_0) - F(\mathbf{x}_{k+1}) \geq \frac{\eta_1}{12} M^{\mathrm{low}} \sum_{k=0}^{k-1} \norm{\mathbf{s}_k}_2^3,
\end{equation} 
which taking into account that $\{F(\mathbf{x}_k)\}$ is bounded below yields
\begin{equation}
    \frac{12}{\eta_1 M^{\mathrm{low}}} \left(F(\mathbf{x}_0) - F^{\mathrm{low}} \right) \geq \sum_{k=0}^{k-1} \norm{\mathbf{s}_k}_2^3.
\end{equation} 
Thus, the series $\sum_{k=0}^{k-1} \norm{\mathbf{s}_k}_2^3$ is convergent and~(\ref{eq:stozero}) holds. The same conclusion is also derived in~\citep[Lemma 5.1]{cartis2011}.
\end{proof}
\label{lemma:localConvCond}
\end{restatable}

\begin{restatable}{lem}{gradientDeviationBound}
Let the approximate gradient $\mathbf{g}_k$ be computed on a set of points $\pazocal{B}_k^g$, with cardinality $|\pazocal{B}_k^g|$. For $\epsilon \geq 4 \sqrt{2}L_f \sqrt{\frac{\ln \frac{1}{\delta} + \frac{1}{4}}{|\pazocal{B}_k^g|}}$ we have with high probability $1 - \delta$ that 
\begin{equation}\label{eq:gradientDeviationBound}
    \norm{\mathbf{g}_k  - \nabla f(\mathbf{x}_k)}_2 \leq \epsilon.
\end{equation} 
In addition, if
\begin{equation}\label{eq:gradientSampling}
    |\pazocal{B}_k^g| \geq 32 L_f^2 \frac{\ln \frac{1}{\delta} + \frac{1}{4}}{C_g^2 \norm{\mathbf{s}_k}_2^4},
\end{equation} 
and Lemma~\ref{lemma:localConvCond} holds, $\mathbf{g}_k$ satisfies Assumption~\ref{assum:gradAgreementAssum}. 
\label{lemma:gradientDeviationBound}
\begin{proof}
The proof can be found in Appendix~\ref{supplement:gradientDeviationBound}.
\end{proof}
\end{restatable}

\begin{restatable}{lem}{hessianDeviationBound}
Let the approximate diagonal Hessian matrix $\mathbf{B}_k$ be computed on a set of points $\pazocal{B}_k^H$, with cardinality $|\pazocal{B}_k^H|$. For $\epsilon \geq \sqrt{d} L_g \frac{\ln \frac{2d}{\delta}}{\pazocal{S}|\pazocal{B}_k^H|}$ we have with high probability $1 - \delta$ that 
\begin{equation}\label{eq:hessianDeviationBound}
    \norm{\mathbf{B}_k  - \Diag(\nabla^2 f(\mathbf{x}_k))}_2 \leq \epsilon.
\end{equation} 
In addition, if
\begin{equation}\label{eq:hessianSampling}
    |\pazocal{B}_k^H| \geq \sqrt{d} L_g \frac{\ln \frac{2d}{\delta}}{\pazocal{S}\norm{\mathbf{s}_k}_2 C_B}
\end{equation} 
and Lemma~\ref{lemma:localConvCond} holds, $\mathbf{B}_k$ satisfies Assumption~\ref{assum:hessAgreementAssum}.
\begin{proof}
The proof can be found in Appendix~\ref{supplement:hessianDeviationBound}.
\end{proof}
\label{lemma:hessianDeviationBound}
\end{restatable}

\begin{restatable}{cor}{gradienthessianSampling} 
If 
\begin{equation}\label{eq:gradienthessianSampling}
    |\pazocal{B}_k| \geq \max \left\{32 L_f^2 \frac{\ln \frac{1}{\delta} + \frac{1}{4}}{C_g^2 \norm{\mathbf{s}_{k-1}}_2^4}, \sqrt{d} L_g \frac{\ln \frac{2d}{\delta}}{\pazocal{S}\norm{\mathbf{s}_{k-1}}_2 C_B} \right\}, 
\end{equation} 
then $\mathbf{g}_k$ and $\mathbf{B}_k$ satisfy Assumptions~\ref{assum:gradAgreementAssum} and~\ref{assum:hessAgreementAssum} with probability $1-\delta$, for $\delta \in (0, 1]$.
\label{cor:gradienthessianSampling}
\begin{proof}
We combine the results of Lemma~\ref{lemma:gradientDeviationBound} and~\ref{lemma:hessianDeviationBound}. 
Note that $\norm{\mathbf{s}_{k-1}}$ is used instead of $\norm{\mathbf{s}_{k}}$. Due to Lemma~\ref{lemma:localConvCond},  $\norm{\mathbf{s}_{k}}_2 \leq \norm{\mathbf{s}_{k-1}}_2 \Leftrightarrow \norm{\mathbf{s}_{k}}_2^{-1} \geq \norm{\mathbf{s}_{k-1}}_2^{-1}$. This modification is useful for the practical application of the sampling schemes. However, this poses a challenge since $C_g$, $C_H$, $L_f$, and $L_g$ are not easily accessible.
\end{proof}
\end{restatable} 

\begin{restatable}{rmrk}{Discussion} 
Lemma~\ref{lemma:localConvCond} and Corollary~\ref{cor:gradienthessianSampling} imply that the sample size is eventually equal to the entire sample size $n$ as Algorithm~\ref{alg:mainAlg} converges. Thus we have
\begin{equation}
    \mathbf{g}_k \to \nabla f(\mathbf{x}_k) \text{ and } \mathbf{B}_k \to \Diag( \nabla^2 f(\mathbf{x}_k)) \text{ as } k \to \infty.
\end{equation} 
This allows us to invoke the deterministic local convergence guarantees as $k \to \infty$ in Theorem~\ref{thm:diagmodelmin}. However, stochastic first- and second-order information from $\mathbf{g}_k$ and $\mathbf{B}_k$ is used. 
\label{rmrk:Discussion}
\end{restatable}

\section{Algorithmic Solution}\label{sec:probSolution}

In Theorem~\ref{thm:minLmaxL2}, it was shown that the optimal $(\nu^*, r^*)$ solving $ \max_{\nu \geq 0, ~r \in \pazocal{D}_\nu} \mathscr {L}_{\xi} (\nu, r)$ is used in~(\ref{eq:solutionh}) to compute the minimizer of~(\ref{prob:cubicOursEquiv}). To solve~(\ref{prob:cubicOursEquiv}) and compute $(\nu^*, r^*)$ Algorithm~\ref{alg:mainAlg} and~\ref{alg:newton} are utilized, respectively. In particular, Lemma~\ref{lemma:nursolution} is employed in Algorithm~\ref{alg:newton}, which is essential for calculating the values of $\nu^*$ and $r^*$.

Let  \textit{VSI}, \textit{SI}, and \textit{UI} stand for the Very Successful, Successful, and Unsuccessful Iteration, respectively, in Algorithm~\ref{alg:mainAlg}~\citep[Section 6.1]{conn2000}. Denote  $\lambda_d^+ (\mathbf{B}_k)$ as the minimal non-negative diagonal shift that makes $\mathbf{B}_k$ sufficiently positive definite to allow a stable computation of the TR step. Details on the selection of $\lambda_d^+ (\mathbf{B}_k)$ in Algorithm~\ref{alg:newton} can be found in~\citet[Section 7.3.11]{conn2000} and~\citep{gould1999}.

\begin{restatable}[Proof in Appendix~\ref{appendix:nursolution}]{lem}{nursolution} 
The optimal values $\nu^*$ and $r^*$ achieving
\begin{equation}\label{prob:minLmaxL4}
\max_{\nu \geq 0, ~r \in \pazocal{D}_\nu} \mathscr {L}_{\xi} (\nu, r)
\end{equation} 
are given by $r^* = \sqrt[3]{\xi}$ and by solving 
\begin{equation}\label{eq:systemnur2}
\phi (\nu^*, r^*) = \frac{1}{\norm{\mathbf{s}(\nu^*, r^*)}_2} - \frac{1}{\sqrt[3]{\xi}} = 0,
\end{equation} w.r.t. $\nu^*$, respectively.
\label{lemma:nursolution}
\end{restatable}

\begin{algorithm}[!ht]
    \caption{\textsc{AdaCubic} algorithm}
    \label{alg:mainAlg}
    \begin{algorithmic}[1]
    \State{Set $\xi_k\gets 1$, $\kappa_{\text{easy}} \in (0, 1)$, $0 < \alpha_2 < 1 \leq \alpha_1$, and $0 < \eta_1 \leq \eta_2 < 1$. } 
	\Repeat \Comment{$k$-th iteration, $k=0,1,\dots$}
	{\LineCommentCont{The function $F$, $\mathbf{B}_k$, and~$\mathbf{g}_k$ are evaluated on the same batch.}}
        {\LineCommentCont{\textsc{RootFinder} is Algorithm~\ref{alg:newton}.}}
	\State{$\mathbf{s}_{k+1}, \: \nu_{k+1} \gets \textsc{RootFinder}(\mathbf{B}_{k}, \mathbf{g}_{k},  \xi_k, \kappa_{\text{easy}})$} \label{line:innerRootFinder}
\State{Compute $\rho_k$ using \begin{equation*}
\rho_k = \frac{F(\mathbf{x}_k) - F(\mathbf{x}_k + \mathbf{s}_{k+1} )}{F(\mathbf{x}_k) - \mathfrak{m}_{\nu_{k+1}}(\mathbf{s}_{k+1})}
\end{equation*}}\label{line:rhodef}

\If{$\rho_k \geq \eta_1$}	
	\State{$\mathbf{x}_{k+1} \gets \mathbf{x}_k + \mathbf{s}_{k+1}$} 
\Else
 \State{$\mathbf{x}_{k+1} \gets \mathbf{x}_k$} 
\EndIf 
 \label{line:rhok}
    \State{Update $\xi_k$ using \begin{equation*}
    \xi_{k+1} \gets \begin{cases}
      \max \left\{\alpha_1 \norm{\mathbf{s}_{k+1}}_2^3,\: \xi_k \right\} & \mbox{if $\rho_k \geq \eta_2$} \hspace{.4cm}  \text{\Comment{VSI}} 
      \\
      \text{keep the same} \: \xi_k &  \mbox{if $\rho_k \in [\eta_1, \eta_2)$} \hfill  \text{\Comment{SI}} \\ 
      \max \left\{ \alpha_2 \norm{\mathbf{s}_{k+1}}_2^3, \epsilon_m \right\}& \mbox{if $\rho_k \leq \eta_1$}  \hspace{.9cm} \text{\Comment{UI}}
    \end{cases}
  \end{equation*} where $\epsilon_m \approx 10^{-6}$.} \label{line:updatexi}	
\Until{execution stops (e.g., after a specific number of training epochs)}
\end{algorithmic}
\end{algorithm}

\begin{algorithm}[!ht]
\caption{Find model minimizer}
\label{alg:newton}
\begin{algorithmic}[1]
\Procedure{RootFinder}{$\mathbf{B}, \mathbf{g},  \xi, \kappa_{\text{easy}}$}
\State{Set $r \gets \sqrt[3]{\xi}$} \label{lineAlgNewton:initr}
\If{$\mathbf{B}$ is positive definite} 
\State{$\nu \gets 0$} \label{lineAlgNewton:initnu1}
\Else 
\LineCommentCont{For some $\lambda_d^+ (\mathbf{B})$ barely smaller than $\lambda_d (\mathbf{B})$}
\State{$\nu \gets -2 \:\lambda_d^+ (\mathbf{B}) \: \big/ \: r$}\label{lineAlgNewton:initnu2}  
\EndIf
\State{Compute $\mathbf{s} = -(\mathbf{B} + \frac{1}{2}\nu \: r\:\mathbf{I})^{-1}\:\mathbf{g}$}
\If{$\norm{\mathbf{s}}_2^3 \leq \xi$}
\If{$\mathbf{B}$ is positive definite or $\norm{\mathbf{s}}_2^3 = \xi$}
\State{\textbf{return} $\mathbf{s}$, $\nu$}
\Else
\StatePar{Compute the eigenvector $\mathbf{u}_d$ that corresponds to the eigenvalue $\lambda_d (\mathbf{B})$. Then find the root $\alpha$ of the equation $\norm{\mathbf{s} + \alpha~\mathbf{u}_d }_2 = \xi^{1/3}$ which 
makes the model $\mathfrak{m}_{\nu}(\mathbf{s} + \alpha \: \mathbf{u}_d)$ the smallest.
}
\State{\textbf{return} $\mathbf{s} + \alpha~\mathbf{u}_d$, $\nu$}
\EndIf
\EndIf
\LineComment{The following, produces $\mathbf{s}^*$ and $\nu^*$ in Lemma~\ref{lemma:globalSol}.}
\While{$|\norm{\mathbf{s}}_2-\xi^{1/3}| \leq \kappa_{\text{easy}} \:\xi^{1/3}$}\label{lineAlgNewton:neutonalgstart} 
\LineComment{By Remark~\ref{rmrk:safequardr}, $\nu$ increases.}
\State{$\nu \gets \nu - \phi(\nu, r) \: / \: \partial_{\nu} \phi(\nu, r)$} \label{lineAlgNewton:newtonUpdate}
\State{$\mathbf{s} = -(\mathbf{B} + \frac{1}{2}\nu \: r\:\mathbf{I})^{-1}\:\mathbf{g}$}
\EndWhile \label{lineAlgNewton:neutonalgend}
\State{\textbf{return} $\mathbf{s}$, $\nu$}
\EndProcedure%
\end{algorithmic}
\end{algorithm} 

Next, we clarify the role and physical interpretation of the \textsc{AdaCubic} hyperparameters as they appear in Algorithms~\ref{alg:mainAlg} and~\ref{alg:newton}:

\begin{itemize}
\item \textbf{${\eta_1}$ (acceptance threshold).}
$\eta_1 \in (0,1)$ is the minimum ratio between the actual loss reduction and the predicted reduction of the cubic model required to accept a step. If $\rho_k \ge \eta_1$, the step is considered successful and the parameters are updated. This parameter ${\eta_1}$ controls how cautiously the algorithm accepts update steps. Smaller values make acceptance easier, while larger values enforce stricter agreement between the cubic model $\mathfrak{m}_{\nu_{k+1}} (\mathbf{s}_{k+1})$ and the objective function $F(\mathbf{x}_{k}+\mathbf{s}_{k+1})$.

\item \textbf{${\eta_2}$ (very successful threshold).}
$\eta_2 \geq \eta_1$ identifies very successful iterations. When $\rho_k \ge \eta_2$, the effective trust-region boundary is expanded, allowing larger steps in subsequent iterations. This mechanism accelerates convergence when the cubic model $\mathfrak{m}_{\nu_{k+1}} (\mathbf{s}_{k+1})$ is highly accurate.

\item \textbf{${\alpha_1}$ (expansion factor).}
$\alpha_1 \geq 1$ controls the increase of the trust-region parameter $\xi_k$ after very successful iterations, thereby expanding the effective trust-region boundary.

\item \textbf{${\alpha_2}$ (shrinkage factor).} 
$\alpha_2 \in (0,1)$ decreases the trust-region boundary after unsuccessful iterations ($\rho_k \le \eta_1$). By shrinking the trust-region boundary, more conservative updates are obtained, thereby improving robustness in regions where the cubic model $\mathfrak{m}_{\nu_{k+1}}(\mathbf{s})$ is less accurate.

\item \textbf{${\kappa_{\text{easy}}}$ (root-finding tolerance).} 
$\kappa_{\text{easy}} \in (0,1)$ specifies the error tolerance to terminate the Newton iterations when solving the cubic subproblem in Algorithm~\ref{alg:newton}. ${\kappa_{\text{easy}}}$ determines how close the norm of the computed step should be to the trust-region boundary before the termination of the dual variable calculation. Smaller values enforce higher accuracy in solving the subproblem, while larger values favor computational efficiency.

\end{itemize}

Overall, $\eta_1$ and $\eta_2$ govern step acceptance, $\alpha_1$ and $\alpha_2$ regulate updates of the trust-region boundary, and $\kappa_{\text{easy}}$ balances accuracy and efficiency in the inner solver of Algorithm~\ref{alg:newton}. \textsc{AdaCubic} adaptively computes the dual parameter $\nu_{k+1}$, which determines the step $\mathbf{s}_{k+1}$, the acceptance ratio $\rho_k$, and consequently the evolution of the trust-region parameter $\xi_k$. The dual variable $\nu_{k+1}$ encodes local curvature information through the Hessian approximation and acts as an adaptive term in the cubic subproblem. This relationship enables an automatic adjustment of $\xi_k$, allowing \textsc{AdaCubic} to respond effectively to the local geometry of the non-convex loss landscape and to achieve competitive performance across the benchmarks in Section~\ref{sec:expeval}.

\section{Experimental Evaluation}\label{sec:expeval}

Experiments are conducted on computer vision, natural language processing, and signal processing tasks, where the results obtained with the proposed \textsc{AdaCubic} optimizer are compared with those obtained with the \textsc{SGD}, \textsc{Adam}, and \textsc{AdaHessian} optimizers. The natural language processing experiments are conducted using the Hugging Face Transformers library~\citep{transformers2020}. For \textsc{SGD}, \textsc{Adam}, and \textsc{AdaHessian}, the Learning Rate (LR) is fine-tuned. For \textsc{AdaCubic}, the parameters $\eta_1=0.05$, $\eta_2=0.75$, $\alpha_1=2.5$,~$\alpha_2=0.25$, and~$\kappa_{\text{easy}} = 0.01$ are chosen universally in the experimental evaluation. These parameters are chosen based on the analysis in~\citep[Section 17.1]{conn2000}. Table~\ref{tab:adacubic_params} summarizes the universal hyperparameter values used by \textsc{AdaCubic} across all benchmarks.

\begin{table}[!ht]
\centering
\caption{Universal \textsc{AdaCubic} hyperparameter settings.
All hyperparameters in Algorithm~\ref{alg:newton} are fixed across benchmarks.
$\epsilon_m$ denotes a numerical safeguard used in Algorithm~\ref{alg:mainAlg}.}
\label{tab:adacubic_params}
\begin{tabular}{lcccccc}
\toprule
& \multicolumn{6}{c}{\textsc{AdaCubic} Hyperparameters} \\
\cmidrule(lr){2-7}
\textbf{Hyperparameter}
& $\eta_1$
& $\eta_2$
& $\alpha_1$
& $\alpha_2$
& $\kappa_{\text{easy}}$
& $\epsilon_m$ \\
\midrule
\textbf{Assigned Value}
& $0.05$
& $0.75$
& $2.5$
& $0.25$
& $0.01$
& $10^{-6}$ \\
\bottomrule
\end{tabular}
\end{table}
Tables~\ref{tab:exp_overview} and~\ref{tab:opt_hparams} summarize the experimental configurations for each benchmark, including datasets, model architectures, optimizers, and LR settings.

\begin{table}[!ht]
\centering
\caption{Summary of model architectures, training settings, and optimizers used in all experiments.}
\label{tab:exp_overview}
\resizebox{\columnwidth}{!}{
\begin{tabular}{lcccccc}
\toprule
\textbf{Task} & \textbf{Dataset} & \textbf{Model} & \textbf{Batch} & \textbf{Epochs} & \textbf{Optimizers} \\
\midrule
\multirow{2}{*}{CV}
& \texttt{CIFAR-10}   
& \texttt{ResNet20} / \texttt{ResNet32} 
& 256 & 500 
& \textsc{SGD}, \textsc{Adam}, \textsc{AdaHessian}, \textsc{AdaCubic} \\

& \texttt{CIFAR-100}  
& \texttt{ResNet18}            
& 256 & 200 
& \textsc{SGD}, \textsc{Adam}, \textsc{AdaHessian}, \textsc{AdaCubic} \\
\midrule
\multirow{3}{*}{NLU}
& \texttt{SST-2}, \texttt{QNLI}, \texttt{RTE}, \texttt{WNLI} 
& \texttt{SqueezeBERT} 
& 32 & 15 
& \textsc{SGD}, \textsc{AdaHessian}, \textsc{AdaCubic} \\

& \texttt{MRPC}, \texttt{QQP}             
& \texttt{SqueezeBERT} 
& 32 & 15 
& \textsc{SGD}, \textsc{AdaHessian}, \textsc{AdaCubic} \\

& \texttt{STS-B}, \texttt{MNLI}           
& \texttt{SqueezeBERT} 
& 32 & 15 
& \textsc{SGD}, \textsc{AdaHessian}, \textsc{AdaCubic} \\
\midrule
\multirow{2}{*}{LM}
& \texttt{WikiText-2} 
& \texttt{RoBERTa} / \texttt{BERT} / \texttt{DistilBERT} 
& 8 & 6 
& \textsc{SGD}, \textsc{AdaHessian}, \textsc{AdaCubic} \\

& \texttt{PTB}        
& \texttt{RoBERTa} / \texttt{BERT} / \texttt{DistilBERT} 
& 8 & 6 
& \textsc{SGD}, \textsc{AdaHessian}, \textsc{AdaCubic} \\
\midrule
CMI
& \texttt{VISION} 
& \texttt{ResNet18} 
& 256 & 100 
& \textsc{Adam}, \textsc{AdaCubic} \\
\bottomrule
\end{tabular}}
\end{table}

\begin{table}[!ht]
\centering
\caption{Summary of LRs used in all experiments. For the CV benchmark, LRs are decayed by a factor of $10$ at epochs $80$ and $120$ on \texttt{CIFAR-10}, and by a factor of $20$ at epochs $60$, $120$, and $160$ on \texttt{CIFAR-100}. For NLU, LM, and CMI benchmarks, no LR decay is applied. For CMI, LR is decayed by a factor of $10$ at epochs $80$ and $120$. \textsc{AdaCubic} is used with a fixed universal parameter set and does not require LR tuning.}
\label{tab:opt_hparams}
\begin{tabular}{lllcc}
\toprule
\textbf{Optimizer} & \textbf{Task} & \textbf{Dataset(s)} & \textbf{Initial LR} & \textbf{LR Schedule / Tuning} \\
\midrule
\textsc{SGD}
& CV
& \texttt{CIFAR-10} / \texttt{CIFAR-100}
& $0.1$
& Step decay (tuned) \\

\textsc{Adam}
& CV
& \texttt{CIFAR-10} / \texttt{CIFAR-100}
& $10^{-3}$
& Step decay (tuned) \\

\textsc{AdaHessian}
& CV
& \texttt{CIFAR-10} / \texttt{CIFAR-100}
& $0.15$
& Step decay (tuned) \\

\textsc{AdaCubic}
& CV
& \texttt{CIFAR-10} / \texttt{CIFAR-100}
& no LR
& Universal parameters \\

\midrule
\textsc{SGD}
& NLU
& \texttt{SST-2}, \texttt{QNLI}, \texttt{RTE}, \texttt{WNLI}
& $2{\times}10^{-2}$
& Tuned \\

\textsc{SGD}
& NLU
& \texttt{STS-B}
& $2{\times}10^{-3}$
& Tuned \\

\textsc{AdaHessian}
& NLU
& \texttt{SST-2}, \texttt{QNLI}, \texttt{STS-B}, \texttt{MNLI}
& $2{\times}10^{-3}$
& Tuned \\

\textsc{AdaHessian}
& NLU
& \texttt{MRPC}, \texttt{RTE}
& $2{\times}10^{-4}$
& Tuned \\

\textsc{AdaHessian}
& NLU
& \texttt{WNLI}
& $2{\times}10^{-2}$
& Tuned \\

\textsc{AdaCubic}
& NLU
& All \texttt{GLUE} tasks
& no LR
& Universal parameters \\

\midrule
\textsc{SGD}
& LM
& \texttt{WikiText-2}, \texttt{PTB}
& $5{\times}10^{-3}$
& Tuned \\

\textsc{AdaHessian}
& LM
& \texttt{WikiText-2} (all models)
& $5{\times}10^{-4}$
& Tuned \\

\textsc{AdaHessian}
& LM
& \texttt{PTB} (\texttt{RoBERTa})
& $5{\times}10^{-3}$
& Tuned \\

\textsc{AdaHessian}
& LM
& \texttt{PTB} (\texttt{BERT}, \texttt{DistilBERT})
& $5{\times}10^{-4}$
& Tuned \\

\textsc{AdaCubic}
& LM
& \texttt{WikiText-2}, \texttt{PTB}
& no LR
& Universal parameters \\

\midrule
\textsc{Adam}
& CMI
& \texttt{VISION}
& $10^{-4}$
& Tuned \\

\textsc{AdaCubic}
& CMI
& \texttt{VISION}
& no LR
& Universal parameters \\
\bottomrule
\end{tabular}
\end{table}

\noindent\textbf{Computer Vision (CV).} To prove the effectiveness of \textsc{AdaCubic}, experiments are conducted using \texttt{CIFAR-10} and \texttt{CIFAR-100} datasets~\citep{krizhevsky2009}. The experimental results are summarized in Table~\ref{tbl:imageclassification}. In all experiments, a batch size of 256 is used. The mean accuracy and standard deviation (std) over five runs are reported for each experiment. The number of epochs used to train the models on \texttt{CIFAR-10} and \texttt{CIFAR-100} is 500 and 200, respectively. In addition, the optimizers are fine-tuned w.r.t. the initial LR and the decaying LR scheme. For \textsc{SGD}, \textsc{Adam}, and \textsc{AdaHessian}, the initial learning rates are 0.1, 0.001, and 0.15. Furthermore, for \textsc{AdaHessian}, $\beta_1$ and $\beta_2$ are set to 0.9 and 0.999, respectively. On \texttt{CIFAR-10}, the LR is decayed by a factor of 10 at epochs 80 and 120, while on \texttt{CIFAR-100}, the LR is decayed by a factor of 20 at epochs 60, 120, and 160. In addition, spatial averaging~\citep{ya2021} is used for \textsc{AdaCubic} and \textsc{AdaHessian} on \texttt{CIFAR-100}. The entries corresponding to the best accuracy are marked in bold. $\Delta$ reports the accuracy differences between \textsc{AdaCubic} and the strongest competing optimizer in each setting. When spatial averaging is used, the accuracy is shown in gray.

On the \texttt{CIFAR-10} dataset, both \textsc{AdaHessian} and \textsc{AdaCubic} demonstrate higher accuracy than conventional optimization methods like \textsc{SGD} and \textsc{Adam}. It is worth noting that, while both methods excel, \textsc{AdaHessian} achieves a slight edge in accuracy over \textsc{AdaCubic} for \texttt{ResNet20} and \texttt{ResNet32} by 0.15\% and 0.5\%, respectively. This performance distinction underscores the effectiveness of \textsc{AdaCubic} and positions it as a formidable competitor to \textsc{AdaHessian} in enhancing model accuracy on the \texttt{CIFAR-10} dataset.

On the \texttt{CIFAR-100} dataset without spatial averaging, \textsc{AdaCubic} falls behind \textsc{SGD}, \textsc{Adam}, and \textsc{AdaHessian} by margins of 0.81\%, 0.23\%, and 0.64\%, respectively. However, with spatial averaging, both \textsc{AdaHessian} and \textsc{AdaCubic} achieve improved accuracy. This comparative analysis highlights \textsc{AdaCubic}'s distinct performance characteristics, demonstrating its unique capabilities relative to other optimizers in challenging scenarios, such as on the \texttt{CIFAR-100} dataset.

\begin{table}[!ht]
\caption{Accuracy (\%) and std of the accuracy measures for \texttt{ResNet18/20/32} models on \texttt{CIFAR-10} and \texttt{CIFAR-100} datasets. $\Delta$ reports the gap between the strongest competing optimizer and \textsc{AdaCubic}.}
\centering
\begin{tabular}{cccc}  
\cmidrule{2-4}
& \multicolumn{2}{c}{\texttt{CIFAR-10}} & \texttt{CIFAR-100} \\ \cmidrule{2-4}  & \texttt{ResNet20} & \texttt{ResNet32} & \texttt{ResNet18}  \\
\midrule
\textsc{SGD} &  88.52 $\pm$ 0.24  & 89.02 $\pm$ 0.20  & \textbf{72.62} $\pm$ 0.002 \\ 
\textsc{Adam} &  90.26 $\pm$ 0.19  & 91.24 $\pm$ 0.20  & 72.04 $\pm$ 0.13 \\ 
\midrule
\multirow{2}{*}{\textsc{AdaHessian}} & \textbf{91.64} $\pm$ 0.46 & \textbf{93.15} $\pm$ 0.12   & 72.45 $\pm$ 0.16 \\  &-&-& \colorbox{gray!30}{72.59 $\pm$ 0.271} \\
\midrule
\multirow{2}{*}{\textsc{AdaCubic}}  &  91.49 $\pm$ 0.46  & 92.65 $\pm$ 0.19   & 71.81 $\pm$ 0.003 \\  
 & - & - &  \colorbox{gray!30}{72 $\pm$ 0.337} \\
 \midrule
\multirow{2}{*}{$\Delta$}  &  0.15 $\pm$ 0.36   & 0.5 $\pm$ 0.07  & 0.81 $\pm$ 0.001 \\  
 & - & - &  \colorbox{gray!30}{0.59 $\pm$ 0.066} \\
\bottomrule
\end{tabular}
\label{tbl:imageclassification}
\end{table}

\begin{figure}[!ht]
\centering
\rule[-.5cm]{0cm}{4cm}
\subfloat[\label{fig:lossCifar10a}]{\includegraphics[width=0.8\linewidth]{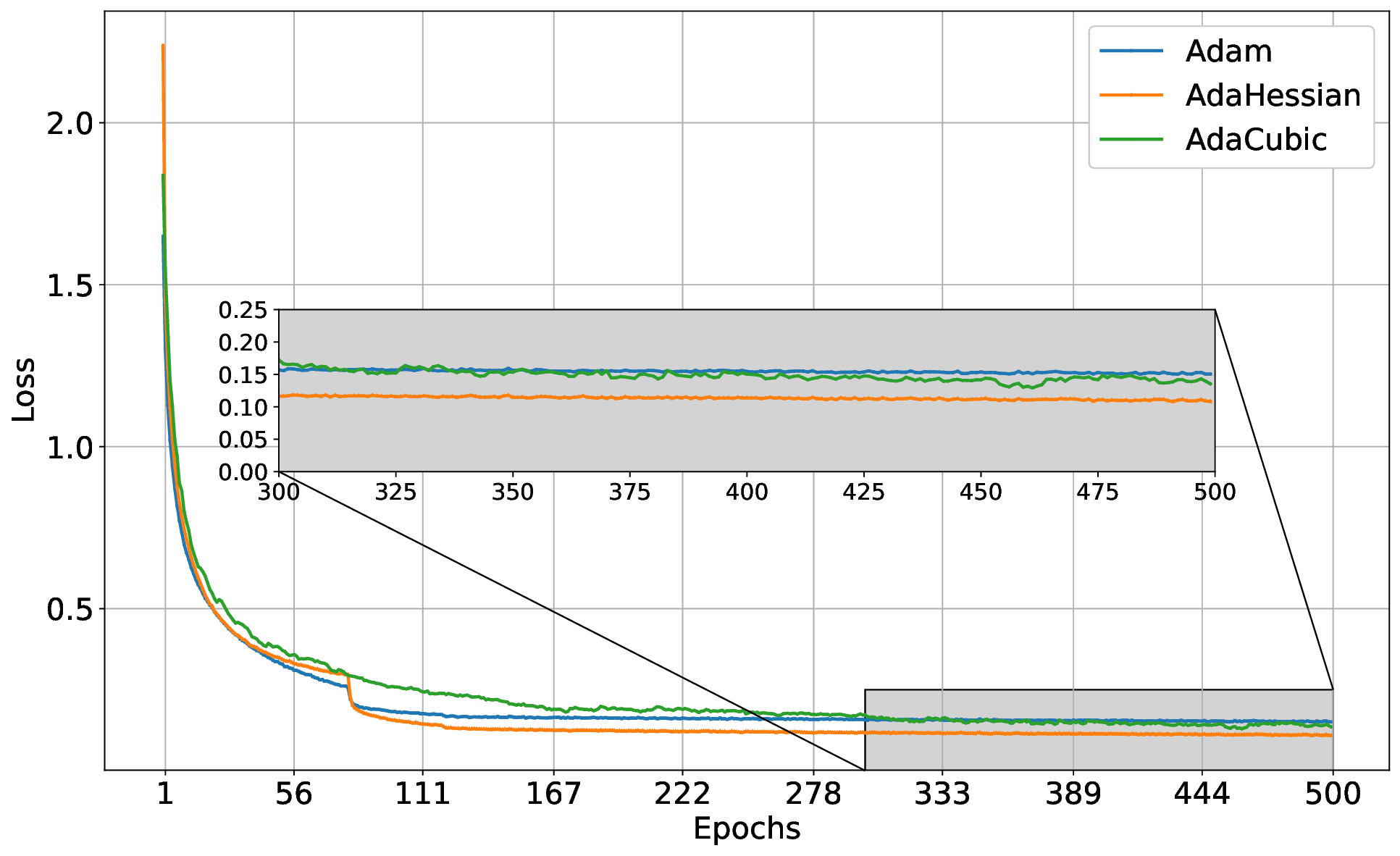}}
\rule[-.5cm]{0cm}{4cm}
\subfloat[\label{fig:lossCifar10b}]{\includegraphics[width=0.8\linewidth]{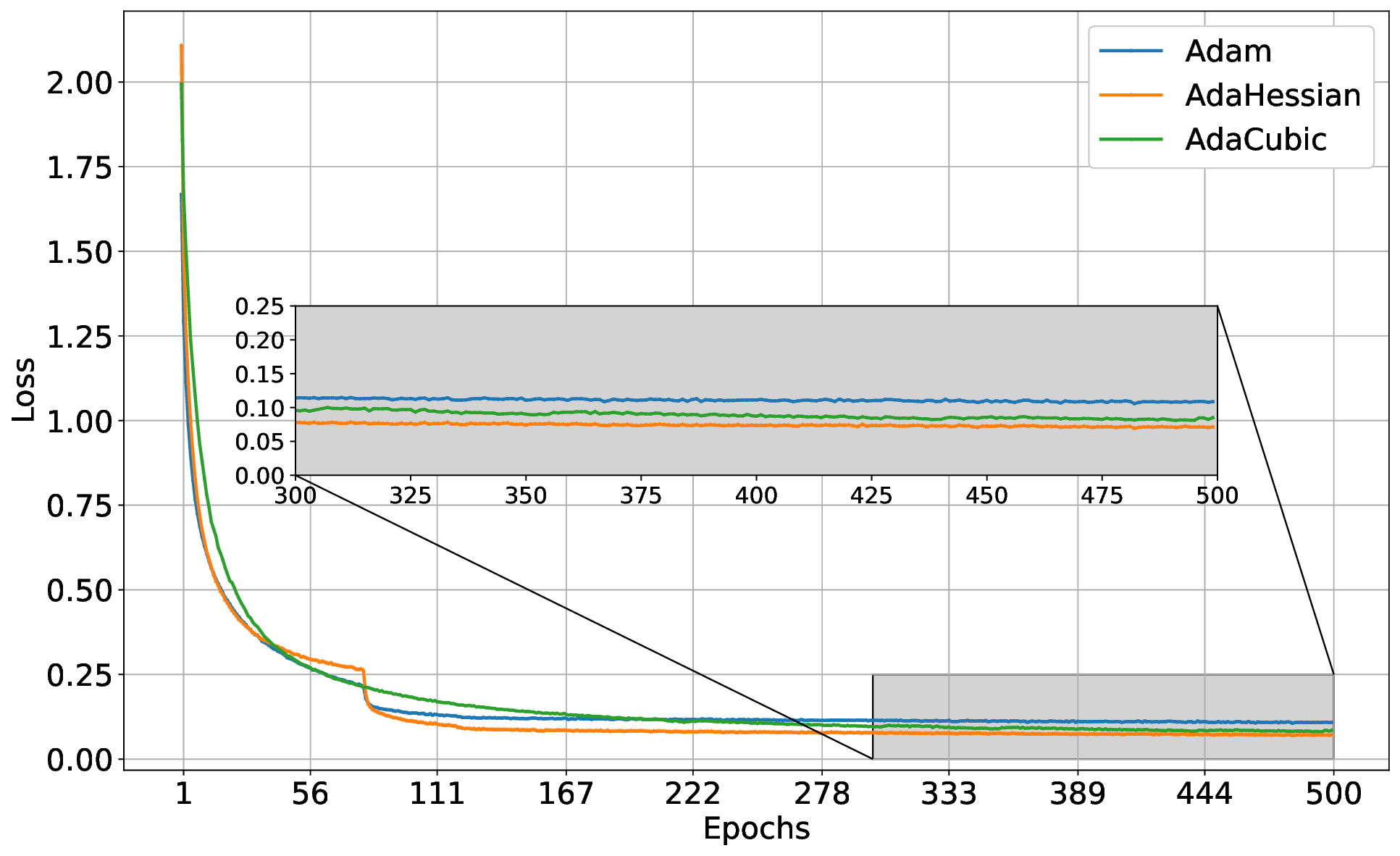}}
\rule[-.5cm]{0cm}{4cm}
\caption{Training loss curve of \texttt{ResNet20} (top) and \texttt{ResNet32} (bottom) on \texttt{CIFAR-10} for \textsc{Adam}, \textsc{AdaHessian}, and \textsc{AdaCubic} optimizers.}
\label{fig:lossCifar10}
\end{figure}

Figure~\ref{fig:lossCifar10} depicts the training loss of \texttt{ResNet20} (top) and \texttt{ResNet32} (bottom) on \texttt{CIFAR-10} for \textsc{Adam}, \textsc{AdaHessian}, and \textsc{AdaCubic} optimizers. As can be seen, the losses of \textsc{Adam} and \textsc{AdaHessian} decrease dramatically at epoch 80, when the LR has decayed by a factor of 10. 
As can be seen, only by using an adaptive LR can the training loss reduction of \textsc{Adam} and \textsc{AdaHessian} match that of \textsc{AdaCubic}. In the last epochs, the loss of \textsc{AdaCubic} is lower than that of \textsc{Adam} and higher than that of \textsc{AdaHessian}. It should be noted that, in all experiments, \textsc{AdaCubic} is used with the same set of parameters and achieves competitive performance compared to the remaining fine-tuned optimizers.

\begin{table}[ht!]
\caption{Figures of merit on \texttt{GLUE} benchmark using \textsc{SGD}, \textsc{AdaHessian}, and \textsc{AdaCubic} optimizers on natural language understanding tasks. $\Delta$ reports the gap between the strongest competing optimizer and \textsc{AdaCubic}}.
\centering
\begin{tabular}{ccccc}  
\toprule
\multirow{2}{*}{\textbf{Dataset}} & \textsc{SGD} & \textsc{AdaHessian} & \textsc{AdaCubic} & $\Delta$ \\
\cmidrule{2-5}
& \multicolumn{4}{c}{\textbf{Accuracy (\%)}} \\
\cmidrule{1-5}
\texttt{SST-2}  & \textbf{91.62} & 90.71 & 90.71 & 0.91 \\
\texttt{QNLI}   & \textbf{90.37} & 89.47 & 90.01 & 0.36 \\
\texttt{RTE}    & \textbf{70.39} & 64.98 & \textbf{70.39} & 0.00 \\
\texttt{WNLI}   & \textbf{56.33} & \textbf{56.33} & \textbf{56.33} & 0.00 \\
\midrule
& \multicolumn{4}{c}{$F_1$ / \textbf{Accuracy (\%)}} \\
\cmidrule{1-5}
\texttt{MRPC}  & \textbf{0.9094 / 87.25} & 0.8562 / 78.18 & {0.9042 / 86.76} & 
$0.0052 / 0.49$ \\
\texttt{QQP}   & \textbf{0.8775 / 90.89} & 0.8742 / 90.82 & 0.8723 / 90.40 &
$0.0052 / 0.49$ \\
\midrule
& \multicolumn{4}{c}{\textbf{Pearson / Spearman Corr.}} \\
\cmidrule{1-5}
\texttt{STS-B} & \textbf{0.8863 / 0.8845} & 0.8786 / 0.8735 & {0.8832 / 0.8814} & 
$0.0031 / 0.0031$ \\
\midrule
& \multicolumn{4}{c}{\textbf{Matched / Mismatched Accuracy (\%)}} \\
\cmidrule{1-5}
\texttt{MNLI}  & \textbf{82.45 / 82.05} & 81.65 / 81.57 & {81.88 / 81.89} &
$0.57 / 0.16$ \\
\bottomrule
\end{tabular}
\label{tbl:gluetask}
\end{table}

On \texttt{CIFAR-10}, \textsc{AdaCubic} consistently outperforms first-order methods (\textsc{SGD}, \textsc{Adam}) and ranks second to \textsc{AdaHessian}, with very small gaps of $0.15\%$ and $0.5\%$ for \texttt{ResNet20} and \texttt{ResNet32}, respectively, as summarized in Table~\ref{tbl:imageclassification}. On \texttt{CIFAR-100} without spatial averaging, \textsc{AdaCubic} trails the best-performing optimizer by at most $0.81\%$. Due to its larger number of classes and increased classification difficulty, CIFAR-100 will possibly lead to optimization regimes with stronger parameter interactions. Since \textsc{AdaCubic}, like \textsc{AdaHessian}, relies on a diagonal approximation of the Hessian, it does not explicitly capture such off-diagonal curvature effects, which may partially explain the observed gap. Importantly, when spatial averaging is applied, the performance of \textsc{AdaCubic} improves and becomes closer to that of \textsc{AdaHessian} and \textsc{SGD}, confirming that part of the gap is related to high-variance curvature estimation.

\noindent\textbf{Natural Language Understanding (NLU).} Table~\ref{tbl:gluetask} summarizes the results on the natural language understanding task. The \texttt{GLUE} benchmark~\citep{wang2018glue} is used to train the \texttt{SqueezeBERT}~\citep{Iandola2020} model for 15 epochs. For \textsc{SGD} and \textsc{AdaHessian}, the initial LR is fine-tuned in all datasets. For \textsc{SGD}, the initial LR is set to $2\cdot 10^{-2}$ for all datasets except from \texttt{STS-B} where it is set to $2\cdot 10^{-3}$. For \textsc{AdaHessian}, the initial LR is set to $2\cdot 10^{-3}$ for \texttt{SST-2}, \texttt{STS-B}, \texttt{MNLI}, and \texttt{QNLI}, to $2\cdot 10^{-4}$ for \texttt{MRPC} and \texttt{RTE}, and to $2\cdot 10^{-2}$ for \texttt{WNLI}. 

The default parameters of the \texttt{SqueezeBERT} model can be found in the official Hugging Face library\footnote{\url{https://github.com/huggingface/transformers/tree/main/examples/pytorch/text-classification}}. The dataset acronyms in the Hugging Face library are \texttt{SST-2}, \texttt{QNLI}, \texttt{RTE}, \texttt{WNLI}, \texttt{MRPC}, \texttt{QQP}, \texttt{STS-B}, and \texttt{MNLI}, while the model acronym is \texttt{squeezebert/squeezebert-uncased}.

To simplify the experimental evaluation, the experiments are divided into four groups, each corresponding to a different performance measure. Group 1 consists of the \texttt{SST-2}, \texttt{QNLI}, \texttt{RTE}, and \texttt{WNLI} datasets. Group 2 consists of the \texttt{MRPC} and \texttt{QQP} datasets, while groups 3 and 4 consist of the \texttt{SST-B} and \texttt{MNLI} datasets, respectively. The entries corresponding to the best metrics are marked in bold. $\Delta$ reports the accuracy differences between \textsc{AdaCubic} and the strongest competing optimizer in each setting.

\begin{itemize}
    \item \textit{Group 1.} Concerning accuracy measure, \textsc{AdaCubic} and \textsc{AdaHessian} demonstrate the same performance on \texttt{SST-2}, while \textsc{SGD} performs better by 0.91\%. On \texttt{QNLI}, \textsc{SGD} outperforms \textsc{AdaCubic} by 0.36\%, while \textsc{AdaCubic} outperforms \textsc{AdaHessian} by 0.54\%. On \texttt{RTE}, \textsc{AdaCubic} and \textsc{SGD} achieve the same performance, while \textsc{AdaHessian} is outperformed by 5.41\%. On \texttt{WNLI}, all optimizers achieve the same performance. 
    Overall, the mean accuracies achieved by \textsc{SGD}, \textsc{AdaHessian}, and \textsc{AdaCubic} are 77.17\%, 75.37\%, and 76.86\%, respectively. It can be observed that, on average, \textsc{SGD} outperforms \textsc{AdaCubic} by 0.31\%, while \textsc{AdaCubic} outperforms \textsc{AdaHessian} by 1.5\%.
    
    \item \textit{Group 2.} Concerning $F_1$ measure on \texttt{MRPC}, \textsc{SGD} outperforms \textsc{AdaCubic} by  0.0052, while \textsc{AdaCubic} outperforms \textsc{AdaHessian} by 0.048. On the same dataset, \textsc{SGD} achieves higher accuracy than \textsc{AdaCubic} by 0.49\%, whereas \textsc{AdaCubic} outperforms \textsc{AdaHessian} by 8.58\%. Concerning $F_1$ measure on \texttt{QQP}, \textsc{SGD} outperforms \textsc{AdaHessian} by 0.0052, while \textsc{AdaHessian} outperforms \textsc{AdacCubic} by 0.0019. On the same dataset, \textsc{SGD} achieves higher accuracy than \textsc{Adahessian} by 0.07\%, while \textsc{AdaHessian} outperforms \textsc{AdaCubic} by 0.42\%. 
    Overall, the mean $F_1$ values achieved by \textsc{SGD}, \textsc{AdaHessian}, and \textsc{AdaCubic} are 0.89345, 0.8652, and 0.88825, respectively, while the mean accuracies are 89.07\%, 84.5\%, and 88.58\%, respectively. This way, on average, \textsc{SGD} outperforms \textsc{AdaCubic} by 0.0052 and 0.49\%, on $F_1$ and accuracy measures, respectively, while \textsc{AdaCubic} outperforms \textsc{AdaHessian} by 0.02305 and 4.08\%, respectively.

    \item \textit{Group 3.} Concerning Pearson correlation index, \textsc{SGD} outperforms \textsc{AdaCubic} by 0.0031, while \textsc{AdaCubic} outperforms \textsc{AdaHessian} by 0.0046. Regarding the Spearman correlation index, \textsc{SGD} outperforms \textsc{AdaCubic} by 0.0031, while \textsc{AdaCubic} outperforms \textsc{AdaHessian} by 0.0079.

    \item \textit{Group 4.} Concerning matched accuracy~\citep{wang2018glue}, \textsc{SGD} outperforms \textsc{AdaCubic} by 0.57\%, while \textsc{AdaCubic} outperforms \textsc{AdaHessian} by 0.23\%. Concerning mismatched accuracy~\citep{wang2018glue}, \textsc{SGD} outperforms \textsc{AdaCubic} by 0.16\%, while \textsc{AdaCubic} outperforms \textsc{AdaHessian} by 0.32\%. 
\end{itemize}

It is worth noting that \textsc{AdaCubic} exhibits the second-best performance with a pre-fixed universal set of parameters, while \textsc{SGD} and \textsc{AdaHessian} are fine-tuned w.r.t. the initial LR.

\noindent\textbf{Language Modeling (LM).}
Tables~\ref{tbl:languagemodelling1} and~\ref{tbl:languagemodelling2} summarize the results on the language modeling, where perplexity~\citep{jelinek1977} is used as an evaluation metric. \texttt{PTB}~\citep{marcinkiewicz1994} and \texttt{wikitext-2}~\citep{merity2017pointer} datasets are used to train \texttt{RoBERTa}~\citep{liu2019roberta}, \texttt{BERT}~\citep{bert2019}, and \texttt{DistilBERT}~\citep{Sanh2019} models with \textsc{SGD}, \textsc{AdaHessian}, and \textsc{AdaCubic} optimizers.

\begin{table}[!ht]
\caption{Perplexity achieved by \textsc{SGD}, \textsc{AdaHessian}, and \textsc{AdaCubic} on \texttt{wikitext-2} dataset.}
\centering
\begin{tabular}{cccc}  
\toprule
\textbf{Optimizer}& \texttt{RoBERTa}& \texttt{BERT} & \texttt{DistilBERT} \\
\midrule
\textsc{SGD} & \textbf{3.547} &  13.380 & \textbf{6.118} \\ 
\textsc{AdaCubic} & 3.756 & \textbf{5.759} & 6.565\\ 
\textsc{AdaHessian} & 4.374 & 16.151 & 6.822\\ 
\bottomrule
\end{tabular}
\label{tbl:languagemodelling1}
\end{table}

\begin{table}[ht!]
\caption{Perplexity  using \textsc{SGD}, \textsc{AdaHessian}, and \textsc{AdaCubic} on the \texttt{PTB} dataset.}
\centering
\begin{tabular}{cccc}  
\toprule
\textbf{Optimizer}& \texttt{RoBERTa}& \texttt{BERT} & \texttt{DistilBERT} \\
\midrule
\textsc{SGD} & \textbf{4.345} & 17.344 & 8.299  \\ 
\textsc{AdaCubic} & 5.145 & \textbf{14.170} & \textbf{7.334}\\ 
\textsc{AdaHessian} & 7.582 & 20.851 & 10.182\\ 
\bottomrule
\end{tabular}
\label{tbl:languagemodelling2}
\end{table}

The initial LR of \textsc{SGD} is fine-tuned to $5\cdot10^{-3}$ for all models and both datasets. For \textsc{AdaHessian}, the initial LR is fine-tuned to $5\cdot10^{-4}$ for all models on \texttt{wikitext-2} dataset. On \texttt{PTB} dataset, the initial LR of \textsc{AdaHessian} optimizer is set to $5\cdot 10^{-3}$ to train \texttt{RoBERTa} model, while the remaining models are trained with initial LR $5\cdot 10^{-4}$. 
The remaining parameters for the trained models can be found in the official Hugging Face library\footnote{\url{https://github.com/huggingface/transformers/tree/main/examples/pytorch/language-modeling}}. The dataset acronyms in the Hugging Face library are \texttt{ptb\_text\_only} and \texttt{wikitext-2-raw-v1}. In contrast, the model acronyms are \texttt{roberta-base}, \texttt{bert-base-cased}, and~ \texttt{distilbert-base-uncased}.

\begin{figure*}[!ht]
    \centering
    \rule[-.5cm]{0cm}{4cm}
    \subfloat[\centering \texttt{RoBERTa}]{{\includegraphics[width=5.2cm]{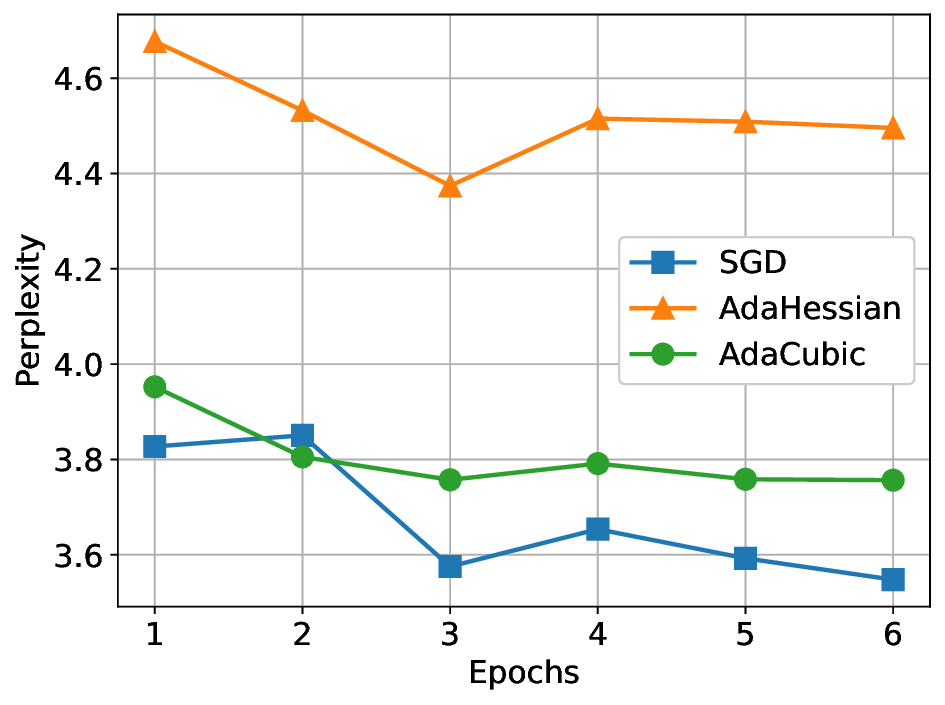} }}%
    \rule[-.5cm]{0cm}{4cm}
    \subfloat[\centering \texttt{BERT}]{{\includegraphics[width=5.2cm]{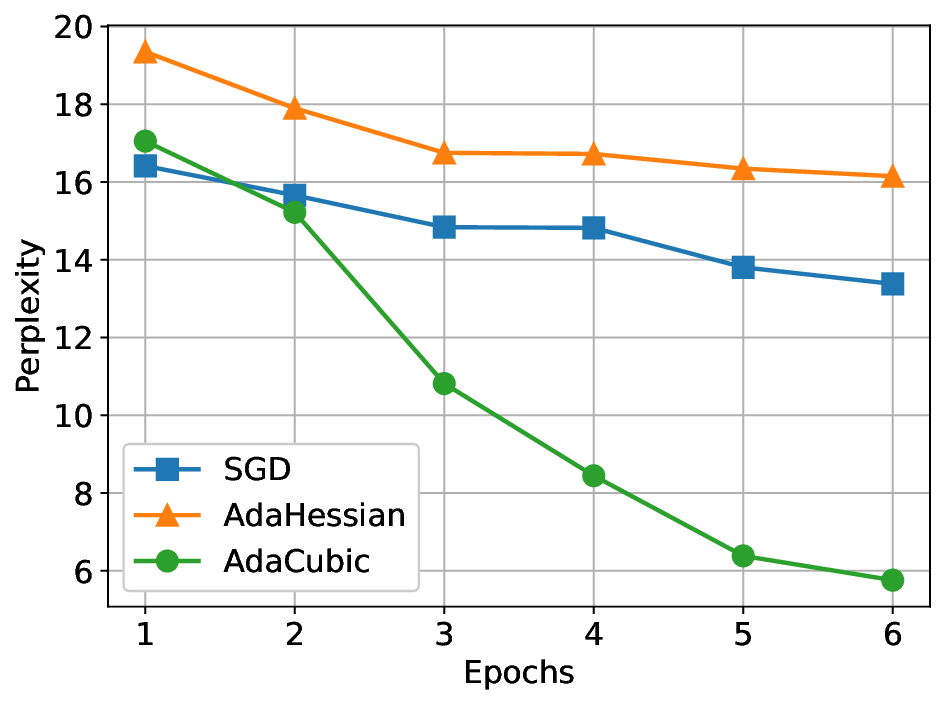} }}%
    \rule[-.5cm]{0cm}{4cm}
    \subfloat[\centering \texttt{DistilBERT}]{{\includegraphics[width=5.2cm]{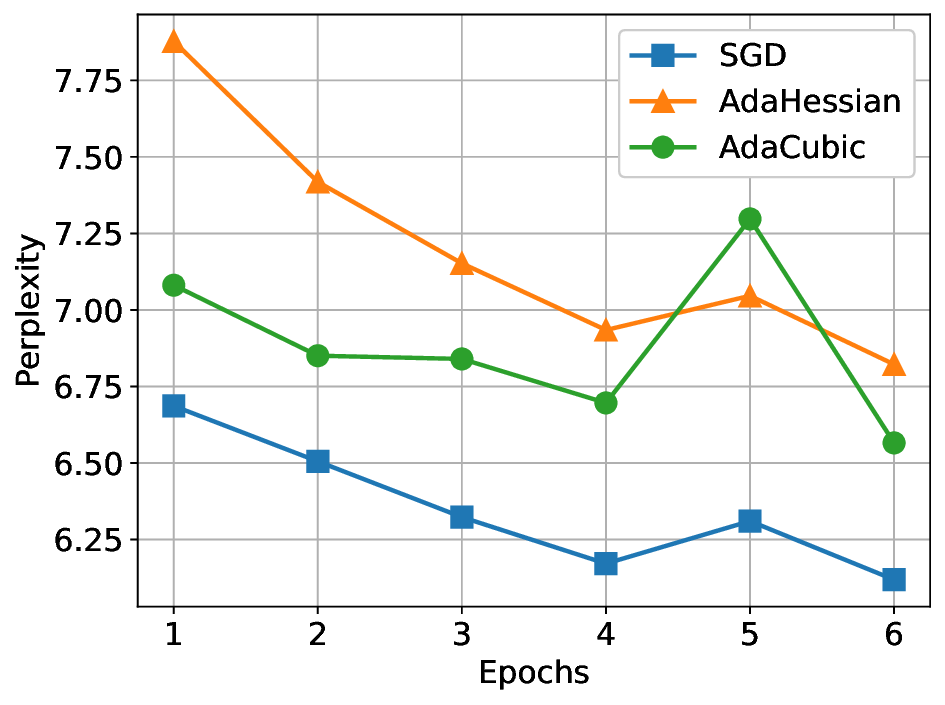} }}%
    \rule[-.5cm]{0cm}{4cm}
    \caption{Perplexity vs. epochs for  \texttt{RoBERTa}, \texttt{BERT}, and \texttt{DistilBERT} models on \texttt{wikitext-2} dataset.}%
    \label{fig:languagemodelling1}%
\end{figure*}
\begin{figure*}[!ht]
    \centering\rule[-.5cm]{0cm}{4cm}
    \subfloat[\centering \texttt{RoBERTa}]{{\includegraphics[width=5.2cm]{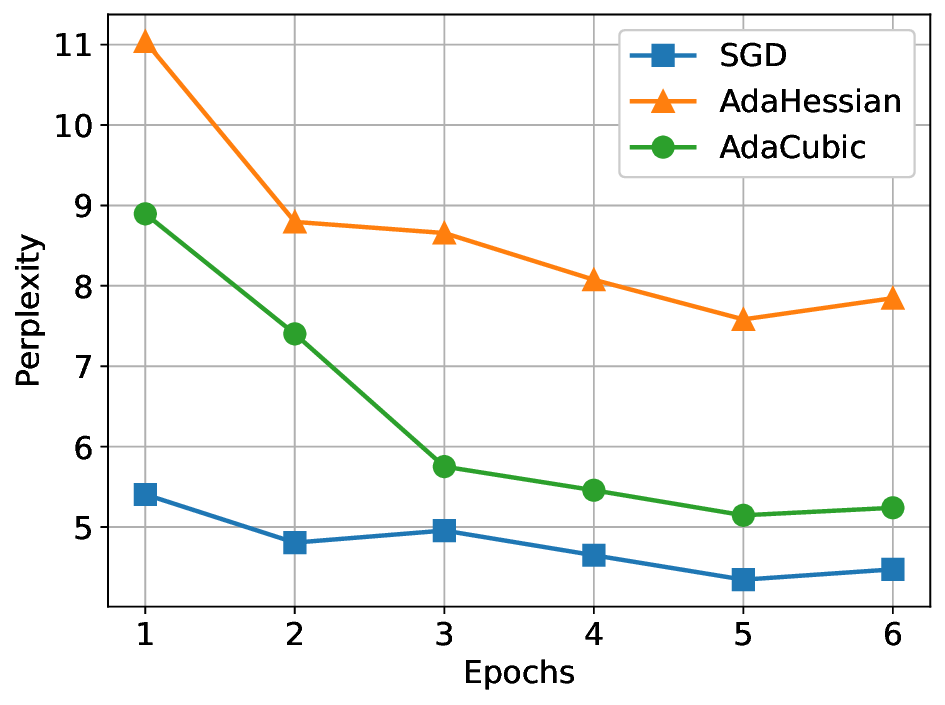} }}%
    \rule[-.5cm]{0cm}{4cm}
    \subfloat[\centering \texttt{BERT}]{{\includegraphics[width=5.2cm]{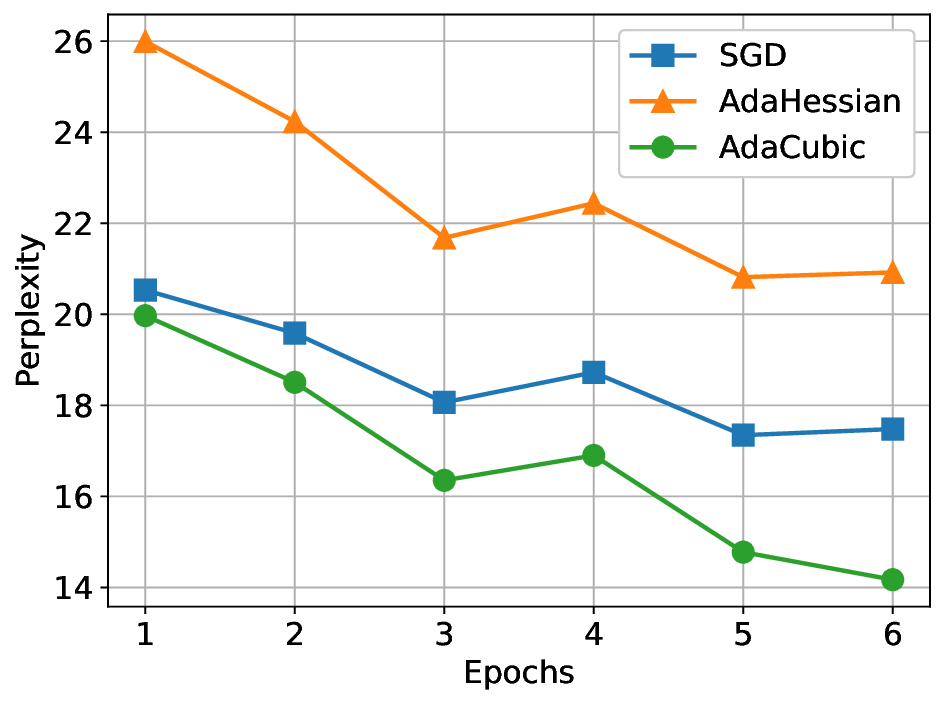} }}%
    \rule[-.5cm]{0cm}{4cm}
    \subfloat[\centering \texttt{DistilBERT}]{{\includegraphics[width=5.2cm]{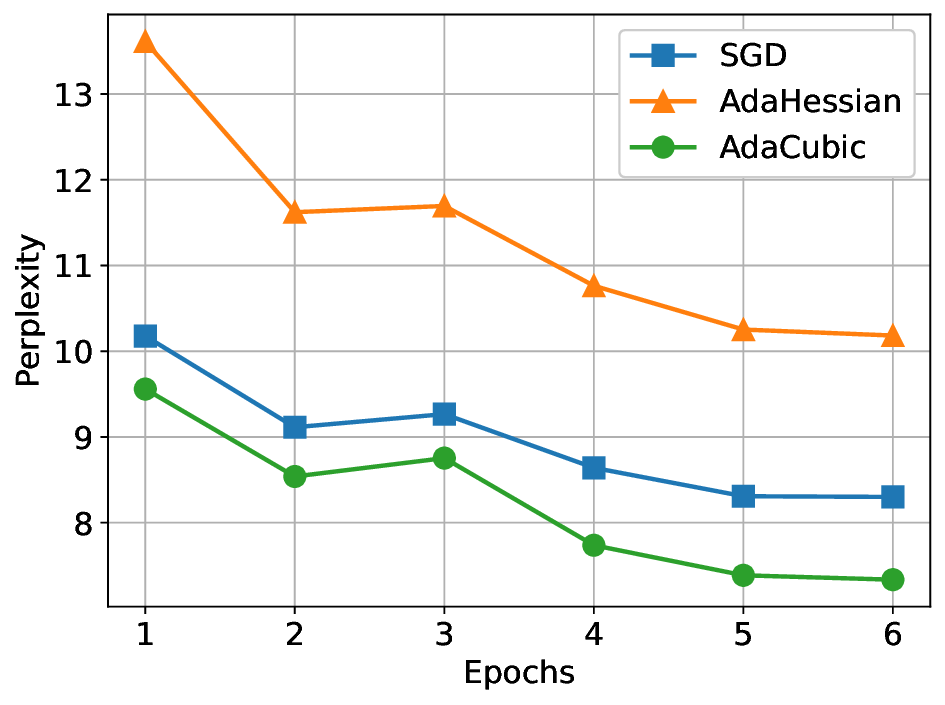} }}%
    \rule[-.5cm]{0cm}{4cm}
    \caption{Perplexity vs. epochs for  \texttt{RoBERTa}, \texttt{BERT}, and \texttt{DistilBERT} models on \texttt{PTB} dataset.}%
    \label{fig:languagemodelling2}%
\end{figure*}

First, the perplexity measurements gathered for the \texttt{wikitext-2} dataset in Table~\ref{tbl:languagemodelling1} are discussed. When \texttt{RoBERTa} is used, \textsc{SGD} outperforms \textsc{AdaCubic} and \textsc{AdaHessian} by 0.209 and 0.827, respectively. Next, when \texttt{BERT} is used, \textsc{AdaCubic} outperforms \textsc{SGD} and \textsc{AdaHessian} by 7.621 and 10.392, respectively. For the \texttt{DistilBERT} model, \textsc{SGD} outperforms \textsc{AdaCubic} and \textsc{AdaHessian} by 0.447 and 0.704, respectively. We observe that in all models, \textsc{AdaCubic} outperforms \textsc{AdaHessian} and performs better or competitively when compared to \textsc{SGD}. Figure~\ref{fig:languagemodelling1} depicts the perplexity metric vs. epochs for all models and optimizers on the \texttt{wikitext-2} dataset.

Table~\ref{tbl:languagemodelling2} gathers perplexity measures on the \texttt{PTB} dataset. When \texttt{RoBERTa} is used, \textsc{SGD} outperforms \textsc{AdaCubic} and \textsc{AdaHessian} by 0.8 and 3.237, respectively. For the \texttt{BERT} model, \textsc{AdaCubic} outperforms \textsc{SGD} and \textsc{AdaHessian} by 3.174 and 6.681, respectively. For the \texttt{DistilBERT} model, \textsc{AdaCubic} outperforms \textsc{SGD} and \textsc{AdaHessian} by 0.965 and 2.848, respectively. Figure~\ref{fig:languagemodelling2} depicts the perplexity metric vs. epochs for all models and optimizers on the \texttt{PTB} dataset. 

On the NLU benchmark, Table~\ref{tbl:gluetask}, \textsc{AdaCubic} consistently achieves either the best or the second-best performance across all tasks, with the performance gaps reported in the $\Delta$ column remaining small. The second-best performance of \textsc{AdaCubic} on certain GLUE tasks can be understood in light of recent Hessian-based analyses of Transformers~\citep{zhang2024transformers}. In particular, \citet{zhang2024transformers} shows that Transformer models exhibit block-wise heterogeneity in their Hessian structure, with strong curvature differences and interactions across parameter groups. While \textsc{AdaCubic} explicitly leverages second-order information through diagonal Hessian approximations, such approximations may be insufficient to capture cross-parameter or block-level curvature interactions fully. This likely explains why \textsc{AdaCubic} remains highly competitive but does not consistently outperform finely tuned baselines on Transformer-based tasks. Similar conclusions hold for the LM benchmark, where \textsc{AdaCubic} consistently achieves either the best or second-best performance across all datasets.

Overall, it should be noted that \textsc{AdaCubic} exhibits the best or second-best performance with a pre-fixed universal set of parameters, while \textsc{SGD} and \textsc{AdaHessian} are fine-tuned w.r.t. the initial LR.

\noindent\textbf{Camera Model Identification (CMI).} The publicly available \texttt{VISION} dataset~\citep{shullani2017vision} is utilized for camera model identification. \texttt{VISION} includes 648 \texttt{Native} videos, which remain unaltered post-capture by the camera. These \texttt{Native} videos were disseminated via social media platforms such as \texttt{YouTube} and \texttt{WhatsApp}, with corresponding versions included in the dataset. Of the 684 \texttt{Native} videos, 644 were shared via \texttt{YouTube} and 622 via \texttt{WhatsApp}. Additional details on \texttt{VISION} can be found in~\citep{shullani2017vision}. Taking into account the \texttt{VISION} dataset naming conventions outlined in \citet{shullani2017vision}, videos captured by devices D04, D12, D17, and D22 are excluded due to issues encountered during frame extraction or audio track retrieval. 

\begin{table}[!ht]
    \centering
    \caption{CMI accuracy (\%) using \texttt{ResNet18}.}
    \label{tab:AudioVisionUnimodal}
    \begin{tabular}{c*{6}{c}}
        \toprule
        & \multicolumn{3}{c}{\textsc{AdaCubic}} & \multicolumn{3}{c}{\textsc{Adam}} \\
        \cmidrule(lr){2-4} \cmidrule(lr){5-7}
        & \texttt{Native} & \texttt{WhatsApp} & \texttt{YouTube} & \texttt{Native} & \texttt{WhatsApp} & \texttt{YouTube} \\
        \midrule 
        Fold 0  & 97.40 & 96.10 & 94.59 & 96.10 & 93.50 & 91.9  \\
        Fold 1  & 93.51 & 93.51 & 93.24 & 94.80 & 90.90 & 93.24 \\
        Fold 2  & 94.81 & 92.22 & 94.59 & 90.90 & 88.31 & 95.94 \\
        Fold 3  & 93.42 & 93.43 & 91.89 & 93.42 & 94.73 & 82.43 \\
        Fold 4  & 94.73 & 88.16 & 93.24 & 94.73 & 88.15 & 95.94 \\ \midrule
\makecell{Mean \\ $\pm$ std}  & \makecell{94.77 \\ $\pm$ 1.43} & \makecell{93.68 \\ $\pm$ 2.59} & \makecell{93.51 \\ $\pm$ 1.01} & \makecell{93.99 \\ $\pm$ 1.76} & \makecell{91.11 \\ $\pm$ 2.66} & \makecell{91.89 \\ $\pm$ 4.98} \\
        \bottomrule
    \end{tabular}
\end{table}

The videos are partitioned into training, testing, and validation sets to conduct a typical five-fold stratified cross-validation. The audio content from each video is extracted, and the log-Mel spectrogram for each extracted audio clip is computed using three distinct window sizes and hop sizes. This results in 3-channel log-Mel spectrograms that capture various frequency details of the audio content. The 3-channel log-Mel spectrograms are then fed into \texttt{ResNet18} to perform CMI. Furthermore, for \textsc{Adam}, $\beta_1$ and $\beta_2$ are set to 0.9 and 0.999, respectively. The LR is decayed by a factor of 10 at epochs 80 and 120 with an initial value $10^{-4}$.

Table~\ref{tab:AudioVisionUnimodal} summarizes the results when \textsc{AdaCubic} and \textsc{Adam} optimizers are used. The mean accuracy achieved using \textsc{AdaCubic} in the \texttt{Native}, \texttt{WhatsApp}, and \texttt{YouTube} benchmarks is $94.77\%$, $93.68\%$, and $93.51\%$, respectively. In comparison, the mean accuracy with \textsc{Adam} is $93.99\%$ for \texttt{Native}, $91.11\%$ for \texttt{WhatsApp}, and $91.89\%$ for \texttt{YouTube}. This indicates that \textsc{AdaCubic} is more accurate than \textsc{Adam} by $0.78\%$, $2.57\%$, and $1.62\%$ in the \texttt{Native}, \texttt{WhatsApp}, and \texttt{YouTube} benchmarks, respectively. In terms of std, \textsc{AdaCubic} demonstrates greater consistency than \textsc{Adam} by achieving lower std values of $0.33$, $0.07$, and $3.97$ in the \texttt{Native}, \texttt{WhatsApp}, and \texttt{YouTube} benchmarks, respectively. Implementation details for the audio CMI task can be found in~\citep{tsingalis2024camera}.

\begin{figure}[!ht]
  \centering{\rule[-.5cm]{0cm}{4cm}\includegraphics[trim={0.3cm 0.3cm 0cm 0cm},clip, width=0.6\linewidth]{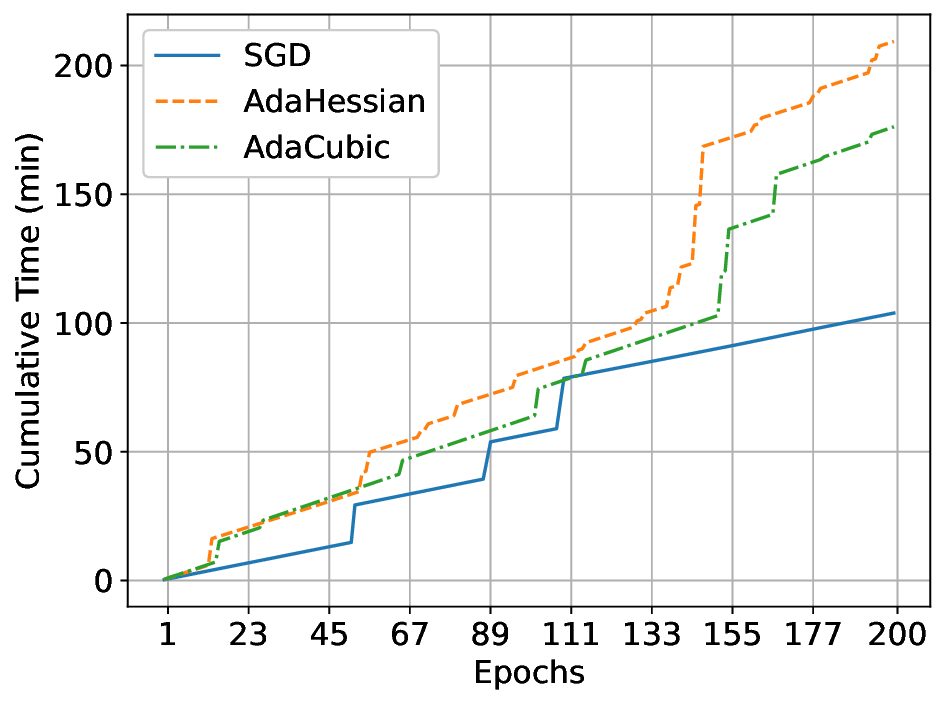}\rule[-.5cm]{0cm}{4cm}} 
  \caption{Cumulative time vs. epochs for \textsc{SGD}, \textsc{AdaHessian}, and \textsc{AdaCubic} for \texttt{ResNet20} on \texttt{CIFAR-10}.}
  \label{fig:timeComplex}
\end{figure} 

\begin{figure}[!ht]
\centering
\rule[-.5cm]{0cm}{4cm}
\subfloat[\label{fig:timevslossComplex}Cumulative time vs. loss]{%
    \includegraphics[width=0.47\linewidth]{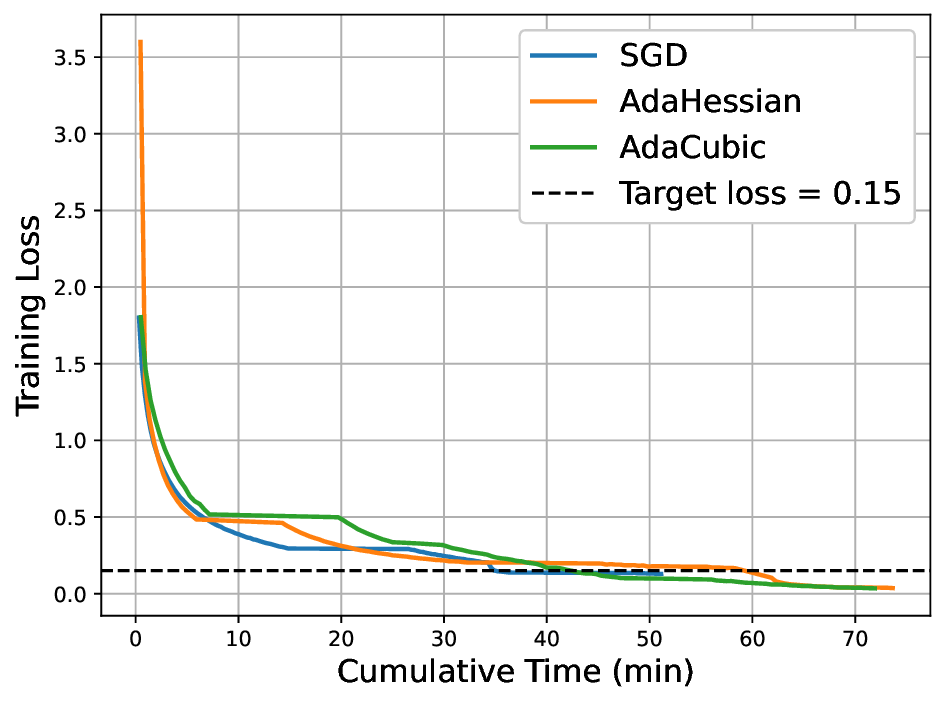}
}
\hfill
\rule[-.5cm]{0cm}{4cm}
\subfloat[\label{fig:lossvsepoch}Training loss vs. epochs]{%
    \includegraphics[width=0.47\linewidth]{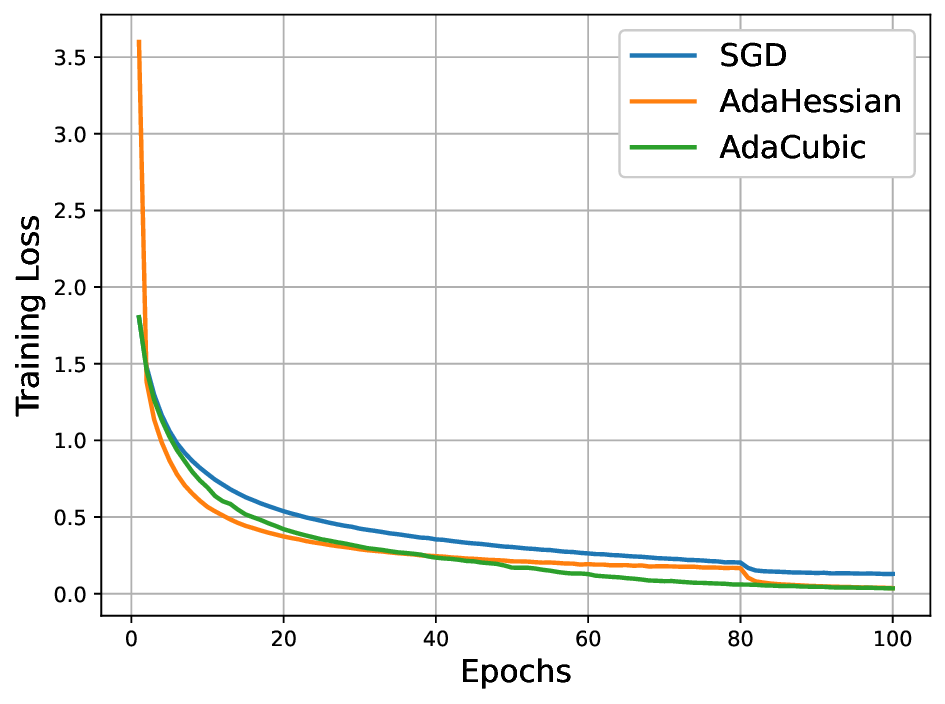}
}
\rule[-.5cm]{0cm}{4cm}
\caption{Comparison of \textsc{SGD}, \textsc{AdaHessian}, and \textsc{AdaCubic} on \texttt{ResNet20} and \texttt{CIFAR-10}. Training loss vs. cumulative time over epochs (Left). Training loss vs. epochs (Right).}
\label{fig:time_loss_compare}
\end{figure}

\section{Computational Complexity and Discussion}
\label{sec:complexity}

The performance and time complexity of the second-order methods depend on the approximation of the second-order information captured by the Hessian matrix. Similarly to \textsc{AdaHessian}, \textsc{AdaCubic} leverages the Hutchinson method~\citep{bekas2007} to approximate the diagonal of the Hessian matrix. Figure~\ref{fig:timeComplex} depicts the time complexity of \textsc{SGD}, \textsc{AdaHessian}, and \textsc{AdaCubic} when they are used to train \texttt{ResNet20} on \texttt{CIFAR-10}. As can be seen, the time complexity of the first-order optimizer \textsc{SGD} is smaller than that of the two second-order optimizers, with \textsc{AdaCubic} having less time complexity than \textsc{AdaHessian}. 

Figure~\ref{fig:timevslossComplex} shows the training loss vs. cumulative time for \textsc{SGD}, \textsc{AdaHessian}, and \textsc{AdaCubic}. Figure~\ref{fig:lossvsepoch} shows the training loss vs. epochs for \textsc{SGD}, \textsc{AdaHessian}, and \textsc{AdaCubic}. The training loss in Figure~\ref{fig:lossvsepoch} corresponds to that in Figure~\ref{fig:timevslossComplex}. The horizontal dashed line in Figure~\ref{fig:timevslossComplex} marks the target loss threshold of $0.15$. \textsc{AdaCubic} reaches this threshold after 55 epochs and 42.40 minutes. In comparison, \textsc{SGD} and \textsc{AdaHessian} require 83 and 81 epochs, corresponding to 35.16 and 61.85 minutes. Table~\ref{tab:wallclock_threshold} summarizes the latter results. Although \textsc{AdaCubic} needs more time than \textsc{SGD} due to the computation of the second-order information, \textsc{AdaCubic} reaches the desired loss in fewer epochs without any LR tuning. This highlights \textsc{AdaCubic} as an efficient trade-off between computational cost and convergence quality.

\begin{table}[!ht]
    \centering
    \caption{Execution time in minutes required to reach a target loss threshold when ResNet20 is trained on \texttt{CIFAR-10}.}
    \label{tab:wallclock_threshold}
    \begin{tabular}{c*{2}{c} *{2}{c} *{2}{c}}
        \toprule
        & \multicolumn{2}{c}{\textsc{SGD}}
        & \multicolumn{2}{c}{\textsc{AdaHessian}}
        & \multicolumn{2}{c}{\textsc{AdaCubic}} \\
        \cmidrule(lr){2-3}\cmidrule(lr){4-5}\cmidrule(lr){6-7}
        & Epoch & Time & Epoch & Time & Epoch & Time \\
        \midrule
        Result & 83 & 35.16 & 81 & 61.85 & 55 & 42.40 \\
        \bottomrule
    \end{tabular}
\end{table}

Additionally, storing the Hessian matrix increases the memory consumption of any second-order optimizer. Using Hutchinson’s method for approximating the diagonal of the Hessian, the second-order information is represented by the diagonal approximation of the Hessian matrix, which leads to a $\pazocal{O}(d)$ memory complexity~\citep{bekas2007}. This additional memory cost is incurred by \textsc{AdaCubic} relative to first-order methods such as \textsc{SGD}.

Furthermore, when utilizing~\citet{bekas2007}, the approximation of the diagonal of the Hessian demands an additional gradient back-propagation. The additional gradient back-propagation step is also needed in \textsc{AdaHessian}. When comparing \textsc{AdaCubic} with \textsc{Adam}, the latter shares similar memory consumption due to the requirement of the gradient momentum term, but it does not necessitate an additional gradient back-propagation. 

\begin{table}[!ht]
\centering
\caption{Comparison of optimization methods used in the experimental evaluation. Recall that $d$ denotes the number of model parameters and $\pazocal{S}$ the number of random vectors used in the diagonal Hessian approximation.}
\label{tab:optimizer_comparison}
\begin{tabular}{lccccc}
\toprule
\textbf{Optimizer} 
& \textbf{Order} 
& \textbf{Sensitivity} 
& \textbf{Extra Backward Pass} 
& \textbf{Time Cost} 
& \textbf{Memory Footprint} \\
\midrule
\textsc{SGD}
& First 
& High 
& No 
& $d$
& $d$ \\

\textsc{Adam} 
& First 
& High 
& No 
& $d$
& $3d$ \\

\textsc{AdaHessian} 
& Second 
& Medium 
& Yes 
& $\pazocal{S}\,d$
& $4d$ \\

\textsc{AdaCubic }
& Second 
& Low
& Yes 
& $\pazocal{S}\,d$
& $2d$ \\
\bottomrule
\end{tabular}
\end{table}

Table~\ref{tab:optimizer_comparison} summarizes the optimization methods used in the experimental evaluation, highlighting their optimization order, sensitivity to hyperparameters, and computational overhead. The ``Order'' column indicates whether an optimizer relies on first- or second-order information. The sensitivity of the optimizers w.r.t. the LR is summarized in the ``Sensitivity'' column. The sensitivity of \textsc{SGD}, \textsc{Adam}, and \textsc{AdaHessian} w.r.t. the LR is discussed thoroughly in~\citep{ya2021}. \textsc{AdaCubic} has low sensitivity because it achieves competitive performance using a universal set of hyperparameters. The ``Extra Backward Pass'' column indicates whether additional back-propagation steps are required per optimization iteration, which directly relates to the use of second-order information. The reported time cost is dominated by the back-propagation procedure and is expressed as a function of the number of model parameters $d$. Recall that $\pazocal{S}$ is the number of random vectors used in the approximation of the diagonal Hessian matrix. Using $\pazocal{S}$ random vectors requires $\pazocal{S}$ backpropagation steps, increasing the time cost linearly. The ``Memory Footprint'' refers to the memory needed to store the gradient, the moment terms, and the approximated diagonal Hessian. As can be seen, the memory footprint of \textsc{Adam} and \textsc{AdaHessian} is $3d$ and $4d$, respectively, as the gradient and moments need memory relative to the number of parameters $d$. \textsc{AdaHessian} needs an additional memory footprint of $d$ for the storage of the approximate diagonal Hessian. \textsc{AdaCubic} shows a $2d$ memory overhead relative to \textsc{SGD}, since the gradient must be retained to compute the Hutchinson-based approximation of the diagonal Hessian, which is subsequently used by Algorithm~\ref{alg:newton}.

However, according to Algorithm~\ref{alg:newton}, \textsc{AdaCubic} requires only the approximated diagonal Hessian for its updates, yielding a theoretical memory footprint of $\pazocal{O}(d)$. The gap between practical and theoretical memory costs comes from the design of modern deep-learning frameworks, such as PyTorch, which are optimized for first-order optimization methods. Thus, computing the diagonal Hessian approximation requires retaining intermediate gradient information. Developing a custom implementation that directly computes the diagonal Hessian without storing such intermediates is beyond the scope of this work.

\section{Conclusions}
\label{sec:conclusions}

\textsc{AdaCubic}, a novel adaptive cubic regularized second-order optimizer, has been proposed. \textsc{AdaCubic} leverages an approximate Hessian diagonal to reduce the computational cost induced by estimating curvature information. Although many cubically regularized methods have been proposed in the literature, none have been extensively tested in practical deep-learning applications. The effectiveness of the proposed optimizer has been demonstrated through experiments on computer vision, natural language processing, and signal processing tasks, all using deep neural networks trained on various datasets. With a pre-fixed universal selection of parameters, \textsc{AdaCubic} exhibits better or competitive performance when compared to other state-of-the-art fine-tuned optimizers. This fact makes \textsc{AdaCubic} an attractive solution for optimizing deep neural networks. 

\section*{Acknowledgments}

This work was supported by the Hellenic Foundation for Research and Innovation (HFRI) under the HFRI PhD Fellowship grant (Fellowship Number: 1376) and the ``2nd Call for HFRI Research Projects to support Faculty Members \& Researchers'' (Project Number: 3888). The results were obtained using the High-Performance Computing Infrastructure and Resources of Aristotle University of Thessaloniki (AUTh). The authors would like to acknowledge the support provided by the IT Center of AUTh throughout the progress of this research work. 

\bibliography{main}
\bibliographystyle{tmlr}

\newpage

\appendix

\section{Summary of Dependencies} \label{appendix:theorySummary}

\begin{figure}[!ht]
  \centering
  \subfloat{
    \begin{adjustbox}{width=0.48\textwidth}
      \begin{tikzpicture}[
    node distance = 6mm and 18mm,
    arr/.style = {-{Straight Barb[scale=0.8]}, rounded corners=1ex, semithick},
    N/.style = {draw, rounded corners, thick, fill=#1,
                text width=10em, align=center, inner ysep=2ex},
    N/.default = white
    ]

    \begin{scope}[start chain=going below, nodes={on chain}]
        \node[N=red!20] (thm1) {Theorem~\ref{thm:minLmaxL2}};
        \node[N=green!20] (cor1) {Corollary~\ref{cor:strongduality}};
        \node[N=red!20] (thm2) {Theorem~\ref{thm:probequiv}};
        \node[N=red!20] (thm3) {Theorem~\ref{thm:diagmodelmin}};
    \end{scope}

    \node[N=blue!20, right=of thm1] (lemma2) {Lemma~\ref{lemma:minLmaxL}};
    \node[N=blue!20, below=of lemma2] (lem1) {Lemma~\ref{lemma:globalSol}};
    \node[N=blue!20, right=of thm2] (lemma3)  {Lemma~\ref{lemma:aproxmodelmin}};
    \node[N=blue!20, right=of lemma3] (lemma4)  {Lemma~\ref{lemma:localConvCond}};
    \node[N=blue!20, above=of lemma4] (lem56) {Lemmata~\ref{lemma:gradientDeviationBound},~\ref{lemma:hessianDeviationBound}\\Corollary~\ref{cor:gradienthessianSampling}};
    \node[N=cyan!20, above=of lem56] (Alg2)  {Algorithm~\ref{alg:newton}};
    \node[N=blue!20, above=of Alg2] (lem7)  {Lemmata~\ref{lemma:nursolution},~\ref{lemma:newtonphi},~\ref{lemma:newtonconv},~\ref{lemma:terminationNewtonnu}};
    
    \draw[arr] (thm1) -- (cor1);
    \draw[arr] (cor1) -- (thm2);
    \draw[arr] (thm2) -- (thm3);
    \draw[arr] ([yshift=0.5ex]lem1.east) -- ++(1.,0) |- ([yshift=-1.2ex]Alg2.west);
    \draw[arr] ([yshift=0.5ex]lem1.west) -- ++(-1.,0) |- ([yshift=-1.2ex]thm1.east);
    \draw[arr] ([yshift=-1.ex]lem1.west) -- ++(-1.,0) |- ([yshift=1.ex]thm2.east);
    
    \draw[arr] (lemma2.west) -- (thm1.east);
    \draw[arr] (thm2.east) -- (lemma3.west);
    \draw[arr] (lemma3.east) -- (lemma4.west);
    \draw[arr] (lemma4.north) -- (lem56.south);
    \draw[arr] (lem7.south) -- (Alg2.north);
\end{tikzpicture}
    \end{adjustbox}}
    \rule[-.5cm]{0cm}{4cm}
    \subfloat{
    \begin{adjustbox}{width=0.48\textwidth}
      \begin{tikzpicture}[
        node distance = 5mm and 20mm,
        start chain = going below,
        arr/.style = {{Straight Barb[scale=0.8]}-, rounded corners=1ex, semithick},
        N/.style = {draw, rounded corners, thick, fill=#1,
                    text width=9em, align=center, inner ysep=2ex},
        N/.default = white,
        every edge/.style = {draw, arr}
        ]   
    
        \begin{scope}[nodes={on chain, join=by arr}]
        \node[N=red!20, inner xsep=1em, inner ysep=3ex] (t3) {Theorem~\ref{thm:diagmodelmin}};
        \end{scope}
        
        \node (l141516) [N=blue!20, right=of t3] {Lemmata~\ref{lemma:lemma2Nesterov2006},~\ref{lemma:lemma3Nesterov2006},~\ref{lemma:lemma4Nesterov2006}};
        
        \node (l17) [N=blue!20, below=of l141516] {Lemma~\ref{lemma:localOptimalityCriterion}};
        \node (l18) [N=blue!20, below=of l17] {Lemma~\ref{lemma:lemma5Nesterov2006}};
        \node (t2)  [N=red!20, below=of l18]  {Theorem~\ref{thm:probequiv}};
        \node (cor3) [N=green!20, above=of l141516]  {Corollary~\ref{cor:globalSolDiag}};
        
        \draw[arr] ([yshift=+2.2ex]t3.east) -- ++(1.5,0) |- (l141516.west);
        \draw[arr] ([yshift=+0.7ex]t3.east) -- ++(1,0)   |- (l17.west);
        \draw[arr] ([yshift=-0.7ex]t3.east) -- ++(1.2,0) |- (l18.west);
        \draw[arr] ([yshift=-2.2ex]t3.east) -- ++(1.5,0) |- (t2.west);

        \node (l3) [N=blue!20, below=of t2, yshift=-0mm] {Lemma~\ref{lemma:aproxmodelmin}};
        \draw[arr] (l3) -- (t2);
        \draw[arr] (l18) -- (l17);
        
        \node (l9)  [N=blue!20, right=10mm of l17] {Lemma ~\ref{lemma:diagHessian}};
        \node (l10) [N=blue!20, above=of l9] {Lemma~\ref{lemma:lemma1Nesterov2006}};
        \node (l8)  [N=blue!20, below=of l9] {Lemma \ref{lemma:diagSpec}};
        
        \draw[arr] (l10) -- (l9);
        \draw[arr] (l9) -- (l8);
        
        \draw[arr] (l141516) -- (l10);
        \draw[arr] (l141516) -- (cor3);
        \draw[arr] (l17) -- (l9);
    \end{tikzpicture}
    \end{adjustbox}}
    
    \subfloat{
    \begin{adjustbox}{width=0.58\textwidth}
    \begin{tikzpicture}[
    node distance = 6mm and 18mm,
    arr/.style = {-{Straight Barb[scale=0.8]}, rounded corners=1ex, semithick},
    N/.style = {draw, rounded corners, thick, fill=#1,
                text width=10em, align=center, inner ysep=2ex},
    N/.default = white
    ]

    \begin{scope}[start chain=going below, nodes={on chain}]
        \node[N=blue!20] (lem3) {Lemma~\ref{lemma:aproxmodelmin}};
        \node[N=blue!20] (lem4) {Lemma~\ref{lemma:localConvCond}};
        \node[N=blue!20] (lem22) {Lemma~\ref{lemma:matrixBernsteinTwoSources}};
        \node[N=blue!20] (lem20) {Lemmata~\ref{lemma:basedDefOnTwoIndRandVars},~\ref{lem:5.4.10.alt}};
    \end{scope}

    \node[N=blue!20, right=of lem3] (lem5) {Lemma~\ref{lemma:gradientDeviationBound}};
    \node[N=green!20, below=of lem5] (cor2) {Corollary~\ref{cor:gradienthessianSampling}};
    \node[N=blue!20, below=of cor2] (lem6) {Lemma~\ref{lemma:hessianDeviationBound}};
    \node[N=blue!20, right=of lem5] (lem19) {Lemma~\ref{lemma:vectorBernstein}};

    \draw[arr] (lem3) -- (lem4);
    \draw[arr] ([yshift=1.5ex]lem4.east) -- ++(1.2,0) |- (lem5.west);
    \draw[arr] ([yshift=.15ex]lem4.east) -- ([yshift=.4ex]cor2.west);
    \draw[arr] ([yshift=-1.0ex]lem4.east) -- ++(1.2,0) |- ([yshift=1.0ex]lem6.west);

    \draw[arr] (lem20) -- (lem22);

    \draw[arr] ([yshift=-1.0ex]lem22.east) -- ([yshift=-.5ex]lem6.west);

    \draw[arr] (lem19.west) -- (lem5.east);
    \draw[arr] (lem6.north) -- (cor2.south);
    \draw[arr] (lem5.south) -- (cor2.north);
    
\end{tikzpicture} 
    \end{adjustbox}}
    
    \subfloat{
    \begin{adjustbox}{width=1\textwidth}
   \begin{tikzpicture}[
    node distance = 6mm and 18mm,
    arr/.style = {-{Straight Barb[scale=0.8]}, rounded corners=1ex, semithick},
    N/.style = {draw, rounded corners, thick, fill=#1,
                text width=12em, align=center, inner ysep=2ex},
    N/.default = white
    ]

    \node[N=yellow!30] (ass1) {Assumption~\ref{assum:continuityAssum}};

    \node[N=red!20, left=of ass1] (thm3) {Theorem~\ref{thm:diagmodelmin}};

    \node[N=blue!20, right=of ass1] (lemGroup) {
        Lemmata~\ref{lemma:localConvCond},~\ref{lemma:gradientDeviationBound},~\ref{lemma:hessianDeviationBound},~\ref{lemma:diagHessian}
    };

    \draw[arr] (ass1) -- (thm3);
    \draw[arr] (ass1) -- (lemGroup);

\end{tikzpicture}

\begin{tikzpicture}[
    node distance = 6mm and 18mm,
    arr/.style = {-{Straight Barb[scale=0.8]}, rounded corners=1ex, semithick},
    N/.style = {draw, rounded corners, thick, fill=#1,
                text width=12em, align=center, inner ysep=2ex},
    N/.default = white
    ]

    \node[N=yellow!30] (ass3) {Assumption~\ref{assum:hessAgreementAssum}};

    \node[N=blue!20, left=of ass3] (lem5) {Lemma~\ref{lemma:hessianDeviationBound}};

    \node[N=green!20, right=of ass3] (cor2) {Corollary~\ref{cor:gradienthessianSampling}};

    \draw[arr] (ass3) -- (lem5);
    \draw[arr] (ass3) -- (cor2);

\end{tikzpicture}
    \end{adjustbox}}
\caption{Logical connection between key lemmata, theorems, and corollaries throughout~\Cref{sec:proposed,sec:covAnalysis,sec:probSolution}~and~Appendices~\labelcref{appendix:globalSol,appendix:minLmaxL,appendix:minLmaxL2,appendix:probequiv,supplement:gradientDeviationBound,supplement:hessianDeviationBound,appendix:nursolution,supplement:diagHessian,supplement:diagHessianANDlemma1Nesterov2006,supplement:algDetails,supplement:diagmodelminRelatedLemmata,supplement:diagmodelmin,supplement:vectormatrixBernstein}.}
  \label{fig:outlineTheory}
\end{figure}

\section{Supporting Proofs}

\subsection{Proof of Lemma~\ref{lemma:globalSol}} \label{appendix:globalSol}

Here, we follow the guidelines in~\citep[Theorem 7.2.1]{conn2000}. Let us assume that $\mathbf{s}^*$ is a minimizer of $\hat{m}(\mathbf{s})$ subject to $\norm{\mathbf{s}^*}_2^3 \leq \xi$. Then, there is a Lagrange multiplier $\nu^*$ such that
\begin{equation}\label{eq:cscond}
\nu^*\:(\norm{\mathbf{s}^*}_2^3 - \xi ) = 0 \Leftrightarrow 
\begin{cases} \nu^* = 0, & \text{inactive constraint}  \\
\norm{\mathbf{s}^*}_2^3 = \xi, & \text{active constraint}. 
\end{cases}
\end{equation} 
$(\ref{eq:cscond})$ is the unfolded \textit{Complementary Slackness (CS)} condition~\citet{bertsekas2017} for the constrained optimization problem~(\ref{prob:cubicOursEquiv}). The active case occurs when $\mathbf{s}^*$ lies on the boundary of $\Omega$, i.e., $g_{\xi} (\mathbf{s}^*) = 0 \Leftrightarrow \norm{\mathbf{s}^*}_2^3 = \xi $ and the inactive case occurs when $\mathbf{s}^*$ lies in the interior of $\Omega$, i.e. $g_{\xi} (\mathbf{s}^*) < 0 \Leftrightarrow \norm{\mathbf{s}^*}_2^3 < \xi$. 

\begin{itemize}[leftmargin=*]

\item[$\rhd$] \textbf{Active constraint case.}
We assume that $\mathbf{s}^*$ is a minimizer of $\hat{m}(\mathbf{s})$ subject to $\norm{\mathbf{s}^*}_2^3 = \xi$. From the first-order optimality conditions~\citet{bertsekas2017}, there exists a Lagrange multiplier $\nu^*$, such that  
{\allowdisplaybreaks
\begin{multline}\label{eq:possDef}
\nabla_{\mathbf{s}}\pazocal{L}_{\xi}(\mathbf{s}^*, \nu^*) = \bm{0} \Leftrightarrow \nabla_{\mathbf{s}} \hat{m} (\mathbf{s}^*) + \nu^*\:\nabla_{\mathbf{s}}  g_{\xi} (\mathbf{s}^*) = \bm{0} \\ \Leftrightarrow \underbrace{\nabla^2 f(\mathbf{x}_k) \: \mathbf{s}^* + \nabla f(\mathbf{x}_k)}_{\nabla_{\mathbf{s}} \hat{m} (\mathbf{s}^*)} + \frac{\nu^*}{2}\norm{\mathbf{s}^*}_2 \mathbf{s}^* = \bm{0}  \Leftrightarrow  \\ \left(\nabla^2 f(\mathbf{x}_k) + \frac{\nu^*}{2}\norm{\mathbf{s}^*}_2~\mathbf{I}\right)\: \mathbf{s}^* = -\nabla f(\mathbf{x}_k),
\end{multline}} 
\noindent where the identity $\nabla_\mathbf{s}\norm{\mathbf{s}}_2^3 = 3\norm{\mathbf{s}}_2 \mathbf{s}$ was used for some $\mathbf{s}$. Let $\mathbf{s}$ be a feasible point on the boundary of $\Omega$, i.e., $\norm{\mathbf{s}}_2^3 = \xi$. The Taylor expansion of $\hat{m}(\mathbf{s})$ around the minimizer $\mathbf{s}^*$ is 
\begin{equation}\label{eq:taylor}
\hat{m} (\mathbf{s})= \hat{m}(\mathbf{s}^*) + (\mathbf{s} - \mathbf{s}^*)^T \: \nabla_{\mathbf{s}} \hat{m}(\mathbf{s}^*) +  \frac{1}{2} \: (\mathbf{s} - \mathbf{s}^*)^T \: \nabla_{\mathbf{s}}^2\hat{m}(\mathbf{s}^*)  \: (\mathbf{s} - \mathbf{s}^*).
\end{equation} 
From the second line in~(\ref{eq:possDef}), we also have
\begin{equation}\label{eq:gradmhat}
\nabla_{\mathbf{s}} \hat{m} (\mathbf{s}^*) = -\frac{\nu}{2}\norm{\mathbf{s}^*}_2~ \mathbf{s}^*.
\end{equation} 
Given~(\ref{eq:gradmhat}) and the fact that $\mathbf{s}$ and $\mathbf{s}^*$ are feasible points on the boundary of $\Omega$, i.e., $\norm{\mathbf{s}^*}_2^3 = \xi = \norm{\mathbf{s}}_2^3$, we have
{\allowdisplaybreaks
\begin{multline}
(\mathbf{s} - \mathbf{s}^*)^T \: \nabla_{\mathbf{s}} \hat{m} (\mathbf{s}^*) =  - \frac{\nu^*}{2}\norm{\mathbf{s}^*}_2~(\mathbf{s} - \mathbf{s}^*)^T\mathbf{s}^*  =  
\frac{\nu^*}{2} \norm{ \mathbf{s}^*}_2~(\norm{\mathbf{s}^*}_2^2 - \mathbf{s}^T\mathbf{s}^*) \\ = \frac{\nu^*}{2} \norm{\mathbf{s}^*}_2~\left[\frac{1}{2}\:\left(\xi^{2/3} + \xi^{2/3} \right) - \mathbf{s}^T\mathbf{s}^*\right] =   \frac{\nu^*}{2} \norm{\mathbf{s}^*}_2~\left[\frac{1}{2}\:\left(\norm{ \mathbf{s}^*}_2^2 + \norm{\mathbf{s}}_2^2 \right) - \mathbf{s}^T\mathbf{s}^*\right],
\end{multline}} 
which implies
\begin{equation}\label{eq:gradmhatproduct}
   (\mathbf{s} - \mathbf{s}^*)^T \: \nabla_{\mathbf{s}} \hat{m} (\mathbf{s}^*) = \frac{\nu^*}{4} \norm{\mathbf{s}^*}_2~ (\mathbf{s} - \mathbf{s}^*)^T(\mathbf{s} - \mathbf{s}^*).
\end{equation}

Combining~(\ref{eq:taylor}),~(\ref{eq:gradmhatproduct}), and $\nabla^2_{\mathbf{s}}  \hat{m} (\mathbf{s}^*) = \nabla^2 f(\mathbf{x}_k)$ gives
{\allowdisplaybreaks
\begin{multline}\label{eq:taylor1}
\hat{m} (\mathbf{s}) = \hat{m} (\mathbf{s}^*) + \frac{1}{4}\nu^* \norm{\mathbf{s}^*}_2~ (\mathbf{s} - \mathbf{s}^*)^T(\mathbf{s} - \mathbf{s}^*) +  \frac{1}{2}\:(\mathbf{s} - \mathbf{s}^*)^T \nabla^2 f(\mathbf{x}_k) (\mathbf{s} - \mathbf{s}^*) \\ =  \hat{m} (\mathbf{s}^*)  + \frac{1}{2}\:(\mathbf{s} - \mathbf{s}^*)^T \left(\nabla^2 f(\mathbf{x}_k) + \frac{\nu^*}{2} \norm{\mathbf{s}^*}\:\mathbf{I}\right) (\mathbf{s} - \mathbf{s}^*).
\end{multline}} 
The second-order optimality condition~\citet[Proposition 4.3.1]{bertsekas2017} for $\mathbf{z}\in\mathbb{R}^d$ yields 
\begin{equation}\label{eq:secOrdmhat1}
\mathbf{z}^T \Bigl(\nabla_{\mathbf{s}}^2 \hat{m} (\mathbf{s}^*) + \nu^*\: \nabla_{\mathbf{s}}^2 g_{\xi} (\mathbf{s}^*)\Bigr)\mathbf{z} \geq 0, 
\end{equation} 
where 
\begin{equation}\label{eq:secOrdmhat2}
\underbrace{\nabla_{\mathbf{s}}^2 \hat{m} (\mathbf{s}^*)}_{\nabla^2 f(\mathbf{x}_k)} + \nu^*\: \nabla_{\mathbf{s}}^2 g_{\xi} (\mathbf{s}^*) =  \left(\nabla^2 f(\mathbf{x}_k) + \frac{\nu^*}{2}\norm{\mathbf{s}^*}_2~\mathbf{I}\right) +  \frac{\nu^*}{2}\:\frac{\mathbf{s}^* (\mathbf{s}^*)^T}{\norm{\mathbf{s}^*}_2}
\end{equation} 
such that $\mathbf{z}^T \nabla_{\mathbf{s}} g_{\xi} (\mathbf{s}^*) = \frac{1}{2}\norm{\mathbf{s}^*}_2 \:\mathbf{z}^T \mathbf{s}^* = 0 \Leftrightarrow \mathbf{z}^T \mathbf{s}^* = 0$. Since $\mathbf{s}^* \neq \bm{0}$, we have 
{\allowdisplaybreaks
\begin{equation}\label{eq:possDef1}
\begin{aligned}
\mathbf{z}^T\left\{ \left(\nabla^2 f(\mathbf{x}_k) + \frac{\nu^*}{2}\norm{\mathbf{s}^*}_2~\mathbf{I}\right) + \frac{\nu^*}{2}\:\frac{\mathbf{s}^* (\mathbf{s}^*)^T}{\norm{\mathbf{s}^*}_2}\right\}\mathbf{z} \geq 0  & \Leftrightarrow  \\ 
\mathbf{z}^T \left(\nabla^2 f(\mathbf{x}_k) + \frac{\nu^*}{2}\norm{\mathbf{s}^*}_2~\mathbf{I}\right)\mathbf{z} + \frac{\nu^*}{2} \:\frac{(\mathbf{z}^T \mathbf{s}^*)^2 }{\norm{\mathbf{s}^*}_2} \geq 0 &.
\end{aligned}
\end{equation}} 
Using $\mathbf{z}^T \mathbf{s}^* = 0$ in~(\ref{eq:possDef1}) we get 
\begin{equation}\label{eq:possDef2}
\mathbf{z}^T \: \left(\nabla^2 f(\mathbf{x}_k) + \frac{\nu^*}{2}\norm{\mathbf{s}^*}_2~\mathbf{I}\right) \: \mathbf{z} \geq 0.
\end{equation} This indicates that $\nabla^2 f(\mathbf{x}_k) + \frac{\nu^*}{2}\norm{\mathbf{s}^*}_2~\mathbf{I}$ is positive semi-definite for vectors in the direction of the null-space of $\nabla_{\mathbf{s}} g_{\xi} (\mathbf{s}^*)$, i.e., perpendicular to $\nabla_{\mathbf{s}} g_{\xi} (\mathbf{s}^*)$.

It remains to consider vectors $\mathbf{w}\in \mathbb{R}^d$ that do not belong to the null-space of $\nabla_{\mathbf{s}} g_{\xi} (\mathbf{s}^*)$, i.e., $\mathbf{w}^T\nabla_{\mathbf{s}} g_{\xi} (\mathbf{s}^*) \neq 0$, and prove that $\nabla^2 f(\mathbf{x}_k) + \frac{\nu^*}{2}\norm{\mathbf{s}^*}_2~\mathbf{I}$ is also positive semi-definite. To this end, define the line $\mathbf{s}=\mathbf{s}^* + \alpha~\mathbf{w}$ as a function of $\alpha$. Because we are interested in $\mathbf{w}$, such that $\mathbf{w}^T\nabla_{\mathbf{s}} g_{\xi}  (\mathbf{s}^*) \neq 0$, the line intersects the constraint $g_{\xi} (\mathbf{s}) = 0 \Leftrightarrow \norm{\mathbf{s}}_2^3 = \xi$ in two values of $\alpha$. For $\alpha= 0$ we have $\mathbf{s}=\mathbf{s}^*$ and the aforementioned discussion holds. For $ \alpha \neq 0$, $\mathbf{s}$ satisfies $\norm{\mathbf{s}}_2^3 = \xi$. In the latter case, we may write $\mathbf{s} - \mathbf{s}^* = \alpha~\mathbf{w}$. From~(\ref{eq:taylor1}), we arrive at
\begin{equation}\label{eq:notNullSpace}
\hat{m}(\mathbf{s}) = \hat{m} (\mathbf{s}^*) + \frac{\alpha^2}{2}\: \mathbf{w}^T \left(\nabla^2 f(\mathbf{x}_k) + \frac{\nu^*}{2} \norm{\mathbf{s}^*}_2~\mathbf{I}\right)\mathbf{w},
\end{equation} 
with $\alpha \neq 0$. Given the assumption that $\mathbf{s}^*$ is a minimizer, i.e., $\hat{m}(\mathbf{s}^*) \leq \hat{m} (\mathbf{s})$,~(\ref{eq:notNullSpace}) implies that $\nabla^2 f(\mathbf{x}_k) + \frac{\nu^*}{2}\norm{\mathbf{s}^*}_2~\mathbf{I}$ is positive semi-definite. So far, we have shown that if 
$\mathbf{s}^*$ is a minimizer subject to $\norm{\mathbf{s}^*}_2^3 = \xi$, then $\nabla^2 f(\mathbf{x}_k) + \frac{\nu^*}{2}\norm{\mathbf{s}^*}_2~\mathbf{I}$ is positive semi-definite either in the direction of the null-space of $\nabla_{\mathbf{s}} g_{\xi} (\mathbf{s}^*)$ or not. Conversely, if $\nabla^2 f(\mathbf{x}_k) + \frac{\nu^*}{2}\norm{\mathbf{s}^*}_2~\mathbf{I}$ is positive semi-definite, from~(\ref{eq:taylor1}) and~(\ref{eq:notNullSpace}), we arrive at $\hat{m}(\mathbf{s}^*) \leq \hat{m} (\mathbf{s})$, i.e., $\mathbf{s}^*$ is a minimizer subject to $\norm{\mathbf{s}^*}_2^3 = \xi$.

Regarding the uniqueness of the solution, when $\nabla^2 f(\mathbf{x}_k) + \frac{\nu^*}{2}\norm{\mathbf{s}^*}_2~\mathbf{I}$ is positive definite, from~(\ref{eq:taylor1}) and~(\ref{eq:notNullSpace}) we have that $\hat{m}(\mathbf{s}^*) < \hat{m} (\mathbf{s})$, which indicates that $\mathbf{s}^*$ is a unique minimizer subject to $\norm{\mathbf{s}^*}_2^3 = \xi$.

\item[$\rhd$] \textbf{Inactive constraint case.}

In this case, we assume that $\mathbf{s}^*$ is a minimizer of $\hat{m}(\mathbf{s})$ subject to $\norm{\mathbf{s}^*}_2^3 < \xi$ when $\nu^* = 0$. From~(\ref{eq:possDef}) we obtain
\begin{equation}\label{eq:inactiveSystEq}
\nabla^2 f(\mathbf{x}_k) \mathbf{s}^* = -\nabla f(\mathbf{x}_k).
\end{equation} 
From the second-order optimality condition~\citet[Proposition 4.3.1]{bertsekas2017}, it is implied that $\nabla_{\mathbf{s}\mathbf{s}}^2 \pazocal{L}_{\xi}(\mathbf{s}^*, \nu^*)$ is positive semi-definite. Using the latter fact, along with the fact that $\nabla_{\mathbf{s}\mathbf{s}}^2 \pazocal{L}_{\xi}(\mathbf{s}^*, \nu^*) = \nabla^2 f(\mathbf{x}_k)$ when $\nu^*=0$, we get that $\nabla^2 f(\mathbf{x}_k)$ is positive semi-definite. This, in turn, implies that we are dealing with a convex problem. 

Conversely, when $\nabla^2 f(\mathbf{x}_k)$ is positive semi-definite and $\nu^* = 0$, we can use the Taylor expansion of $\hat{m} (\mathbf{s})$ in~(\ref{eq:taylor}) along with the fact that 
\begin{equation}
\begin{aligned}
\nabla_{\mathbf{s}}\pazocal{L}_{\xi}(\mathbf{s}^*, \nu^*) = \bm{0} & \Leftrightarrow \nabla_{\mathbf{s}} \hat{m} (\mathbf{s}^*) + 
\underbrace{\cancelto{0}{\nu^*\:\nabla_{\mathbf{s}}  g_{\xi} (\mathbf{s}^*) }}_{\substack{\text{as $\nu^*=0$}}} = \bm{0} \\ & \Leftrightarrow  \nabla_{\mathbf{s}} \hat{m} (\mathbf{s}^*) = \bm{0}
\end{aligned}
\end{equation} 
to show that $\hat{m}(\mathbf{s}^*) \leq \hat{m} (\mathbf{s})$. This implies that $\mathbf{s}^*$ is a minimizer subject to $\norm{\mathbf{s}^*}_2^3 < \xi$. Regarding the uniqueness of the solution, when $\nabla^2 f(\mathbf{x}_k)$ positive definite and $\nu^*=0$, we can solve~(\ref{eq:taylor1}) w.r.t. $\mathbf{s}^* = - \nabla^2 f(\mathbf{x}_k)^{-1}\nabla f(\mathbf{x}_k)$, which indicates that $\mathbf{s}^*$ is a unique minimizer subject to $\norm{\mathbf{s}^*}_2^3 < \xi$.
\end{itemize} 

Given that no assumption has been made on the structure of $\nabla f(\mathbf{x})$, we can repeat the aforementioned proof using $\Diag(\nabla^2 f(\mathbf{x}))$ instead of $\nabla f(\mathbf{x})$ to arrive at Corollary~\ref{cor:globalSolDiag}.

\begin{restatable}{cor}{globalSolDiag}
A vector $\mathbf{s}^*$ is a minimizer of $\hat{m}(\mathbf{s})$ subject to $\norm{\mathbf{s}^*}_2^3 \leq \xi$ if and only if satisfies
\begin{equation}\label{eq:firstOrderCondDiag}
\left(\Diag(\nabla^2 f(\mathbf{x}_k)) + \frac{\nu^*}{2} \norm{\mathbf{s}^*}_2~ \mathbf{I} \right)\: \mathbf{s}^* = - \nabla f(\mathbf{x}_k),
\end{equation} 
\begin{equation} \label{eq:globalSolDiag}
\Diag(\nabla^2 f(\mathbf{x}_k)) + \frac{\nu^*}{2} \norm{\mathbf{s}^*}_2~ \mathbf{I} \succeq 0,
\end{equation} 
and $\nu^* \:(\norm{\mathbf{s}^*}_2^3 - \xi) = 0$, where $\nu^* \geq 0$. If $\nabla^2 f(\mathbf{x}_k) + \frac{\nu^*}{2} \norm{\mathbf{s}^*}_2~ \mathbf{I} \succ 0 $, then the minimizer $\mathbf{s}^*$ is unique.
\label{cor:globalSolDiag}
\end{restatable}  
Corollary~\ref{cor:globalSolDiag} will be used in the proof of Theorem~\ref{thm:diagmodelmin}.

\subsection{Proof of Lemma~\ref{lemma:minLmaxL}}\label{appendix:minLmaxL}

Starting from the primal optimization problem
\allowdisplaybreaks
\begin{equation}
\min_{\mathbf{s}\in \mathbb{R}^d} \pazocal{L}_{\xi} (\mathbf{s}, \nu) \stackrel{(\ref{eq:lagrangeCubicOursEquiv})}{=} \\ \min_{\substack{\mathbf{s} \in \mathbb{R}^d\   \norm{\mathbf{s}}_2^2=\tau }} \nabla f(\mathbf{x}_k)^T \mathbf{s} + \frac{1}{2}\: \mathbf{s}^T \nabla^2 f(\mathbf{x}_k) \mathbf{s} + \frac{\nu}{6}\: \Bigl(\tau^{3/2} - \xi \Bigr),
\label{eq:baseInequality}
\end{equation} 
where $\nu$ is the Lagrange multiplier, the optimal value of the primal problem can be expressed as 
\begin{equation}
\min_{\mathbf{s}\in \mathbb{R}^d} \pazocal{L}_{\xi} (\mathbf{s}, \nu) \stackrel{(\ref{eq:lagrangeCubicOursEquiv})}{=} \min_{\substack{\mathbf{s} \in \mathbb{R}^d\  \tau \geq 0 }} \max_{r\in \pazocal{D}_\nu} \Biggl\{ \nabla f(\mathbf{x}_k)^T \: \mathbf{s} + \frac{1}{2}\: \mathbf{s}^T \: \nabla^2 f(\mathbf{x}_k) \: \mathbf{s} + \frac{\nu}{6}\: \Bigl(\tau^{3/2} - \xi\Bigr) + \frac{r\nu}{4}\left(\norm{\mathbf{s}}_2^2-\tau \right)\Biggr\},
\label{eq:baseInequality11}
\end{equation}
where $r$ is the Lagrange multiplier associated to the constraint $\norm{\mathbf{s}}_2^2=\tau$~\citep[Section 5.4]{boyd2004}. It is essential to highlight that the optimality conditions outlined in~\citet[Proposition 4.2.1]{bertsekas2017} explicitly require $r$ to belong to $\mathbb{R}$. However, $r$ is restricted to $\pazocal{D}_{\nu}$ for reasons that become apparent as the proof unfolds. If the weak duality property is applied to the right-hand side (RHS) of~(\ref{eq:baseInequality11}), we arrive at 
\allowdisplaybreaks
\begin{equation}\label{eq:baseInequality12}
\min_{\mathbf{s}\in \mathbb{R}^d} \pazocal{L}_{\xi} (\mathbf{s}, \nu) \geq  \max_{r\in \pazocal{D}_\nu} \min_{\substack{\mathbf{s} \in \mathbb{R}^d\  \tau \geq 0 }} \Biggl\{\nabla f(\mathbf{x}_k)^T \:\mathbf{s} + \frac{1}{2}\: \mathbf{s}^T \: \nabla^2 f(\mathbf{x}_k) \: \mathbf{s} + \\ \frac{\nu}{6}\: \Bigl(\tau^{3/2} - \xi\Bigr) + \frac{r\nu}{4}\left(\norm{\mathbf{s}}_2^2-\tau \right)\Biggl\}. 
\end{equation} 
From the first-order optimality condition~\citet{bertsekas2017}, the optimal value in the Left Hand Side (LHS) of~(\ref{eq:baseInequality}) w.r.t. $\mathbf{s}$ is attained by $\mathbf{s}$ that satisfies $\nabla_{\mathbf{s}} \pazocal{L}_{\xi} (\mathbf{s}, {\nu})  = \bm{0}$, i.e.,
\begin{equation}\label{eq:baseInequalityLHS}
\left(\nabla^2 f(\mathbf{x}_k) + \frac{\nu}{2} \norm{\mathbf{s}}_2 \:\mathbf{I}\right) \: \mathbf{s} = -\nabla f(\mathbf{x}_k), \quad \nu \geq 0.
\end{equation} 
At this point, we note that~(\ref{eq:baseInequalityLHS}) differs from~(\ref{eq:firstOrderCond}), because the stationarity of $\pazocal{L}_{\xi} (\mathbf{s}, {\nu})$ is studied w.r.t. $\mathbf{s}$ only. Denote the RHS of (\ref{eq:baseInequality12}) as
\allowdisplaybreaks
\begin{multline}\label{eq:Lagrangescr} 
\mathscr{L}_{\xi} (\mathbf{s}, \nu, r, \tau) =  \:\nabla f(\mathbf{x}_k)^T \: \mathbf{s} + \frac{1}{2}\: \mathbf{s}^T \: \nabla^2 f(\mathbf{x}_k) \: \mathbf{s} +  \frac{\nu}{6}\: \Bigl(\tau^{3/2} - \xi\Bigr) + \frac{r\nu}{4} \Bigl(\norm{\mathbf{s}}_2^2-\tau \Bigr) \\ = \: \nabla f(\mathbf{x}_k)^T \mathbf{s} + \frac{1}{2}\: \mathbf{s}^T \left(\nabla^2 f(\mathbf{x}_k) + \frac{\nu \: r}{2} \: \mathbf{I}\right) \mathbf{s} + \frac{\nu}{6}\: \Bigl(\tau^{3/2} - \xi\Bigr) - \frac{r\nu}{4} \tau.
\end{multline} 
We start with the case $\nu > 0$. Solving $\partial_{\tau}\mathscr{L}_{\xi} (\mathbf{s}, \nu, r, \tau) = 0 $ w.r.t. $\tau$, we get 
\begin{equation}\label{eq:tau}
\tau^* = r^2,
\end{equation} 
where $r \in \pazocal{D}_{\nu}$. Restricting $r$ in $\pazocal{D}_\nu$, implies that $r > 0$, which in turn implies $\tau^* > 0$, as $\norm{\mathbf{s}} = \tau^*$. If $\tau^* = 0$, we have $\norm{\mathbf{s}} = 0$, which leads to the trivial solution, i.e., the zero vector. Solving $\nabla_{\mathbf{s}}\mathscr{L}_{\xi} (\mathbf{s}, \nu, r, \tau) = \bm {0} $ w.r.t. $\mathbf{s}$, we arrive at 
\begin{equation} \label{eq:systemeqL1}
\nabla f(\mathbf{x}_k) =- \left(\nabla^2 f(\mathbf{x}_k) + \frac{\nu \: r}{2}~\mathbf{I} \right) \mathbf{s}.
\end{equation} 
For $r \in \pazocal{D}_\nu$, we get from~(\ref{eq:systemeqL1})
\begin{equation}\label{eq:hx1}
\mathbf{s} (\nu, r) = - \left(\nabla^2 f(\mathbf{x}_k) + \frac{\nu \: r}{2}~\mathbf{I}\right)^{-1} \nabla f(\mathbf{x}_k), \quad \nu > 0,
\end{equation} 
which implies the dependence of $\mathbf{s}$ on the variables $\nu$ and $r$. Restricting $r$ in $\pazocal{D}_\nu$, we achieve the invertibility in~(\ref{eq:hx1}) when $\nu > 0$. Substituting~(\ref{eq:tau}) and (\ref{eq:systemeqL1}) in~(\ref{eq:Lagrangescr}), we get 
\begin{equation}\label{eq:hLscr}
\mathscr{L}_{\xi}(\mathbf{s} (\nu, r), \nu, r) =  -\frac{1}{2}\:\mathbf{s}(\nu, r)^T\: \left(\nabla^2 f(\mathbf{x}_k) + \frac{\nu \: r}{2} \:\mathbf{I}\right) \: \mathbf{s}(\nu, r) - \frac{\nu}{6}\xi - \frac{\nu}{12} r^3.
\end{equation} 
Combining~(\ref{eq:baseInequality12}) and~(\ref{eq:hLscr}) we get for $\nu > 0$
\begin{equation}\label{eq:baseInequality2}
\min_{\mathbf{s}\in \mathbb{R}^d} \pazocal{L}_{\xi} (\mathbf{s}, \nu) \geq \max_{r \in \pazocal{D}_\nu}\: {\mathscr{L}}_{\xi}(\mathbf{s} (\nu, r), \nu, r).
\end{equation} 
The derivative of~(\ref{eq:hLscr}) w.r.t. $r$ is
\begin{equation}\label{eq:partialLagrangrhoxi}
\partial_{r} \mathscr{L}_{\xi} (\mathbf{s}(\nu, r), \nu, r)  = \frac{\nu}{4} \:\Bigl( \norm{\mathbf{s}(\nu, r)}_2^2 -r^2\Bigr).
\end{equation} 
Thus, for any $\nu > 0$, the optimal value in the RHS of~(\ref{eq:baseInequality2}) is attained for $r^* \in \pazocal{D}_\nu$ that solves
\begin{equation}\label{eq:optwrtr}
\left(\frac{\partial{\mathscr{L}_{\xi}(\mathbf{s} (\nu, r), \nu, r)}}{\partial r}\right)_{r = r^*} = 0.
\end{equation} 
Using~(\ref{eq:partialLagrangrhoxi}) in~(\ref{eq:optwrtr}), we have
\begin{equation} \label{eq:rho}
r^* = \norm{\mathbf{s}(\nu, r^*)}_2\: \text{for any}\: \nu > 0. 
\end{equation}  
Restricting $r^*$ in $\pazocal{D}_\nu$, we avoid the trivial solution $\mathbf{s}(\nu, r^*) = \bm{0}$ for any $\nu > 0$. This restriction on $r$ in (\ref{eq:baseInequality11}) is precisely due to its inclusion in~$\pazocal{D}_{\nu}$. 

Using~(\ref{eq:systemeqL1}) in~(\ref{eq:lagrangeCubicOursEquiv}) we attain
\begin{equation}
\pazocal{L}_{\xi} (\mathbf{s}(\nu, r), \nu) = 
- \mathbf{s} (\nu, r)^T \: \left(\nabla^2 f(\mathbf{x}_k)+ \frac{\nu \: r}{2}~\mathbf{I}\right)\: \mathbf{s}(\nu, r) +  \frac{1}{2}\:\mathbf{s}(\nu, r)^T \: \nabla^2 f(\mathbf{x}_k) \: \mathbf{s}(\nu, r)  + \frac{\nu}{6}\:\norm{\mathbf{s}(\nu, r)}_2^3 - \frac{\nu}{6}\xi. 
\end{equation} 
Adding and subtracting the terms $\frac{\nu}{12}r^3$ and $\frac{\nu \: r}{4}\: \norm{\mathbf{s}(\nu, r)}_2^2$, we get
{\allowdisplaybreaks
\begin{multline}
\pazocal{L}_{\xi} (\mathbf{s}(\nu, r), \nu) = 
- \mathbf{s} (\nu, r)^T \: \left(\nabla^2 f(\mathbf{x}_k)+ \frac{\nu \: r}{2}~\mathbf{I}\right)\: \mathbf{s}(\nu, r) +  \frac{1}{2}\:\mathbf{s}(\nu, r)^T \: \nabla^2 f(\mathbf{x}_k) \: \mathbf{s}(\nu, r) + \frac{\nu}{6}\:\norm{\mathbf{s}(\nu, r)}_2^3 \\ - \frac{\nu}{6}\xi  + \underbrace{\left(\frac{\nu}{12}r^3-\frac{\nu}{12}r^3 \right)}_{0} + \underbrace{\left(\frac{\nu \: r}{4}\: \norm{\mathbf{s}(\nu, r)}_2^2 - \frac{\nu \: r}{4}\: \norm{\mathbf{s}(\nu, r)}_2^2 \right)}_{0}.
\end{multline}} 
Next, using~(\ref{eq:hLscr}), with appropriate rearrangements we arrive at
{\allowdisplaybreaks
\begin{multline}\label{eq:Lagrangianrelation1}
\pazocal{L}_{\xi} (\mathbf{s}(\nu, r), \nu) =  \mathscr{L}_{\xi}(\mathbf{s}(\nu, r), \nu, r) + \frac{\nu}{12}r^3 + \frac{\nu}{6}\:\norm{\mathbf{s}(\nu, r)}_2^3 - \frac{\nu \: r}{4}\: \norm{\mathbf{s}(\nu, r)}_2^2 \\ = \mathscr{L}_{\xi}(\mathbf{s}(\nu, r), \nu, r) + \frac{\nu}{12}\:\Bigl( r^3 + 2~\norm{\mathbf{s}(\nu, r)}_2^3 - 3r\: \norm{\mathbf{s}(\nu, r)}_2^2 \Bigr) \\ = \mathscr{L}_{\xi}(\mathbf{s}(\nu, r), \nu, r) +  \frac{\nu}{12}\:\left(\norm{\mathbf{s}(\nu, r)}_2-r\right)^2 \left(r+2~\norm{\mathbf{s}(\nu, r)}_2 \right).
\end{multline}}
Then, using~(\ref{eq:partialLagrangrhoxi}) for $r \in \pazocal{D}_\nu$ and $\nu > 0$ we obtain
\begin{equation}\label{eq:Lagrangianrelation6}
\pazocal{L}_{\xi} (\mathbf{s}(\nu, r), \nu) =  \mathscr{L}_{\xi}(\mathbf{s}(\nu, r), \nu, r) + \frac{4}{3\nu}\: \frac{\left(r+2~\norm{\mathbf{s}(\nu, r)}_2 \right)}{\left(r+\norm{\mathbf{s}(\nu, r)}_2 \right)^2} \Bigl(\partial_{r} \mathscr{L}_{\xi}(\mathbf{s}(\nu, r), \nu, r) \Bigr)^2.
\end{equation} 
When~(\ref{eq:optwrtr}) is satisfied for some $\nu > 0$, $\mathscr{L}_{\xi}(\mathbf{s}(\nu, r), \nu, r)$ given by~(\ref{eq:Lagrangianrelation6}) is maximized w.r.t. $r \in \pazocal{D}_\nu$. From~(\ref{eq:Lagrangianrelation6}) we have
\begin{equation}\label{eq:Lagrangianrelation4}
\pazocal{L}_{\xi} (\mathbf{s}(\nu, r^*), \nu) = \max_{r \in \pazocal{D}_\nu}\: \mathscr{L}_{\xi}(\mathbf{s}(\nu, r), \nu, r).
\end{equation} 
In order to obtain~(\ref{eq:minLmaxL}), we need to show
\begin{equation}\label{eq:Lagrangianrelation5}
\pazocal{L}_{\xi} (\mathbf{s}(\nu, r^*), \nu) = \min_{\mathbf{s}\in \mathbb{R}^d}\:\pazocal{L}_{\xi} (\mathbf{s}, \nu), \quad \nu >0.
\end{equation} 
When $r^* \in \pazocal{D}_\nu$ in~(\ref{eq:hx1}) and using $r^* = \norm{\mathbf{s}(\nu, r^*)}_2$ for some $\nu >0$ in~(\ref{eq:rho}), we get
\begin{equation}\label{eq:hx2}
\left(\nabla^2 f(\mathbf{x}_k) + \frac{\nu}{2} \norm{\mathbf{s} (\nu, r^*)}\:\mathbf{I}\right)\:\mathbf{s} (\nu, r^*) = - \nabla f(\mathbf{x}_k)
\end{equation} 
which implies that $\mathbf{s}(\nu, r^*)$ minimizes $\pazocal{L}_{\xi} (\mathbf{s}, \nu)$.

We conclude with the case $\nu = 0$. In this case, we observe that~(\ref{eq:minLmaxL}) is easily attained by applying~(\ref{eq:baseInequality12}) when equality holds, which concludes the proof.

\subsection{Proof of Theorem~\ref{thm:minLmaxL2}}\label{appendix:minLmaxL2}

From the \textit{weak duality} in~(\ref{prob:cubicOursEquiv}) and~(\ref{eq:minLmaxL}) we have
\begin{equation}\label{eq:baseInequality1}
\min_{\mathbf{s}\in \mathbb{R}^d} \max_{\nu \geq 0} \pazocal{L}_{\xi} (\mathbf{s}, \nu) \geq  \max_{\nu \geq 0} \min_{\mathbf{s}\in \mathbb{R}^d} \pazocal{L}_{\xi} (\mathbf{s}, \nu)  \stackrel{(\ref{eq:minLmaxL})}{=}  \max_{\nu \geq 0} \max_{r \in \pazocal{D}_\nu} \mathscr{L}_{\xi}(\mathbf{s}(\nu, r), \nu, r)   =  \max_{r \in \pazocal{D}_\nu,\nu \geq 0} \mathscr{L}_{\xi}(\mathbf{s}(\nu, r), \nu, r),
\end{equation} 
where the last term in~(\ref{eq:baseInequality1}) refers to a joint optimization problem. We start with the case $\nu > 0$. The derivative of~(\ref{eq:hLscr}) w.r.t. $\nu$ is
\begin{equation}\label{eq:partialLagrangnu}
\partial_{\nu} \mathscr{L}_{\xi}\bigl(\mathbf{s}(\nu, r), \nu, r\bigr) = \frac{r}{4}\:\norm{\mathbf{s} (\nu, r)}_2^2 - \frac{\xi}{6} - \frac{r^3}{12},
\end{equation} 
where $r \in \pazocal{D}_\nu$.

To prove that~(\ref{eq:baseInequality1}) holds with equality and subsequently prove~(\ref{eq:minLmaxL2}), we study the optimality conditions that maximize the RHS of~(\ref{eq:baseInequality1}). The optimal value in the RHS of~(\ref{eq:baseInequality1}) w.r.t $\nu > 0$ is achieved by some $\nu^* > 0$ that solves 
\begin{equation} \label{eq:optwrtnu}
\left(\frac{\partial \mathscr{L}_{\xi}\bigl(\mathbf{s}(\nu, r), \nu, r \bigr)}{\partial \nu}\right)_{\nu = \nu^*}=0, 
\end{equation} 
for $r \in \pazocal{D}_\nu$. In addition, the optimal value in the RHS of~(\ref{eq:baseInequality1}) w.r.t $r \in \pazocal{D}_\nu$ is achieved if~(\ref{eq:optwrtr}) or equivalently~(\ref{eq:rho}) holds. Given that~(\ref{eq:rho}) holds for any $\nu > 0$, without loss of generality, we assume that~(\ref{eq:rho}) also holds for $\nu^* > 0$, i.e.,
\begin{equation} \label{eq:rho1}
r^* = \norm{\mathbf{s}(\nu^*, r^*)}_2\: \text{for any}\: \nu^* > 0. 
\end{equation} 
When~(\ref{eq:rho1}) holds, from~(\ref{eq:optwrtnu}) we get
\begin{equation}\label{eq:xi}
\xi = \frac{3r^*}{2}\:\norm{\mathbf{s} (\nu^*, r^*)}^2_2 - \frac{{r^*}^3}{2}. 
\end{equation} 
Solving~(\ref{eq:hx}) w.r.t. $\nabla f(\mathbf{x}_k)$ and substituting in~(\ref{eq:lagrangeCubicOursEquiv}) we get 
\allowdisplaybreaks
\begin{multline}\label{eq:lagrangeNablafSubstitute}
\pazocal{L}_{\xi} (\mathbf{s}(\nu, r), \nu) =  - \mathbf{s} (\nu, r)^T \: \left(\nabla^2 f(\mathbf{x}_k) + \frac{\nu \: r}{2}~\mathbf{I} \right) \: \mathbf{s} (\nu, r) + \\ \frac{1}{2}\:\mathbf{s} (\nu, r)^T \: \nabla^2 f(\mathbf{x}_k) \: \mathbf{s} (\nu, r) +  \frac{\nu}{6}\: \left( \norm{\mathbf{s} (\nu, r)}_2^3 -\xi \right).
\end{multline} 
Then, applying~(\ref{eq:xi}) for $r \in \pazocal{D}_\nu$, we have
\allowdisplaybreaks
\begin{multline}
\pazocal{L}_{\xi}  (\mathbf{s}(\nu, r), \nu) =  - \mathbf{s} (\nu, r)^T \: \left(\nabla^2 f(\mathbf{x}_k) + \frac{\nu \: r}{2}~\mathbf{I}\right) \: \mathbf{s} (\nu, r)  +  \frac{1}{2}\:\mathbf{s} (\nu, r)^T \: \nabla^2 f(\mathbf{x}_k) \: \mathbf{s} (\nu, r)  + \hspace{3cm}\\ \frac{\nu}{6}\: \norm{\mathbf{s} (\nu, r)}_2^3 -  \frac{\nu \: r}{4} \:\norm{\mathbf{s} (\nu, r)}_2^2 + \frac{\nu \: r^3}{12}.
\end{multline} 
Adding and subtracting $\frac{\nu \: r}{4} \:\norm{\mathbf{s} (\nu, r)}_2^2$ we obtain
\allowdisplaybreaks
\begin{multline}
\pazocal{L}_{\xi}  (\mathbf{s}(\nu, r), \nu) =  -\frac{1}{2}\:\mathbf{s} (\nu, r)^T \: \left(\nabla^2 f(\mathbf{x}_k) + \frac{\nu \: r}{2}~\mathbf{I}\right) \: \mathbf{s} (\nu, r) - \frac{\nu \: r}{4} \:\norm{\mathbf{s} (\nu, r)}_2^2  + \frac{\nu}{6}\: \norm{\mathbf{s} (\nu, r)}_2^3  - \\ \frac{\nu \: r}{4} \:\norm{\mathbf{s} (\nu, r)}_2^2 + \frac{\nu \: r^3}{12}.
\end{multline} 
Similarly adding and subtracting $\frac{\nu \xi}{6}$ and $\frac{\nu \: r^3}{12}$ reveals the term $\mathscr{L}_{\xi} (\mathbf{s} (\nu, r), \nu, r)$ yielding
\begin{equation}\label{eq:Lagrangianrelation2}
\pazocal{L}_{\xi}  (\mathbf{s}(\nu, r), \nu) \stackrel{(\ref{eq:lagrangianscr})}{=}  \mathscr{L}_{\xi} (\mathbf{s} (\nu, r), \nu, r) + \nu\left(\frac{r^3}{12} + \frac{\xi}{6} -\frac{r}{4}~\norm{\mathbf{s} (\nu, r)}_2^2\right) +  \hfill \frac{\nu}{12}\left( r^3 + 2\norm{\mathbf{s} (\nu, r)}_2^3 - 3r \:\norm{\mathbf{s} (\nu, r)}_2^2\right).
\end{equation}
The terms inside the first bracket of~(\ref{eq:Lagrangianrelation2}) are identified as $-\partial_{\nu} \mathscr{L}_{\xi}\bigl(\mathbf{s}(\nu, r), \nu, r\bigr)$, yielding
{\allowdisplaybreaks
\begin{multline}
\pazocal{L}_{\xi} (\mathbf{s}(\nu, r), \nu) =   \mathscr{L}_{\xi} (\mathbf{s} (\nu, r), \nu, r) - \nu~\partial_{\nu}\mathscr{L}_{\xi}\bigl(\mathbf{s}(\nu, r), \nu, r\bigr) +   \frac{\nu}{12}\:\left(\norm{\mathbf{s}(\nu, r)}_2-r\right)^2 \left(r+2~\norm{\mathbf{s}(\nu, r)}_2 \right)    \\\overset{(\ref{eq:partialLagrangnu})}{=}  \mathscr{L}_{\xi}\bigl(\mathbf{s}(\nu, r), \nu, r\bigr) +\frac{4}{3\nu} \frac{\left(r+2~\norm{\mathbf{s}(\nu, r)}_2 \right)}{\left(r+\norm{\mathbf{s}(\nu, r)}_2 \right)^2} \Bigl(\partial_{r} \mathscr{L}_{\xi}(\mathbf{s}(\nu, r), \nu, r) \Bigr)^2  -  \nu~\partial_{\nu} \mathscr{L}_{\xi}\bigl(\mathbf{s}(\nu, r), \nu, r\bigr)
\end{multline}} 
and by rearranging terms, we arrive at
\begin{equation}\label{eq:Lagrangianrelation3}
\pazocal{L}_{\xi} (\mathbf{s}(\nu, r), \nu) =  \mathscr{L}_{\xi}\bigl(\mathbf{s}(\nu, r), \nu, r\bigr) - \nu~\partial_{\nu} \mathscr{L}_{\xi}\bigl(\mathbf{s}(\nu, r), \nu, r\bigr) +   \frac{4}{3\nu} \frac{\left(r+2~\norm{\mathbf{s}(\nu, r)}_2 \right)}{\left(r+\norm{\mathbf{s}(\nu, r)}_2 \right)^2} \Bigl(\partial_{r} \mathscr{L}_{\xi}(\mathbf{s}(\nu, r), \nu, r) \Bigr)^2.
\end{equation} 
When~(\ref{eq:optwrtr}) and~(\ref{eq:optwrtnu}) hold, $\mathscr{L}_{\xi}\bigl(\mathbf{s}(\nu, r), \nu, r\bigr)$ is maximized and from~(\ref{eq:Lagrangianrelation3}), we have 
\begin{equation}\label{eq:implystrongduality}
\pazocal{L}_{\xi}\bigl(\mathbf{s}(\nu^*, r^*), \nu^*, r^*\bigr) = \max_{\nu \geq 0, r \in \pazocal{D}_\nu}\: \mathscr{L}_{\xi}(\mathbf{s}(\nu, r), \nu, r),
\end{equation} 
where $r^*$ and $\nu^*$ optimize the RHS of~(\ref{eq:baseInequality1}). Given $r^*$, $\nu^*$, and~(\ref{eq:implystrongduality}), to show~(\ref{eq:minLmaxL2}), we need to prove
\begin{equation}\label{eq:Lagrangianrelation55}
\pazocal{L}_{\xi} \bigl(\mathbf{s}(\nu^*, r^*), \nu^*, r^*\bigr) = \min_{\mathbf{s}\in \mathbb{R}^d}\max_{\nu \geq 0}\: \pazocal{L}_{\xi}(\mathbf{s},\nu).
\end{equation} 
To do so, we need to show that the optimal $\mathbf{s}$ in the RHS of~(\ref{eq:Lagrangianrelation55}) equals $\mathbf{s}(\nu^*, r^*)$ in the LHS of~(\ref{eq:Lagrangianrelation55}). The optimal $\mathbf{s}$ in the RHS of~(\ref{eq:Lagrangianrelation55}) satisfies Lemma~\ref{lemma:globalSol}. Thus, by Lemma~\ref{lemma:globalSol}, if $\mathbf{s}(\nu^*, r^*)$ satisfies the CS condition
\begin{equation}\label{eq:csCond11}
\nu^*\:(\norm{\mathbf{s}(\nu^*, r^*)}_2^3- \xi)=0
\end{equation} 
and the system of equations
\begin{equation}\label{eq:sytemeqOpt}
\left(\nabla^2 f(\mathbf{x}_k) + \frac{\nu^*}{2} \norm{\mathbf{s} (\nu^*, r^*)}\:\mathbf{I}\right)\:\mathbf{s} (\nu^*, r^*) = - \nabla f(\mathbf{x}_k),
\end{equation} 
then (\ref{eq:Lagrangianrelation55}) holds. (\ref{eq:csCond11}) implies $\norm{\mathbf{s}(\nu^*, r^*)}_2^3 =  \xi$ for $\nu^* > 0$, which is true because of~(\ref{eq:rho1}). To prove~(\ref{eq:sytemeqOpt}), we apply~(\ref{eq:hx2}), where without loss of generality we replace $\nu > 0$ with $\nu^* > 0$, and the proof is complete. Corollary~\ref{cor:strongduality}, in the paper's main body, summarizes this proof's main result.

\subsection{Proof of Theorem~\ref{thm:probequiv}}\label{appendix:probequiv}

The proof has two parts. The first part deals with the RHS of~(\ref{eq:equivalence}), while the second part deals with the LHS of~(\ref{eq:equivalence}). A similar procedure is followed to that in~\citet[Proposition 1]{Kloft2009} to prove Theorem~\ref{thm:probequiv}.

\begin{itemize}[leftmargin=*]

\item[$\rhd$] \textbf{First part}.

Let $\mathbf{s}^*$ be the minimizer of~(\ref{prob:cubicOursEquiv}) which satisfies the feasibility condition $g_{\xi} (\mathbf{s}^*) \leq 0$. We want to show that when $M = \nu^*$, $\mathbf{s}^*$ is also a minimizer of~(\ref{prob:cubicNesterov2}). From Lemma~\ref{lemma:globalSol}, we recall that $\nu^* \: (\norm{\mathbf{s}^*}_2^3 - \xi) = 0$. Consequently,
\begin{equation}\label{eq:strongDuality1}
\min_{\mathbf{s} \in \mathbb{R}^d} \max_{\nu \geq 0} \pazocal{L}_{\xi}(\mathbf{s}, \nu) =  \pazocal{L}_{\xi} (\mathbf{s}^*, \nu^*) = \hat{m} (\mathbf{s}^*) + \underbrace{\cancelto{0}{\frac{\nu^*}{6}\: \Bigl(\norm{\mathbf{s}^*}_2^3 - \xi \Bigr)}}_{\substack{\text{0 from CS condition}}}  =  \hat{m} (\mathbf{s}^*).
\end{equation} 
From Corollary~\ref{cor:strongduality}, we have 
\allowdisplaybreaks
\begin{equation}\label{eq:strongDuality2}
\min_{\mathbf{s}\in \mathbb{R}^d} \max_{\nu \geq 0} \pazocal{L}_{\xi}(\mathbf{s}, \nu) = \max_{\nu \geq 0} \underbrace{\min_{\mathbf{s}\in \mathbb{R}^d} \pazocal{L}_{\xi} (\mathbf{s}, \nu)}_{\psi(\nu)} =  \max_{\nu \geq 0}\:\psi(\nu) = \psi (\nu^*), 
\end{equation} 
where $\psi (\nu)$ is the dual function of the constrained optimization problem~(\ref{prob:cubicOursEquiv}). From~(\ref{eq:strongDuality1}) and~(\ref{eq:strongDuality2}) we have
\allowdisplaybreaks
\begin{multline}\label{eq:strongDuality3}
\hat{m}(\mathbf{s}^*) = \psi (\nu^*) 
=  \min_{\mathbf{s} \in \mathbb{R}^d}\: \pazocal{L}_{\xi} (\mathbf{s}, \nu^*) 
=  \min_{\mathbf{s} \in \mathbb{R}^d}\:\Bigl\{ \hat{m}(\mathbf{s}) + \frac{\nu^*}{6}\:\overbrace{\Bigl(\norm{\mathbf{s}}_2^3 - \xi \Bigr)}^{\leq 0, \:\text{by feasibility}} \Bigr\} \\ \leq  \min_{\mathbf{s} \in \mathbb{R}^d}\:\hat{m} (\mathbf{s}) = \hat{m} (\mathbf{s}^*) + \frac{1}{6}\underbrace{\cancelto{0}{\nu^*\:\Bigl(\norm{\mathbf{s}^*}_2^3 - \xi \Bigr)}}_{\text{0 from CS condition}} =  \hat{m} (\mathbf{s}^*).
\end{multline} 
Since the first and the last term in~(\ref{eq:strongDuality3}) are equal, due to the CS condition, the in-between inequalities hold with equality, i.e.,  
\allowdisplaybreaks
\begin{equation} \label{eq:strongDuality4}
\min_{\mathbf{s} \in \mathbb{R}^d}\:\Bigl\{ \hat{m}(\mathbf{s}) + \frac{\nu^*}{6}\:\left(\norm{\mathbf{s}}_2^3 - \xi \right) \Bigr\}
=  \hat{m} (\mathbf{s}^*) + \frac{\nu^*}{6}\:\Bigl(\norm{\mathbf{s}^*}_2^3 - \xi \Bigr). 
\end{equation} 
Removing the constant term $-\frac{\nu^*}{6}\xi$ from both sides of~(\ref{eq:strongDuality4}), we obtain
{\allowdisplaybreaks
\begin{equation} \label{eq:strongDuality5}
\min_{\mathbf{s}\in \mathbb{R}^d}\: m_{\nu^*}(\mathbf{s})
= \hat{m} (\mathbf{s}^*) + \frac{\nu^*}{6}\:\norm{\mathbf{s}^*}_2^3 \stackrel{(\ref{prob:cubicOursEquiv})}{=}  f(\mathbf{x}_k) + \nabla f(\mathbf{x}_k)^T\mathbf{s}_k^* + \frac{1}{2}{\mathbf{s}^*}^T \nabla^2 f(\mathbf{x}_k) \mathbf{s}^* + \\ \frac{\nu^*}{6}\:\norm{\mathbf{s}^*}_2^3 \stackrel{(\ref{eq:basicCubicProblem})}{=}  m_{\nu^*}(\mathbf{s}^*),
\end{equation}} 
which implies that $\mathbf{s}^*$ is also a minimizer of~(\ref{prob:cubicNesterov2}) with $M=\nu^*$ and the first part of the proof is complete.

\item[$\rhd$] \textbf{Second part}. 

Let $\mathbf{s}^*$ be a minimizer of~(\ref{prob:cubicNesterov2}). We should prove that $\mathbf{s}^*$ is also a minimizer of~(\ref{prob:cubicOursEquiv}) when $\xi = \norm{\mathbf{s}^*}_2^3$. For such $\xi$, $g_{\xi} (\mathbf{s}^*) = 0$. We prove the second part by contradiction. Suppose, $\mathbf{s}^*$ is not optimal in~(\ref{prob:cubicOursEquiv}), i.e., there is a feasible point $\mathbf{s}$ such that $\hat{m} (\mathbf{s}) \leq \hat{m} (\mathbf{s}^*)$. For this feasible point we also have $g_{\xi} ({\mathbf{s}}) \leq 0$ and $g_{\xi} ({\mathbf{s}}) \leq g_{\xi} (\mathbf{s}^*)$. Then, we get, 
\begin{multline}\label{eq:strongDuality6}
\hat{m} (\mathbf{s}) \leq \hat{m} (\mathbf{s}^*) \Leftrightarrow  \hat{m} (\mathbf{s}) + g_{\xi} (\mathbf{s}) \leq \hat{m} (\mathbf{s}^*) + g_{\xi} (\mathbf{s}^*)  \\ \Leftrightarrow 
\hat{m} (\mathbf{s}) + \frac{\nu}{6}\:\Bigl(\norm{\mathbf{s}}_2^3-\xi \Bigr) \leq \hat{m} (\mathbf{s}^*) + \frac{\nu}{6}\:\Bigl(\norm{\mathbf{s}^*}_2^3 -\xi\Bigr).
\end{multline} 
Adding $\frac{\nu}{6}\xi$ in both sides of the last inequality in~(\ref{eq:strongDuality6}), using the definition of $m_{M} (\mathbf{s})$~(\ref{eq:basicCubicProblem}) with $M=\nu$, and applying the definition of $\hat{m} (\mathbf{s})$~(\ref{prob:cubicOursEquiv}), we get
\begin{equation}\label{eq:strongDuality7}
m_{\nu} (\mathbf{s}) \leq m_{\nu} (\mathbf{s}^*).
\end{equation} 
This is a contradiction, because $\mathbf{s}^*$ is a minimizer of~(\ref{prob:cubicNesterov2}). Hence, 
$\mathbf{s}^*$ is also a minimizer of~(\ref{prob:cubicOursEquiv}), when $g_{\xi} (\mathbf{s}^*) = 0 \Leftrightarrow \xi = \norm{\mathbf{s}^*}_2^3$, which concludes the second part of the proof.
\end{itemize}

\subsection{Proof of Lemma~\ref{lemma:gradientDeviationBound}} \label{supplement:gradientDeviationBound}

We follow similar lines to the proof of~\citep[Lemma 6 and Theorem 7]{kohler2017}. Note that $\pazocal{B}_k^g$ is used instead of $\pazocal{B}_k$ to emphasize that the deviation bound in~(\ref{eq:gradientDeviationBound}) and the sampling scheme in~(\ref{eq:gradientSampling}) are specifically derived using information associated with $\mathbf{g}_k$.

The proof resorts to Vector Bernstein's inequality in Lemma~\ref{lemma:vectorBernstein} (discussed in Appendix~\ref{supplement:vectormatrixBernstein}). Let us define the centered gradient
\begin{equation}
    \mathbf{z}_{i,k}^{s} = \nabla f_i(\mathbf{x}_k) - \nabla f (\mathbf{x}_k),
\end{equation} 
where $i=1,\dots,n=|\pazocal{B}_k^g|$. First, we show
\begin{equation}
    \norm{\mathbf{z}_{i,k}^{s}}_2 \leq \norm{\nabla f_i(\mathbf{x}_k)} + \norm{\nabla f (\mathbf{x}_k)}_2 \leq 2 L_f,
      \label{Eq.5-21}
\end{equation}
which implies $\norm{\mathbf{z}_{i,k}^{s}}_2^2 \leq 4 L_f^2$. Accordingly, $\sigma^2 \stackrel{\Delta}{=} 4 L_f^2$ in Lemma~\ref{lemma:vectorBernstein}. In (\ref{Eq.5-21}), we have used
\begin{equation}
    \norm{\nabla f (\mathbf{x}_k)}_2 \leq \frac{1}{n}\sum_{i=1}^{n} \norm{\nabla f_i (\mathbf{x})}_2 \leq \frac{1}{n}\sum_{i=1}^{n} L_f = L_f,
\end{equation} 
where the triangle inequality and Assumption~\ref{assum:continuityAssum} have been applied. Then, we have
\begin{equation}\label{eq:zVectorBernStein}
    \mathbf{z}_{k} = \frac{1}{|\pazocal{B}_k^g|} \sum_{i=1}^n \mathbf{z}_{i,k}^{s} = \mathbf{g}_k - \nabla f(\mathbf{x}_k).
\end{equation} 
Using~(\ref{eq:zVectorBernStein}) in Lemma~\ref{lemma:vectorBernstein} for $n=|\pazocal{B}_k^g|$ and $ \sigma^2  = 4 L_f^2$ yields
\begin{equation}
    \Pr (\norm{\mathbf{g}_k - \nabla f(\mathbf{x}_k)}_2 \geq \epsilon) \leq \exp \left(-|\pazocal{B}_k^g| \frac{\epsilon^2}{32 L_f^2} + \frac{1}{4}\right).
\end{equation} 
Next, we require that the probability of the gradient deviation $\Pr (\norm{\mathbf{g}_k - \nabla f(\mathbf{x}_k)}_2 \geq \epsilon)$ is less than some $\delta \in (0, 1]$, i.e.,
\begin{equation}\label{eq:epsilonGradLowerBound}
    \exp \left(-|\pazocal{B}_k^g| \frac{\epsilon^2}{32 L_f^2} + \frac{1}{4}\right) \leq \delta \Leftrightarrow \epsilon \geq 4 \sqrt{2}L_f \sqrt{\frac{\ln \frac{1}{\delta} + \frac{1}{4}}{|\pazocal{B}_k^g|}}.
\end{equation} To derive~(\ref{eq:gradientDeviationBound}), we use~(\ref{eq:epsilonGradLowerBound}) in $\norm{\mathbf{g}_k - \nabla f(\mathbf{x}_k)}_2 \geq \epsilon$, along with Assumption~\ref{assum:gradAgreementAssum} to get
\allowdisplaybreaks
\begin{equation}\label{eq:gradientDeviationBound11}
\epsilon \leq \norm{\mathbf{g}_k  - \nabla f(\mathbf{x}_k)}_2 \geq C_g \norm{\mathbf{s}_k}_2^2 \Leftrightarrow  4 \sqrt{2}L_f \sqrt{\frac{\ln \frac{1}{\delta} + \frac{1}{4}}{|\pazocal{B}_k^g|}} \leq  C_g \norm{\mathbf{s}_k}_2^2,
\end{equation} which yields~(\ref{eq:gradientSampling}). Using the complementary probability 
\begin{equation*}
\end{equation*} with $\delta \in (0, 1]$, it is implied that 
\begin{equation*}
    \norm{\mathbf{g}_k - \nabla f(\mathbf{x}_k)}_2 \leq \epsilon   
\end{equation*} is fulfilled with high probability $1-\delta$ when~(\ref{eq:epsilonGradLowerBound}) holds. The latter derives~(\ref{eq:gradientDeviationBound}), and the proof is complete.

\subsection{Proof of Lemma~\ref{lemma:hessianDeviationBound}} \label{supplement:hessianDeviationBound}

Following similar lines to~\citet[Lemma 8 and Theorem 9]{kohler2017} and using $\pazocal{B}_k^H$ instead of $\pazocal{B}_k$ to emphasize that  the deviation bound in~(\ref{eq:hessianDeviationBound}) and the sampling scheme in~(\ref{eq:hessianSampling}) are obtained using information related to $\mathbf{B}_k$, we get
\allowdisplaybreaks
\begin{multline}\label{eq:approxHessianDiag1}
\mathbf{B}_{k} \stackrel{(\ref{eq:approxHessianDiag})}{=}  \frac{1}{\pazocal{S}} \sum_{s=1}^\pazocal{S} \Diag \left( \mathbf{H}_{k} \mathbf{v}_s \odot \mathbf{v}_s \right)  =   \frac{1}{\pazocal{S}} \sum_{s=1}^\pazocal{S} \Diag \left( \left( \frac{1}{|\pazocal{B}_k^g|} \sum_{i \in \pazocal{B}_k^g} \nabla^2 f_i(\mathbf{x}_k)\right)\mathbf{v}_s \odot \mathbf{v}_s  \right) \\ = 
\frac{1}{\pazocal{S}} \sum_{s=1}^\pazocal{S} \frac{1}{|\pazocal{B}_k^g|} \sum_{i \in \pazocal{B}_k^g}   \Diag \left( \nabla^2 f_i(\mathbf{x}_k) \mathbf{v}_s \odot \mathbf{v}_s \right).
\end{multline}
Let 
\begin{equation}\label{eq:BikDef}
    \mathbf{B}_{i,k}^{s} = \Diag \left(  \nabla^2 f_i(\mathbf{x}_k) \mathbf{v}_s \odot \mathbf{v}_s  \right).
\end{equation} 
For  $\mathbf{A}\in \mathbb{R}^{d \times d}$ it is known that $ \norm{\mathbf{A}}_2 \leq \norm{\mathbf{A}}_F$~\citep{golub2012}. Accordingly, for $\mathbf{B}_{i,k}^{s}$ we obtain
\allowdisplaybreaks
\begin{multline}
    \norm{\mathbf{B}_{i,k}^{s}}_2 =  \norm{\Diag \left(  \nabla^2 f_i(\mathbf{x}_k) \mathbf{v}_s \odot \mathbf{v}_s  \right)}_2 \leq \norm{\Diag \left(\nabla^2 f_i(\mathbf{x}_k) \mathbf{v}_s \odot \mathbf{v}_s \right)}_F \\ = \sqrt{\sum_{j=1}^d \left([\nabla^2 f_i(\mathbf{x}_k) \mathbf{v}_s]_{j} [\mathbf{v}_s]_{j}\right)^2}  
    = \sqrt{\sum_{j=1}^d \left([\nabla^2 f_i(\mathbf{x}_k) \mathbf{v}_s]_{j} \right)^2},  
\end{multline} 
where $[\mathbf{v}_s]_{j} = \pm 1$, yielding 
\begin{equation}\label{eq:BikInequality}
    \norm{\mathbf{B}_{i,k}^{s}}_2  \leq \norm{\nabla^2 f_i(\mathbf{x}_k) \mathbf{v}_s}_2 \leq \norm{\nabla^2 f_i(\mathbf{x}_k)}_2\norm{\mathbf{v}_s}_2.
\end{equation} 
For $\mathbf{v}_s \in \mathbb{R}^d$, $\norm{\mathbf{v}_s}_2 \leq \sqrt{d}\norm{\mathbf{v}_s}_\infty$~\citet{gould1999},  
where $\norm{\mathbf{v}_s}_\infty = \max_{ 1 \leq i \leq d} |[\mathbf{v}_s]_i|=1$. This allows us to rewrite~(\ref{eq:BikInequality}) as
\begin{equation}
    \norm{\mathbf{B}_{i,k}^{s}}_2  \leq \sqrt{d}\norm{\nabla^2 f_i(\mathbf{x}_k)}_2\norm{\mathbf{v}_s}_\infty \leq \sqrt{d} \norm{\nabla^2 f_i(\mathbf{x}_k)}_2.
\end{equation} 
As a result 
\begin{equation}\label{eq:BikUpperBound}
    \norm{\mathbf{B}_{i,k}^{s}}_2  \leq \sqrt{d} L_g,
\end{equation} 
because $\norm{\nabla^2 f_i(\mathbf{x}_k)}_2 \leq L_g$ due to Assumption~\ref{assum:continuityAssum}.

To apply the Matrix Bernstein's inequality in Lemma~\ref{lemma:matrixBernsteinTwoSources},  define the centered Hessian matrix
\begin{equation}\label{eq:ZikDef}
    \mathbf{Z}_{i,k}^{s} = \mathbf{B}_{i,k}^{s} - \Diag(\nabla^2 f (\mathbf{x}_k)),
\end{equation} 
where $i=1,\dots,|\pazocal{B}_k^H|$. 

Let $n'=|\pazocal{B}_k^H|$. From Lemma~\ref{lemma:diagSpec}, using Assumption~\ref{assum:continuityAssum}, and applying the triangle inequality we have 
\begin{equation}\label{eq:diagHessUpperLips}
    \norm{\Diag(\nabla^2 f(\mathbf{x}))}_2 \leq \norm{\nabla^2 f(\mathbf{x})}_2 \leq   \frac{1}{n'}\sum_{i=1}^{n'} \norm{\nabla^2 f_i(\mathbf{x})}_2\leq \frac{1}{n'}\sum_{i=1}^{n'} L_g \leq L_g,
\end{equation} 
Using~(\ref{eq:BikUpperBound}),~(\ref{eq:diagHessUpperLips}), and applying triangle inequality yields
\begin{equation}
    \norm{\mathbf{Z}_{i,k}^{s}}_2 = \norm{\mathbf{B}_{i,k}^{s}  - \Diag(\nabla^2 f(\mathbf{x}))}_2 \leq  \norm{\mathbf{B}_{i,k}^{s} }_2 + \norm{\Diag(\nabla^2 f(\mathbf{x}))}_2 \leq (\sqrt{d}+1) L_g,
\end{equation} 
which implies
\begin{equation}\label{eq:ZikUpperBound}
   \norm{\mathbf{Z}_{i,k}^{s}}_2 \leq  \sqrt{d} L_g,
\end{equation} 
as $d \gg 1$. Let
\begin{equation}\label{eq:zMatrixBernStein}
    \mathbf{Z}_{k} \stackrel{\Delta}{=} \frac{1}{\pazocal{S}} \sum_{s=1}^\pazocal{S}  \frac{1}{|\pazocal{B}_k^H|} \sum_{i \in \pazocal{B}_{k}^H}  \mathbf{Z}_{i,k}^{s} 
    \stackequal{(\ref{eq:ZikDef})}{(\ref{eq:BikDef})}
    \mathbf{B}_k - \Diag(\nabla^2 f(\mathbf{x}_k)).
\end{equation} 
Let also
\begin{multline}
     \sigma^2 \stackrel{\Delta}{=}\left\| \sum_{s=1}^\pazocal{S}  \sum_{i \in \pazocal{B}_{k}^H} \mathbb{E} \left[ {\mathbf{Z}_{i,k}^s}^2 \right] \right\|_2 \leq  \sum_{s=1}^\pazocal{S}  \sum_{i \in \pazocal{B}_{k}^H} \left\| \mathbb{E} \left[ {\mathbf{Z}_{i,k}^s}^2 \right] \right\|_2  \leq  \sum_{s=1}^\pazocal{S}  \sum_{i \in \pazocal{B}_{k}^H} \mathbb{E} \left[ \left\| {\mathbf{Z}_{i,k}^s}^2 \right\|_2 \right] \\ \leq  \sum_{s=1}^\pazocal{S}  \sum_{i \in \pazocal{B}_{k}^H}  \mathbb{E} \left[ \left\| {\mathbf{Z}_{i,k}^s} \right\|_2 \left\| {\mathbf{Z}_{i,k}^s} \right\|_2 \right]  \leq {d} \: \pazocal{S} \: |\pazocal{B}_k^H| \:L_g^2,
\end{multline} 
which implies $\sigma^2 \leq {d} \:\pazocal{S} \:|\pazocal{B}_k^H| \: L_g^2$. Also let $K \stackrel{\Delta}{=} \sqrt{d}\: L_g$. 
Using the latter in Lemma~\ref{lemma:matrixBernsteinTwoSources} implies
\begin{equation}\label{eq:matrixBernsteinCases}
    P \left(\left\| \sum_{s=1}^\pazocal{S}  \sum_{i \in \pazocal{B}_{k}^H} \mathbf{Z}_{i,k}^s \right\|_2 \geq t \right) \leq 
    \begin{cases}
        2d \exp \left(\frac{3}{8} \frac{-t^2}{d \pazocal{S} |\pazocal{B}_k^H| L_g^2} \right), \quad t \leq \frac{\sigma^2}{\sqrt{d}L_g} \\
        2d \exp \left(\frac{3}{8} \frac{-t}{\sqrt{d}L_g} \right), \quad t > \frac{\sigma^2}{\sqrt{d}L_g}.
    \end{cases}
\end{equation} 
Using~(\ref{eq:zMatrixBernStein}) in~(\ref{eq:matrixBernsteinCases}) for $t \leq \frac{\sigma^2}{\sqrt{d}L_g}$ and setting $\epsilon =  \frac{t}{|\pazocal{B}_k^H| \pazocal{S}}$ yields
\begin{equation}
    \Pr (\norm{\mathbf{B}_k - \Diag(\nabla^2 f(\mathbf{x}_k))}_2 \geq \epsilon) \leq   2 d\exp \left(-\frac{3}{8} \pazocal{S} |\pazocal{B}_k^H|  \left(\frac{\epsilon}{\sqrt{d} L_g}\right)^2\right).
\end{equation} 
For some $\delta \in (0, 1]$, we are interested in the upper bound 
\begin{equation}
    \Pr (\norm{\mathbf{B}_k - \Diag(\nabla^2 f(\mathbf{x}_k))}_2 \geq \epsilon) \leq \delta,
\end{equation} 
which implies
\begin{equation} \label{eq:epsilonHessianLowerBound}
   2 d\exp \left(-\frac{3}{8} \pazocal{S} |\pazocal{B}_k^H|  \frac{\epsilon^2}{d L_g^2}\right) \leq \delta \Leftrightarrow \epsilon \geq \sqrt{d} L_g \sqrt{ \frac{\ln \frac{2d}{\delta}}{\pazocal{S}|\pazocal{B}_k^H|}}.
\end{equation} 
Similarly, using~(\ref{eq:zMatrixBernStein}) in~(\ref{eq:matrixBernsteinCases}) for $t > \frac{\sigma^2}{\sqrt{d}L_g}$ and setting $\epsilon =  \frac{t}{|\pazocal{B}_k^H| \pazocal{S}}$ yields
\begin{equation}\label{eq:probUpperBound1}
    \Pr (\norm{\mathbf{B}_k - \Diag(\nabla^2 f(\mathbf{x}_k))}_2 \geq \epsilon) \leq   2 d\exp \left(-\frac{3}{8} \pazocal{S} |\pazocal{B}_k^H|  \frac{\epsilon}{\sqrt{d} L_g}\right).
\end{equation} 
Again, we are interested in the probability $\Pr (\norm{\mathbf{B}_k - \Diag(\nabla^2 f(\mathbf{x}_k))}_2) \geq \epsilon)$ is less than some $\delta \in (0, 1]$, i.e., 
\begin{equation} \label{eq:epsilonHessianLowerBound1}
    2 d\exp \left(-\frac{3}{8} \pazocal{S} |\pazocal{B}_k^H|  \frac{\epsilon}{\sqrt{d} L_g}\right) \leq \delta \Leftrightarrow \epsilon \geq \sqrt{d} L_g \frac{\ln \frac{2d}{\delta}}{\pazocal{S}|\pazocal{B}_k^H|}.
\end{equation}
We have that for $x \leq 1$ that $e^{-x} \leq e^{-x^2}$. Thus, for $d \gg 1$ we have $\frac{\epsilon}{\sqrt{d} L_g} < 1$ and the tightest upper bound of $\Pr (\norm{\mathbf{B}_k - \Diag(\nabla^2 f(\mathbf{x}_k))}_2 \geq \epsilon)$ is~(\ref{eq:probUpperBound1}). (\ref{eq:epsilonHessianLowerBound1}) indicates how large \( \epsilon \) must be for the probability of a deviation in~(\ref{eq:probUpperBound1}) to be at most \( \delta \), depending on the number of Hutchinson samples \( \pazocal{S} \), the mini-batch size \( |\pazocal{B}_k^H| \), the parameter dimension \( d \), and the smoothness constant \( L_g \). Next, using~(\ref{eq:epsilonHessianLowerBound1}) in Assumption~\ref{assum:hessAgreementAssum}, we get
\allowdisplaybreaks
\begin{equation}\label{eq:hessianDeviationBound11}
    \epsilon \leq \norm{\mathbf{B}_k  - \Diag(\nabla^2 f(\mathbf{x}_k))}_2 \leq C_B \: \norm{\mathbf{s}_k}_2 \Leftrightarrow  \sqrt{d} L_g \frac{\ln \frac{2d}{\delta}}{\pazocal{S}|\pazocal{B}_k^H|} \leq  C_B \: \norm{\mathbf{s}_k}_2 \Leftrightarrow  |\pazocal{B}_k^H| \geq \sqrt{d} L_g \frac{\ln \frac{2d}{\delta}}{\pazocal{S}\norm{\mathbf{s}_k}_2 C_B},
\end{equation} which yields~(\ref{eq:hessianSampling}). 

Using the complementary bound 
\begin{equation*}
\Pr (\norm{\mathbf{B}_k - \Diag(\nabla^2 f(\mathbf{x}_k))}_2 \leq \epsilon) \geq 1 - \delta,
\end{equation*} 
it is implied that 
\begin{equation*}
    \norm{\mathbf{B}_k - \Diag(\nabla^2 f(\mathbf{x}_k))}_2 \leq \epsilon   
\end{equation*} 
is fulfilled with high probability $1-\delta$ when~(\ref{eq:epsilonHessianLowerBound1}) holds. The latter is~(\ref{eq:hessianDeviationBound}) and the proof is complete. 

\subsection{Proof of Lemma~\ref{lemma:nursolution}}\label{appendix:nursolution}

In the following, $\mathbf{g}_k$ and $\mathbf{B}_k$ are utilized instead of $\nabla f(\mathbf{x}_k)$ and $\nabla^2 f(\mathbf{x}_k)$. This substitution is performed because $\mathbf{g}_k$ and $\mathbf{B}_k$ are directly used in Algorithms~\ref{alg:mainAlg} and~\ref{alg:newton}, both of which operate on data batches. In this manner, $\mathbf{s}(\nu, r)$ is now defined in terms of $\mathbf{B}_k$ and $\mathbf{g}_k$, rather than $\Diag(\nabla^2 f(\mathbf{x}_k))$ and $\nabla f(\mathbf{x}_k)$, respectively.

To obtain the minimizer $\mathbf{s}(\nu, r)$ in~(\ref{eq:solutionhopt1}), we need to solve w.r.t. $\nu$ and $r$ the system of equations
\begin{equation}\label{eq:systemnur}
\partial_{r} \mathscr{L}_{\xi} (\nu, r) =0\quad \text{and} \quad \partial_{\nu} \mathscr{L}_{\xi} (\nu, r) = 0.
\end{equation} 
The solution $(\nu, r)$ in~(\ref{eq:systemnur}) can also be computed sequentially. We start with the case $\nu > 0$. In this case, we can
first solve $\allowbreak \partial_{r} \mathscr{L}_{\xi} (\nu, r) = 0$ w.r.t. $r \in \pazocal{D}_\nu$. Then, the optimal $r$, can be used to solve $\partial_{\nu} \mathscr{L}_{\xi} (\nu, r) = 0$ w.r.t. $\nu$ to obtain the optimal $\nu > 0$. Solving $\partial_{r} \mathscr{L}_{\xi} (\nu, r) = 0$ w.r.t. $r$, yields
\allowdisplaybreaks
\begin{equation}
\label{eq:rgather}
\frac{\nu}{4}\:\Biggl\{\left[-\mathbf{g}_k^T \left(\mathbf{B}_k + \frac{\nu \: r}{2} \: \mathbf{I}\right)^{-1}\right]  \left[-\left(\mathbf{B}_k + \frac{\nu \: r}{2} \: \mathbf{I}\right)^{-1}\mathbf{g}_k\right] \Biggr\} - \frac{\nu \: r^2}{4} =  0 \stackrel{(\ref{eq:solutionh})}{\Leftrightarrow} \\ 
\nu\:(\norm{\mathbf{s} (\nu, r)}_2^2 - r^2) = 0.
\end{equation} 
For $\mathbf{s} (\nu, r) \neq \bm 0$, $r \in \pazocal{D}_\nu$, and $\nu > 0$.  Fom~(\ref{eq:rgather}), we obtain the root 
\begin{equation}\label{eq:psi1}
r = \norm{\mathbf{s} (\nu, r)}_2.
\end{equation} 
Then, the optimal $\nu$ can be computed by solving $\partial_{\nu} \mathscr{L}_{\xi} (\nu, r) \stackrel{(\ref{eq:lagrangianscr})}{=} 0$ w.r.t. $\nu$, i.e., 
\begin{equation}\label{eq:nugather}
\frac{r}{4}\:\Biggl\{\left[-\mathbf{g}_k^T \left(\mathbf{B}_k + \frac{\nu \: r}{2} \: \mathbf{I}\right)^{-1}\right] \left[-\left(\mathbf{B}_k + \frac{\nu \: r}{2} \: \mathbf{I}\right)^{-1}\mathbf{g}_k\right] \Biggr\} - \frac{\xi}{6} - \frac{r^3}{12} =  0,  
\end{equation} 
which is rewritten as
\begin{equation}\label{eq:nugather1}
\frac{r}{4}\:\norm{\mathbf{s} (\nu, r)}_2^2 - \frac{\xi}{6} - \frac{r^3}{12} = 0 
\end{equation} 
for some $\xi>0$. Substituting~(\ref{eq:psi1}) in~(\ref{eq:nugather1}), the optimal $r$ is given by
\begin{equation}\label{eq:optr}
r = \sqrt[3]{\xi}.
\end{equation} 
Using~(\ref{eq:optr}) in~(\ref{eq:nugather1}), the optimal $\nu$ can be computed by solving 
\begin{equation}\label{eq:psi2}
\omega (\nu, r) = \norm{\mathbf{s} (\nu, r)}_2 - \sqrt[3]{\xi} = 0,
\end{equation} w.r.t. $\nu$ for $r$ fixed. It is shown in~\citet[Section 7.3.3]{conn2000} that instead of solving~(\ref{eq:psi2}), it is more preferable to solve
\begin{equation}
\phi (\nu, r) = \frac{1}{\norm{\mathbf{s} (\nu, r)}_2} - \frac{1}{\sqrt[3]{\xi}} = 0.
\end{equation} 
We conclude with the case $\nu = 0$. In this case, the minimizer $\mathbf{s}(\nu, r)$ in~(\ref{eq:solutionhopt1}) is handled by Algorithm~\ref{alg:newton}, which concludes the proof.

\subsection{Preliminaries for Lemmata~\ref{lemma:hessianDeviationBound} and~\ref{lemma:diagHessian}}\label{supplement:diagHessian}

\begin{restatable}{lem}{diagSpec} 
For $\mathbf{M}\in \mathbb{R}^{d \times d}$ and  $\operatorname{Diag}(\mathbf{M})$ we have 
\begin{equation}\label{eq:specIneq1}
    \norm{\mathbf{M}}_2 \geq \norm{\operatorname{Diag}(\mathbf{M})}_2.
\end{equation}
\begin{proof}
First, we prove
\begin{equation}\label{eq:diagNormSpec}
\norm{\operatorname{Diag}(\mathbf{M})}_2 = \max_{k} |\operatorname{Diag}(\mathbf{M})_{kk}|, 
\end{equation} 
Let $d^* = \max_k\operatorname{Diag}(\mathbf{M})_{kk}$. Then, from the definition of the spectral norm, we have
\begin{equation}
\norm{\operatorname{Diag}(\mathbf{M})}_2 = \max_{\norm{\mathbf{x}}_2=1} \norm{\operatorname{Diag}(\mathbf{M})\mathbf{x}}_2 \leq \max_{\norm{\mathbf{x}}_2=1} \sqrt{\sum_k (\operatorname{Diag}(\mathbf{M})_{kk} \mathbf{x}_k)^2} \leq 
|d^*| \max_{\norm{\mathbf{x}}_2=1} \sqrt{\sum_{k} \mathbf{x}_k^2},
\end{equation} 
which leads to 
\begin{equation}\label{eq:diagLower}
\norm{\operatorname{Diag}(\mathbf{M})}_2 \leq |d^*|.
\end{equation} 
Let $\mathbf{e}_m$ be the vector of all zeros except a 1 in the $m$th position, where $m = \argmax_i \operatorname{Diag}(\mathbf{M})_{ii}$. Then,  
\begin{equation}
\norm{\operatorname{Diag}(\mathbf{M})}_2 = \max_{\norm{\mathbf{x}}_2=1} \norm{\operatorname{Diag}(\mathbf{M})\mathbf{x}}_2 \geq  \norm{\operatorname{Diag}(\mathbf{M})\mathbf{e}_m}_2,
\end{equation} which leads to 
\begin{equation}\label{eq:diagUpper}
\norm{\operatorname{Diag}(\mathbf{M})}_2 \geq |d^*|.
\end{equation} 
(\ref{eq:diagLower}) and~(\ref{eq:diagUpper}) imply~(\ref{eq:diagNormSpec}). From the definition of the spectral norm, we have
{\allowdisplaybreaks
\begin{equation}\label{eq:diadNormUpper}
\norm{\mathbf{M}}_2 = \max_{\norm{\mathbf{x}}_2=1} \sqrt{\frac{\mathbf{x}^T\mathbf{M}^T\mathbf{M}\mathbf{x}}{\norm{\mathbf{x}}_2^2}} = \max_{\norm{\mathbf{x}}_2=1} \frac{\norm{\mathbf{M}\mathbf{x}}_2}{\norm{\mathbf{x}}_2} =  \max_{\mathbf{x},\mathbf{y} \neq \bm{0}} \frac{|\mathbf{x}^T\mathbf{M}\mathbf{y}|}{\norm{\mathbf{x}}_2\norm{\mathbf{y}}_2} \geq |\mathbf{e}_j^T\mathbf{M}\mathbf{e}_i| = |\mathbf{M}_{ij}|.
\end{equation}} 
Restricting~(\ref{eq:diadNormUpper}) in the diagonal elements of $\mathbf{M}$ gives 
\begin{equation}
\norm{\mathbf{M}}_2 \geq \max_{i} |\operatorname{Diag}(\mathbf{M})_{ii}|,
\end{equation} 
which using~(\ref{eq:diagNormSpec}) leads to
\begin{equation} \label{eq:specIneq}
\norm{\mathbf{M}}_2 \geq \norm{\operatorname{Diag}(\mathbf{M})}_2,
\end{equation} which is~(\ref{eq:specIneq1}). This inequality becomes an equality when $\mathbf{M}$ is a diagonal matrix. Therefore,~(\ref{eq:specIneq1}) provides the tightest possible bound in this case.
\end{proof}
\label{lemma:diagSpec}
\end{restatable}

\subsection{Lemmata~\ref{lemma:diagHessian} and~\ref{lemma:lemma1Nesterov2006} }\label{supplement:diagHessianANDlemma1Nesterov2006}

Given the Lipschitz continuity assumption of $\nabla^2 f(\mathbf{x})$, Lemma~\ref{lemma:diagHessian} introduces the Lipschitz continuity of $\Diag (\nabla^2 f(\mathbf{x}))$. Lemma~\ref{lemma:diagHessian} is used in Lemmata~\ref{lemma:lemma1Nesterov2006} and~\ref{lemma:localOptimalityCriterion}. Lemma~\ref{lemma:lemma1Nesterov2006} is an adaptation of \citet[Lemma 1]{nesterov2006} tailored to fit the context of this analysis. Lemma~\ref{lemma:lemma1Nesterov2006} is used in Lemmata~\ref{lemma:lemma2Nesterov2006},~\ref{lemma:lemma3Nesterov2006}, and~\ref{lemma:lemma4Nesterov2006}.

\begin{restatable}{lem}{diagHessian}
If $\nabla^2 f(\mathbf{x})$ is Lipschitz continuous in $\pazocal{F}$,  
$\Diag (\nabla^2 f(\mathbf{x}))$ is also Lipschitz continuous, i.e.,
\begin{equation}\label{eq:diagHessian}
\norm{\operatorname{Diag} (\nabla^2 f(\mathbf{x})) - \operatorname{Diag} (\nabla^2 f(\mathbf{y}))}_2 \leq L_H \norm{\mathbf{x} - \mathbf{y}}_2.
\end{equation}
\begin{proof}
The proof of the lemma is easily obtained by combining the Lipschitz continuity of the Hessian matrix in Assumption~\ref{assum:continuityAssum} (see also Remark~\ref{rmrk:continuityAssum1}) with Lemma~\ref{lemma:diagSpec} in Appendix~\ref{supplement:diagHessian}. In Lemma~\ref{lemma:diagSpec}, we use $\mathbf{M} = \nabla^2 f(\mathbf{x}) - \nabla^2 f(\mathbf{y})$.
\end{proof}
\label{lemma:diagHessian}
\end{restatable} 

\begin{restatable}{lem}{lemma1Nesterov2006} 
For any $\mathbf{x}$ and $\mathbf{y}$ in $\pazocal{F}$, we have
\begin{equation}\label{eq:lemma1Nesterov20061}
    \norm{\nabla f(\mathbf{y}) - \nabla f(\mathbf{x}) - \Diag(\nabla^2 f(\mathbf{x})) (\mathbf{y}-\mathbf{x})}_2  \leq \frac{L_H}{2} \norm{\mathbf{y}-\mathbf{x}}_2
\end{equation} and
\begin{equation}\label{eq:lemma1Nesterov20062}
    \Bigl|f(\mathbf{y}) - f(\mathbf{x}) - \nabla f(\mathbf{x})^T (\mathbf{y}-\mathbf{x}) -  \frac{1}{2} (\mathbf{y} - \mathbf{x})^T\Diag (\nabla^2 f(\mathbf{x}))(\mathbf{y} - \mathbf{x}) \Bigr| \leq \frac{L_H}{6} \norm{\mathbf{y} - \mathbf{x}}_2^3. 
\end{equation}
\begin{proof}
    \citet[Lemma 1]{nesterov2006} does not make any assumption on the structure of $\nabla^2 f(\mathbf{x})$. The only assumption to derive \citet[Lemma 1]{nesterov2006} is the Lipschitz continuity of $\nabla^2 f(\mathbf{x})$. Thus, given Lemma~\ref{lemma:diagHessian} and following the proof guidelines in~\citet[Lemma 1]{nesterov2006},~(\ref{eq:lemma1Nesterov20061}) and~(\ref{eq:lemma1Nesterov20062}) are easily derived, which concludes the proof.
\end{proof}
\label{lemma:lemma1Nesterov2006}
\end{restatable}

\subsection{Details for Algorithm~\ref{alg:newton}}\label{supplement:algDetails}

Here, Algorithm~\ref{alg:newton} is discussed when $\mathbf{B}_k$ and $\mathbf{g}_k$ are utilized instead of $\nabla f(\mathbf{x}_k)$ and $\nabla^2 f(\mathbf{x}_k)$, respectively. This is done because $\mathbf{B}_k$ and $\mathbf{g}_k$ are directly involved in the application of Algorithm~\ref{alg:mainAlg}, which operates on data batches. In this manner, $\mathbf{s}(\nu, r)$ and $\phi(\nu, r)$ are now defined with respect to $\mathbf{B}_k$ and $\mathbf{g}_k$, rather than $\Diag(\nabla^2 f(\mathbf{x}_k))$ and $\nabla f(\mathbf{x}_k)$, respectively.

Lemma~\ref{lemma:newtonphi} provides some useful properties of $\phi(\nu, r)$ exploited in line~\ref{lineAlgNewton:newtonUpdate} of Algorithm~\ref{alg:newton}. Lemma~\ref{lemma:newtonphi} is used in Lemma~\ref{lemma:newtonconv} which shows that for some $r \in \pazocal{D}_\nu$, Newton-Raphson updates in line~\ref{lineAlgNewton:newtonUpdate} of Algorithm~\ref{alg:newton}  converge to the roots of $\phi (\nu, r)=0$ w.r.t. $\nu > 0$. 
However, the Newton-Raphson method may diverge on its own, and appropriate safeguards are necessary to prevent this. These safeguards are adopted from~\citep[Algorithm 7.3.6]{conn2000} in Algorithm~\ref{alg:newton}. Lemma~\ref{lemma:terminationNewtonnu} stems from~\citet[Lemma 7.3.5]{conn2000} adapted to the analysis here. Lemma~\ref{lemma:terminationNewtonnu} provides the termination rule used in Algorithm~\ref{alg:newton}.

\begin{restatable}{lem}{newtonphi} 
Let
\begin{equation}\label{eq:Hxnur}
\tilde{\mathbf{H}}_k (\nu, r) \stackrel{\Delta}{=} \mathbf{B}_k + \frac{\nu \: r}{2} \: \mathbf{I}.
\end{equation} 
Suppose $\mathbf{g}_k \neq \bm{0}$ and $\nu \: r > \max \bigl\{0, \allowbreak  -2\:\lambda_d \bigl(\mathbf{B}_k \bigr) \bigr\}$ for some $\nu>0$ and $r > 0$. Then, the function $\phi (\nu, r)$ is strictly increasing and concave w.r.t. $\nu$ and fixed $r$. The first- and second-order partial derivatives of $\phi(\nu, r)$ w.r.t. $\nu$ are
\begin{equation}\label{eq:firstpartialphi2nuOrg}
\partial_{\nu} \phi (\nu, r) = -\frac{\partial_{\nu} \mathbf{s}(\nu, r)^T \: \mathbf{s}(\nu, r)}{\norm{\mathbf{s}(\nu, r)}_2^3} > 0
\end{equation} 
and
\begin{equation}\label{eq:secondpartialnuphi2Org}
\partial_{\nu}^2 \phi (\nu, r) = 3 \Bigg\{\frac{\bigl(\partial_{\nu} \mathbf{s}(\nu, r)^T\: \mathbf{s}(\nu, r)\bigr)^2}{\norm{\mathbf{s}(\nu, r)}_2^5} -  \frac{\norm{ \partial_{\nu} \mathbf{s}(\nu, r)}_2^2 \norm{\mathbf{s}(\nu, r)}_2^2}{\norm{\mathbf{s}(\nu, r)}_2^5}  \Bigg\} \leq 0,
\end{equation} 
respectively, with
\begin{equation}\label{eq:firstpartialnusOrg}
\partial_{\nu} \mathbf{s}(\nu, r) = -\frac{r}{2}\: \tilde{\mathbf{H}}^{-1}_k (\nu, r)  \: \mathbf{s}(\nu, r).
\end{equation}
\begin{proof} 
\noindent
Following similar lines to~\citet[Lemma 7.3.1]{conn2000}, the first-order partial derivative of $\phi(\nu, r)$  (\ref{eq:systemnur2}) w.r.t. $\nu$ is
\allowdisplaybreaks
\begin{multline}\label{eq:firstpartialphiwrtnu}
\partial_{\nu} \phi (\nu, r) =  \partial_{\nu} \left(\mathbf{s}(\nu, r)^T \mathbf{s}(\nu, r)\right)^{-\frac{1}{2}} - \cancelto{0}{\partial_{\nu} \frac{1}{\sqrt[3]{\xi}}} 
\\ =  -\frac{1}{2}\norm{\mathbf{s}(\nu, r)}_2^{-3} \sum_{\ell=1}^{d} 2\: \partial_{\nu} (\mathbf{s}(\nu, r))_\ell \: (\mathbf{s}(\nu, r))_\ell \stepcounter{equation}\tag{\theequation}  = -\frac{\partial_{\nu} \mathbf{s}(\nu, r)^T \: \mathbf{s}(\nu, r)}{\norm{\mathbf{s}(\nu, r)}_2^3},
\end{multline}
which is~(\ref{eq:firstpartialphi2nuOrg}). 
The second-order partial derivative of $\phi(\nu, r)$ w.r.t. $\nu$ reads 
{\allowdisplaybreaks
\begin{multline} \label{eq:secondpartialnuphi}
\partial_{\nu}^2 \phi(\nu, r) = - \partial_{\nu} \Biggl\{\left(\mathbf{s}(\nu, r)^T \: \mathbf{s}(\nu, r)\right)^{-\frac{3}{2}}  \left(\partial_{\nu} \mathbf{s}(\nu, r)^T \: \mathbf{s}(\nu, r)\right) \Biggr\} \\ = \Biggl\{3\frac{\bigl(\partial_{\nu} \mathbf{s}(\nu, r)^T \: \mathbf{s}(\nu, r)\bigr)^2}{\norm{\mathbf{s}(\nu, r)}_2^5} - \frac{\partial_{\nu}^2 \mathbf{s}(\nu, r)^T\:\mathbf{s}(\nu, r) + \norm{\partial_{\nu}\mathbf{s}(\nu, r)}_2^2 }{\norm{\mathbf{s}(\nu, r)}_2^3} \Biggr\}.
\end{multline}}
The first-order partial derivative of
\begin{equation}\label{eq:sx}
\mathbf{s} (\nu, r) \stackrel{(\ref{eq:solutionh})}{=} - \tilde{\mathbf{H}}^{-1}_k (\nu, r)  \: \mathbf{g}_k
\end{equation} 
w.r.t. $\nu$ is 
\begin{equation} \label{eq:firstpartialnus}
\partial_{\nu} \mathbf{s} (\nu, r) = \frac{r}{2}\: \tilde{\mathbf{H}}^{-2}_k (\nu, r) \:\mathbf{g}_k  \stackrel{(\ref{eq:sx})}{=} -\frac{r}{2}\: \tilde{\mathbf{H}}^{-1}_k (\nu, r)  \:\mathbf{s} (\nu, r),
\end{equation} 
which is~(\ref{eq:firstpartialnusOrg}). From~(\ref{eq:firstpartialnus}) and the assumptions about $\nu$ and $r$, i.e., $ \nu \: r > \max \{0 \allowbreak,  -2\:\lambda_d \bigl(\mathbf{B}_k\bigr) \}$ with $\nu>0$ and $r >0$, we obtain
\begin{equation}\label{eq:firstpartialnusTs}
\partial_{\nu} \mathbf{s} (\nu, r)^T \: \mathbf{s} (\nu, r) =  - \frac{r}{2}~ \mathbf{s} (\nu, r)^T \: \tilde{\mathbf{H}}^{-1}_k (\nu, r)  \: \mathbf{s} (\nu, r) < 0.
\end{equation} 
Using~(\ref{eq:firstpartialnusTs}) in~(\ref{eq:firstpartialphiwrtnu}), we infer that $\phi(\nu, r)$ is strictly increasing w.r.t. $\nu$. The second derivative of~(\ref{eq:sx}) w.r.t. $\nu$ is given by
\begin{equation} \label{eq:secondpartialnus}
\partial_{\nu}^2 \mathbf{s} (\nu, r) = \frac{r^2}{2}~ \tilde{\mathbf{H}}^{-2}_k (\nu, r)  \: \mathbf{s} (\nu, r).
\end{equation} 
From~(\ref{eq:sx}) and (\ref{eq:secondpartialnus}) we get
\begin{equation}\label{eq:secondpartialnusTs}
\partial_{\nu}^2 \mathbf{s} (\nu, r)^T \: \mathbf{s} (\nu, r) = 2~ \norm{\partial_{\nu} \mathbf{s} (\nu, r)}_2^2.
\end{equation} 
The substitution of~(\ref{eq:secondpartialnusTs}) in~(\ref{eq:secondpartialnuphi}) yields
\begin{equation}\label{eq:secondpartialnuphi1}
\partial_{\nu}^2 \phi(\nu, r) =  3 \Bigg\{\frac{\bigl(\partial_{\nu} \mathbf{s}(\nu, r)^T \: \mathbf{s}(\nu, r)\bigr)^2}{\norm{\mathbf{s}(\nu, r)}_2^5} -  
\frac{\norm{ \partial_{\nu} \mathbf{s}(\nu, r)}_2^2 \norm{\mathbf{s}(\nu, r)}_2^2}{\norm{\mathbf{s}(\nu, r)}_2^5}\Bigg\},
\end{equation} which is (\ref{eq:secondpartialnuphi2Org}). The concavity of $\phi(\nu, r)$, w.r.t. $\nu$, i.e., $\partial_{\nu}^2 \phi(\nu, r)\leq 0$, follows by applying the Cauchy-Schwartz inequality, i.e.,$\bigl(\partial_{\nu} \mathbf{s}(\nu, r)^T \: \mathbf{s}(\nu, r)\bigr)^2 \leq \norm{\partial_{\nu} \mathbf{s}(\nu, r)}_2^2 \norm{\mathbf{s}(\nu, r)}_2^2$,
in~(\ref{eq:secondpartialnuphi1}), 
which completes the proof. 
\end{proof}
\label{lemma:newtonphi}
\end{restatable} 

\begin{restatable}{lem}{newtonconv} 
Let $\phi (\nu,r)$ satisfy Lemma~\ref{lemma:newtonphi}. Suppose that for some $\nu>0$ and $r> 0$ we have $\nu \: r > \max \left\{0, -2\:\lambda_d \bigl(\mathbf{B}_k\bigr) \right\}$ and $\phi (\nu, r) < 0$. Then for a fixed $r$, the Newton iterates
\begin{equation}\label{eq:newtoniternur}
\nu^+ \gets \nu - \frac{\phi(\nu, r)}{\partial_{\nu} \phi(\nu, r)}, 
\end{equation} 
will still satisfy $\phi (\nu^+, r) < 0$ and convergence monotonically toward the root $\nu^*$ of $\phi (\nu, r) =0$ w.r.t. $\nu$. The convergence of the Newton iterations w.r.t. $\nu$ is at least linear and ultimately quadratic.
\begin{proof} 
Following similar lines to~\citet[Lemma 7.3.2]{conn2000}, we study the convergence of the Newton iterations in~(\ref{eq:newtoniternur}) w.r.t. $\nu$ when $r$ is fixed. Suppose that $\phi (\nu, r)$ satisfies Lemma~\ref{lemma:newtonphi}, which implies $\partial_{\nu} \phi (\nu, r) > 0$. Then, from the Newton iteration w.r.t. $\nu$ in~(\ref{eq:newtoniternur}), we have
\begin{equation}\label{eq:updatenulinearized}
\phi(\nu, r) + (\nu^{+} - \nu) \:\partial_{\nu} \phi(\nu, r) = 0.
\end{equation} 
According to Lemma~\ref{lemma:newtonphi}, $\phi(\nu, r)$ is concave, i.e., $\partial_{\nu}^2 \phi(\nu,r) \leq 0$. Combining the concavity of $\phi(\nu, r)$ with~(\ref{eq:updatenulinearized}) we get 
\begin{equation*}
\phi (\nu^{+}, r) <  \phi(\nu, r) + (\nu^+ -\nu)\:\partial_{\nu} \phi(\nu, r)  = 0
\end{equation*} 
which proves that  $ \phi(\nu,r) < 0$ is inherited by all Newton iterations w.r.t. $\nu$. Let $(\nu^*, r)$ be the root of $\phi(\nu, r)$. In addition, let $(\nu^I, r)$ be an intermediate point between points $(\nu, r)$ and $(\nu^*, r)$, i.e., $(\nu^I, r) = \alpha~(\nu, r) + (1-\alpha)\:(\nu^*, r)$ with $\alpha \in (0, 1)$. Then, the Taylor expansion about $\nu^*, r$ reads as 
\begin{equation}\label{eq:taylorg}
\phi(\nu, r) = \cancelto{0}{\phi(\nu^*, r)} + \partial_{\nu} \phi(\nu^I, r) (\nu - \nu^*) +  \frac{1}{2}\:\partial_{\nu}^2 \phi(\nu^I, r) (\nu - \nu^*)^2 + \pazocal{O}\left((\nu- \nu^*)^3\right).
\end{equation} 
We assume that $(\nu, r)$ is close to $(\nu^*, r)$. This proximity implies that the last term in equation (\ref{eq:taylorg}) becomes even closer to zero, and consequently, it is omitted from the subsequent analysis. The assumption that $(\nu, r)$ is in proximity to $(\nu^*, r)$ ensures the convergence of the Newton method~\citep{bertsekas2017}. This proximity is achieved by utilizing the safeguarded Newton Algorithm~\ref{alg:newton}. A comprehensive analysis of the safeguarded Newton methodology can be found  in~\citep{conn2000}.

Subtracting $\nu^*$ from both sides of~(\ref{eq:newtoniternur}), i.e., 
\begin{equation}\label{eq:nusubtrack}
\nu^{+} - \nu^* = (\nu - \nu^*) - \frac{\phi (\nu, r)}{\partial_{\nu} \phi (\nu, r)}
\end{equation}
and substituting~(\ref{eq:taylorg}) in~(\ref{eq:nusubtrack}), we arrive at
\begin{equation}\label{eq:taylorg1}
\nu^+ - \nu^* = \left(1-\frac{\partial_{\nu} \phi (\nu^I, r)}{\partial_{\nu} \phi (\nu, r)}\right) (\nu - \nu^*) - \frac{1}{2} \frac{\partial_{\nu}^2 \phi (\nu^I, r)}{\partial_{\nu} \phi (\nu, r)}(\nu - \nu^*)^2.
\end{equation} 
We examine the following cases:
\begin{enumerate}
\item If $\Bigl| 1 \allowbreak - \frac{ \partial_{\nu} \phi (\nu^I, r)}{\partial_{\nu} \phi (\nu, r)} \Bigr| > 1$,~(\ref{eq:taylorg1}) diverges.
\item If $\Bigl| 1 \allowbreak-\frac{\partial_{\nu} \phi (\nu^I, r)}{\partial_{\nu} \phi (\nu, r)} \Bigr| < 1$, we have at least linear convergence in~(\ref{eq:taylorg1}).
\item If $\Bigl| 1-\frac{\partial_{\nu} \phi (\nu^I, r)}{\partial_{\nu} \phi (\nu, r)} \Bigr| = 0$, we have quadratic convergence as the linear term vanishes in~(\ref{eq:taylorg1}). 
\end{enumerate} 
From the concavity of $\phi(\nu, r)$ w.r.t. $\nu$, we have that $\partial_{\nu} \phi(\nu, r)$ is decreasing, which implies that $\Bigl| 1-\frac{\partial_{\nu}\phi (\nu^I, r)}{\partial_{\nu} \phi (\nu, r)} \Bigr| < 1$. Thus, the Newton iterations w.r.t. $\nu$ in ~(\ref{eq:newtoniternur}) converge at least linearly and ultimately quadratically, which completes the proof. 
\end{proof}
\label{lemma:newtonconv}
\end{restatable}

\begin{restatable}{rmrk}{safequardr} 
Lemma~\ref{lemma:newtonphi} implies that $\partial_{\nu} \phi(\nu, r) > 0$. Suppose that for some $\nu$ and $r$ we have $\phi(\nu, r) < 0$, i.e., Lemma~\ref{lemma:newtonconv} holds. Then, by~(\ref{eq:newtoniternur}) we have that $\nu^+ > \nu$. Given $\nu^+ > \nu$ and the initial values of $r$ and $\nu$ in lines~\ref{lineAlgNewton:initr},~\ref{lineAlgNewton:initnu1}, and~\ref{lineAlgNewton:initnu2} of Algorithm~\ref{alg:newton}, the optimal $r$ and $\nu$ always satisfy $\mathbf{B}_k + \frac{\nu \: r }{2}\: \mathbf{I}\succ \bm{0}$.
\label{rmrk:safequardr}
\end{restatable}

\begin{restatable}{lem}{terminationNewtonnu} 
For some $\nu>0$ and $r>0$, suppose $\nu \: r > \max \{0, -2\:\lambda_d (\mathbf{B}_k) \}$ and
\begin{equation}
|\norm{\mathbf{s} (\nu, r)}_2 - \xi^{\frac{1}{3}}| \leq \kappa_{\text{easy}}~\xi^{\frac{1}{3}}
\end{equation} 
with $\kappa_{\text{easy}} \in (0, 1)$. Then 
\begin{equation}
\mathfrak{\hat{m}} (\mathbf{s} (\nu, r)) \leq (1 - \kappa_{\text{easy}})^2 ~\mathfrak{\hat{m}} (\mathbf{s}_{k}^*),
\end{equation} 
where $\mathbf{s}_{k}^*$ is the minimizer of~(\ref{prob:cubicOursEquiv}) that satisfies Lemma~\ref{lemma:globalSol}, and 
\begin{equation}\label{eq:basicCubicProblemApproxHat}
    \mathfrak{\hat{m}}(\mathbf{s}) \stackrel{\Delta}{=}  F(\mathbf{x}_k) +  \mathbf{g}_k^T \mathbf{s} + \frac{1}{2}\: \mathbf{s}^T \: \mathbf{B}_k \: \mathbf{s}.
\end{equation}
\label{lemma:terminationNewtonnu}
\begin{proof} A similar proof to that in~\citet[Lemma 7.3.5]{conn2000} can be devised for the constraint $\norm{\mathbf{s}}_2^3 \leq \xi$.
\end{proof}
\end{restatable}

\subsection{Preliminaries for Theorem~\ref{thm:diagmodelmin}}
\label{supplement:diagmodelminRelatedLemmata}

Here,~\citet[Lemma 1]{nesterov2006},~\citet[Lemma 2]{nesterov2006},~\citet[Lemma 3]{nesterov2006},~\citet[Lemma 4]{nesterov2006}, and~\citet[Lemma 5]{nesterov2006}, correspond to Lemmata~\ref{lemma:lemma1Nesterov2006},~\ref{lemma:lemma2Nesterov2006},~\ref{lemma:lemma3Nesterov2006},~\ref{lemma:lemma4Nesterov2006} and~\ref{lemma:lemma5Nesterov2006}, respectively, proven for $\Diag (\nabla^2 f(\mathbf{x}))$ in place of $\nabla^2 f(\mathbf{x})$. In addition, Lemma~\ref{lemma:localOptimalityCriterion} is proven providing details not included in~\citep[Section 2]{nesterov2006}.

Let the level set
\begin{equation}
    \mathfrak{L} (c) = \{\mathbf{x} \in \mathbb{R}^d \colon  f(\mathbf{x}) \leq c\}, 
\end{equation} and assume that $\pazocal{F} \subseteq \mathfrak{L} (f(\mathbf{x}_0))$. Let
\begin{equation}\label{eq:auxhatfDef} 
    \hat{f} (\mathbf{x}, \mathbf{y}) = f(\mathbf{x}) + \nabla f(\mathbf{x})^{T}(\mathbf{y} - \mathbf{x})  + \frac{1}{2} (\mathbf{y} - \mathbf{x})^T\Diag (\nabla^2 f(\mathbf{x}))(\mathbf{y} - \mathbf{x}) + \frac{M}{6} \norm{\mathbf{y} - \mathbf{x}}_2^3,
\end{equation} 
\begin{equation}\label{eq:TMx}
    T_M (\mathbf{x}) = \argmin_{\mathbf{y}} \hat{f} (\mathbf{x}, \mathbf{y}),
\end{equation} 
and 
\begin{equation}\label{eq:barfmin}
    \bar{f}_M (\mathbf{x}) = \min_{\mathbf{y}} \hat{f} (\mathbf{x}, \mathbf{y}).
\end{equation} 
That is, 
\begin{equation}\label{eq:auxhatTM}
    \bar{f}_M (\mathbf{x}) = \hat{f} (\mathbf{x}, T_M (\mathbf{x})).
\end{equation} 
To compute $T_M(\mathbf{x})$ in~(\ref{eq:TMx}), we solve $\nabla_{\mathbf{y}} \hat{f} (\mathbf{x}, \mathbf{y}) = \bm{0}$, i.e.,
\begin{equation}\label{eq:nablaybarf}
    \nabla f(\mathbf{x}) + \Diag (\nabla^2 f(\mathbf{x})) (\mathbf{y}-\mathbf{x}) + \frac{M}{2} \norm{\mathbf{y}-\mathbf{x}}_2 (\mathbf{y}-\mathbf{x})
     = \bm{0}.
\end{equation} 
Let $r_M(\mathbf{x}) = \norm{\mathbf{x} - T_M(\mathbf{x})}_2$. For $\mathbf{y} = T_M (\mathbf{x})$ in~(\ref{eq:nablaybarf}), if we multiply both sides of~(\ref{eq:nablaybarf}) by $T_M(\mathbf{x})-\mathbf{x}$ we arrive at 
\begin{equation}\label{eq:nablaybarf1}
    \nabla f(\mathbf{x})^T (T_M(\mathbf{x})-\mathbf{x}) +  (T_M(\mathbf{x})-\mathbf{x})^T\Diag(\nabla^2 f(\mathbf{x})) (T_M(\mathbf{x})-\mathbf{x}) +  \frac{M}{2}\norm{(T_M(\mathbf{x})-\mathbf{x})}_2^3 = 0.
\end{equation}

\begin{restatable}{lem}{lemma2Nesterov2006} 
For any $\mathbf{x} \in \pazocal{F}$ with $f(\mathbf{x}) \leq f(\mathbf{x}_0)$, 
we have 
\begin{equation}\label{eq:lemma2Nesterov}
     \nabla f(\mathbf{x})^T (\mathbf{x}-T_M(\mathbf{x})) \geq 0.
\end{equation} 
Moreover, if $M \geq \frac{2}{3}L_H$ and $\mathbf{x} \in \mathrm{int}\:\pazocal{F}$, then
\begin{equation}\label{eq:TML}
   T_M (\mathbf{x}) \in \mathfrak{L} (f(\mathbf{x})).
\end{equation}
\begin{proof}
Using Corollary~\ref{cor:globalSolDiag}, we obtain
\begin{equation} \label{eq:globalSolDiag1}
\Diag(\nabla^2 f(\mathbf{x})) + \frac{M}{2} \norm{\mathbf{x} - \mathbf{y}}_2~ \mathbf{I} \succeq 0,
\end{equation} 
which when pre-multiplied by  $(T_M(\mathbf{x}) - \mathbf{x})^T$ and post-multiplied by $(T_M(\mathbf{x}) - \mathbf{x})$ yields
\begin{equation}\label{eq:globalSolDiag2}
    (T_M (\mathbf{x}) - \mathbf{x})^T\Diag(\nabla^2 f(\mathbf{x}))(T_M (\mathbf{x}) -  \mathbf{x}) +  \frac{M}{2} \norm{T_M(\mathbf{x}) - \mathbf{x}}_2^3 \geq 0.
\end{equation} 
Then, combining~(\ref{eq:nablaybarf1}) with~(\ref{eq:globalSolDiag2}) we arrive at~(\ref{eq:lemma2Nesterov}). 

\noindent \textbf{Assumption and Contradiction.} We now show that $T_M(\mathbf{x}) \in \mathfrak{L}(f(\mathbf{x}))$. Following the approach of~\citet[Lemma 2]{nesterov2006}, we proceed by contradiction by assuming $T_M(\mathbf{x}) \notin \mathfrak{L}(f(\mathbf{x}))$ for $M \geq \frac{2}{3}L_H$. We then show that this assumption cannot hold, leading to a contradiction. Thus, we conclude that $T_M(\mathbf{x}) \in \mathfrak{L}(f(\mathbf{x}))$.

By assuming $T_M (\mathbf{x}) \notin \mathfrak{L} (f(\mathbf{x}))$, there exist
\begin{equation}\label{eq:ya}
    \mathbf{y}_{\alpha} = (1-\alpha)\: \mathbf{x} + \alpha \: T_M(\mathbf{x}),
\end{equation} 
with $\alpha \in [0, 1]$, such that 
\begin{equation}\label{eq:contradictionIneq}
    f(\mathbf{y}_{\alpha}) > f(\mathbf{x}).
\end{equation} Using the upper bound of~(\ref{eq:lemma1Nesterov20062}) with $\mathbf{y} = \mathbf{y}_{\alpha}$ we obtain
\allowdisplaybreaks
\begin{multline}
    f(\mathbf{y}_{\alpha}) \leq  f(\mathbf{x}) + \nabla f(\mathbf{x})^T(\mathbf{y}_{\alpha}-\mathbf{x}) +  \frac{1}{2} (\mathbf{y}_{\alpha}-\mathbf{x})^T\Diag(\nabla^2 f(\mathbf{x}))(\mathbf{y}_{\alpha}-\mathbf{x}) +  \frac{L_H}{6}\norm{\mathbf{y}_{\alpha}-\mathbf{x}}_2^3  \stackrel{(\ref{eq:ya})}{=} \\
    f(\mathbf{x}) + \alpha \nabla f(\mathbf{x})^T (T_M(\mathbf{x})-\mathbf{x}) +  \frac{\alpha^2}{2} (T_M(\mathbf{x})-\mathbf{x})^T\Diag(\nabla^2 f(\mathbf{x}))(T_M(\mathbf{x})-\mathbf{x}) + \frac{\alpha^3 L_H}{6}  \norm{T_M(\mathbf{x}) - \mathbf{x}}_2^3,
\end{multline} 
which using~(\ref{eq:nablaybarf1}) implies
\begin{equation}\label{eq:faf}
   f(\mathbf{y}_{\alpha}) - f(\mathbf{x}) \leq  
    - \underbrace{\left(\alpha - \frac{\alpha^2}{2}\right) \nabla f(\mathbf{x})^T (\mathbf{x} - T_M(\mathbf{x}))}_{\geq 0 \text{ as } \alpha \in [0, 1] \text{ and}~(\ref{eq:ffhatTMIneq1}) }  - \frac{\alpha^2}{2}\left(\frac{M}{2}-\frac{\alpha L_H}{3}\right)\norm{T_M(\mathbf{x}) - \mathbf{x}}_2^3.
\end{equation} 
For $\alpha \leq 1$, we have
\begin{equation}\label{eq:lowerboundM}
    \frac{M}{2} - \frac{\alpha L_H}{3} \geq \frac{M}{2} - \frac{L_H}{3}, 
\end{equation} which by using our assumption $M \geq \frac{2}{3} L_H$ it is implied that $f(\mathbf{y}_{\alpha}) \leq f(\mathbf{x})$ in~(\ref{eq:faf}). However, $f(\mathbf{y}_{\alpha}) \leq f(\mathbf{x})$ contradicts~(\ref{eq:contradictionIneq}) for $M\geq \frac{2}{3}L_H$, which in turn leads to~(\ref{eq:TML}), and the proof is complete.
\end{proof}
\label{lemma:lemma2Nesterov2006}
\end{restatable}

\begin{restatable}{lem}{lemma3Nesterov2006} 
If $T_M (\mathbf{x}) \in \pazocal{F}$ then
\begin{equation}\label{eq:lemma3Nesterov2006}
    \norm{\nabla f(T_M(\mathbf{x}))}_2 \leq \frac{L_H + M}{2} r_M^2 (\mathbf{x}).
\end{equation}
\begin{proof}
Setting $\mathbf{y} = T_M (\mathbf{x})$ in~(\ref{eq:lemma1Nesterov20061}) and~(\ref{eq:nablaybarf}) we get 
\allowdisplaybreaks
\begin{equation}\label{eq:lemma1Nesterov2006111}
    \norm{\nabla f(T_M (\mathbf{x})) - \nabla f(\mathbf{x}) - \Diag(\nabla^2 f(\mathbf{x})) (T_M (\mathbf{x})-\mathbf{x})}_2   \leq  \frac{L_H}{2} \norm{T_M (\mathbf{x})-\mathbf{x}}_2
\end{equation} 
and 
\allowdisplaybreaks
\begin{equation}\label{eq:lemma1Nesterov200611}
    \norm{\nabla f(\mathbf{x}) + \Diag (\nabla^2 f(\mathbf{x})) (T_M (\mathbf{x})-\mathbf{x})}_2 =  \frac{M}{2} \norm{T_M (\mathbf{x})-\mathbf{x}}_2^2 = \frac{M}{2} r_M^2(\mathbf{x}),
\end{equation} 
respectively. Then, combining~(\ref{eq:lemma1Nesterov2006111}) with~(\ref{eq:lemma1Nesterov200611}) and the definition of the reverse triangle inequality, 
we arrive at~(\ref{eq:lemma3Nesterov2006}) and the proof is complete.
\end{proof}
\label{lemma:lemma3Nesterov2006}
\end{restatable}

\begin{restatable}{lem}{lemma4Nesterov2006} 
For any $\mathbf{x} \in \pazocal{F}$ we have 
\begin{equation}\label{eq:ffhatTMIneq1}
    \bar{f}_M (\mathbf{x}) \leq \min_{\mathbf{y}} \left(\frac{M+L_H}{6}\norm{\mathbf{y} - \mathbf{x}}_2^3 + f(\mathbf{y})\right)
\end{equation} and
\begin{equation}\label{eq:lemma2Nesterov1}
    f(\mathbf{x}) -  \bar{f}_M (\mathbf{x}) \geq \frac{M}{12} r_M^3(\mathbf{x}).
\end{equation} Moreover, if $M \geq L_H$, then $T_M(\mathbf{x}) \in \pazocal{F}$ and
\begin{equation}\label{eq:lemma2Nesterov2}
    f(T_M(\mathbf{x})) \leq \bar{f}_M (\mathbf{x}).
\end{equation}
\begin{proof}
In the following, we have used the relaxed condition $M \geq L_H$ as $M \geq L_H > \frac{2}{3}L_H$. Note that the relaxed condition $M \geq L_H$ also satisfies Lemma~\ref{lemma:lemma2Nesterov2006}. From the lower and upper bound of~(\ref{eq:lemma1Nesterov20062}) we have
\begin{equation}
    \hat{f} (\mathbf{x}, \mathbf{y}) \leq \frac{M+L_H}{6}\norm{\mathbf{y} - \mathbf{x}}_2^3 + f(\mathbf{y})
\end{equation} 
and
\begin{equation}\label{eq:ffhatineq}
    f(\mathbf{y}) \leq \hat{f} (\mathbf{x}, \mathbf{y}),
\end{equation} 
respectively. Thus, we have
\begin{equation}\label{eq:fbarfIneq}
    f(\mathbf{y}) \leq \hat{f} (\mathbf{x}, \mathbf{y}) \leq \frac{M+L_H}{6}\norm{\mathbf{y} - \mathbf{x}}_2^3 + f(\mathbf{y}).
\end{equation} 
Minimizing~(\ref{eq:fbarfIneq}) all sides w.r.t. $\mathbf{y}$ yields
\allowdisplaybreaks
\begin{equation}
    \min_{\mathbf{y}} f(\mathbf{y}) \leq \min_{\mathbf{y}} \hat{f} (\mathbf{x}, \mathbf{y}) \leq  \min_{\mathbf{y}} \left(\frac{M+L_H}{6}\norm{\mathbf{y} - \mathbf{x}}_2^3 + f(\mathbf{y})\right).
\end{equation} 
which in turn using $\mathbf{y} = T_{M} (\mathbf{x})$ yields
\begin{equation}\label{eq:ffhatTMIneq}
    f(T_{M}(\mathbf{x})) \leq \hat{f} (\mathbf{x}, T_{M} (\mathbf{x})) \leq  \min_{\mathbf{y}} \left(\frac{M+L_H}{6}\norm{\mathbf{y} - \mathbf{x}}_2^3 + f(\mathbf{y})\right).
\end{equation} Additionally, using~(\ref{eq:auxhatTM}) in~(\ref{eq:ffhatTMIneq}) we obtain~(\ref{eq:ffhatTMIneq1}). Note that~(\ref{eq:ffhatTMIneq1}) aligns with the results presented in~\citet[Lemma 4]{nesterov2006}, with the distinction that $\Diag (\nabla^2 f(\mathbf{x}))$ is used in place of $\nabla^2 f(\mathbf{x})$. From the LHS of~(\ref{eq:ffhatTMIneq}) and~(\ref{eq:auxhatfDef}) we obtain
\allowdisplaybreaks
\begin{multline}\label{eq:ffTMineq}
    f(\mathbf{x}) - f(T_M(\mathbf{x})) \geq f(\mathbf{x}) - \hat{f}(\mathbf{x}, T_M(\mathbf{x}))\\ =  - \nabla f(\mathbf{x})^T (T_M(\mathbf{x})-\mathbf{x}) - \frac{1}{2} (T_M(\mathbf{x})-\mathbf{x})^T \Diag (\nabla^2 f(\mathbf{x})) (T_M(\mathbf{x})-\mathbf{x})  - \frac{M}{6} \norm{T_M(\mathbf{x})-\mathbf{x}}_2^3.
\end{multline} In addition, from~(\ref{eq:nablaybarf1}) we have
\allowdisplaybreaks
\begin{equation}\label{eq:nablaybarf3}
   -\frac{1}{2}(T_M(\mathbf{x})-\mathbf{x})^T\Diag(\nabla^2 f(\mathbf{x})) (T_M(\mathbf{x})-\mathbf{x}) =  \frac{1}{2}\nabla f(\mathbf{x})^T (T_M(\mathbf{x})-\mathbf{x}) + \frac{M}{4}\norm{T_M(\mathbf{x})-\mathbf{x}}_2^3,
\end{equation} which combined with~(\ref{eq:ffTMineq}) and~(\ref{eq:auxhatTM}) gives
\begin{equation}
    f(\mathbf{x}) - \bar{f}_M (\mathbf{x}) \geq -\frac{1}{2} \nabla f(\mathbf{x})^T (T_M(\mathbf{x})-\mathbf{x}) + \frac{M}{12} r_M^3(\mathbf{x}), 
\end{equation} which in turn combined with~(\ref{eq:lemma2Nesterov}) yields~(\ref{eq:lemma2Nesterov1}). To conclude the proof, setting $\mathbf{y} = T_M(\mathbf{x})$ in the LHS of~(\ref{eq:fbarfIneq}) and using~(\ref{eq:auxhatTM}) we obtain~(\ref{eq:lemma2Nesterov2}).
\end{proof}
\label{lemma:lemma4Nesterov2006}
\end{restatable}

\begin{restatable}{lem}{localOptimalityCriterion} 
If $\mathbf{x} \in \pazocal{F}$ then
\allowdisplaybreaks
\begin{equation}\label{eq:localOptimalityCriterion}
    \mu_{M_i} (\mathbf{x}_{i+1}) \stackrel{\Delta}{=} \max \Biggl\{ \sqrt{\frac{2}{L_H + M_i}\norm{\nabla f(\mathbf{x}_{i+1})}_2},  -\frac{2}{2L_H + M_i} \lambda_{\text{min}} (\Diag(\nabla^2 f(\mathbf{x}_{i+1}))) \Biggr\}.
\end{equation}
\begin{proof}
From~(\ref{eq:lemma3Nesterov2006}) we have
\begin{equation}\label{eq:lemma3Nesterov2006Iter}
    \sqrt{\frac{2}{L_H + M}  \norm{\nabla f(\mathbf{x}_{i+1})}_2} \leq \norm{\mathbf{x}_{i+1}- \mathbf{x}_{i}}_2.
\end{equation} 
(\ref{eq:globalSolDiag1}) implies
\allowdisplaybreaks
\begin{multline} \label{eq:globalSolDiag11}
    \Diag(\nabla^2 f(\mathbf{x})) -  L_H \norm{\mathbf{x} - \mathbf{y}}_2 \: \mathbf{I} + \frac{M}{2} \norm{\mathbf{x} - \mathbf{y}}_2\: \mathbf{I} \succeq  -  L_H \norm{\mathbf{x} - \mathbf{y}}_2\: \mathbf{I} \Leftrightarrow \\ 
    \lambda_{\text{min}} \left(\Diag(\nabla^2 f(\mathbf{x}))\right) -  L_H \norm{\mathbf{x} - \mathbf{y}}_2   \geq  - \left(\frac{M}{2}+  L_H\right) \norm{\mathbf{x} - \mathbf{y}}_2,
\end{multline} 
which by setting $\mathbf{y} = \mathbf{x}_{t+1}$ and $\mathbf{x} = \mathbf{x}_{t}$ yields
\allowdisplaybreaks
\begin{equation}\label{eq:globalSolDiag12}
    \lambda_{\text{min}} \left(\Diag(\nabla^2 f(\mathbf{x}_{t}))\right) -  L_H \norm{\mathbf{x}_t - \mathbf{x}_{t+1}}_2^2   \geq  - \left(\frac{M}{2}+  L_H\right) \norm{\mathbf{x}_t - \mathbf{x}_{t+1}}_2.
\end{equation} 
Combining~\citet[Corrolary 1.2.3]{nesterov2018} with Lemma~\ref{lemma:diagHessian} for $\mathbf{y} = \mathbf{x}_{t+1}$ and $\mathbf{x} = \mathbf{x}_{t}$ yields
\allowdisplaybreaks
\begin{equation}\label{eq:globalSolDiag13}
    \lambda_{\text{min}} \left(\Diag(\nabla^2 f(\mathbf{x}_{t+1}))\right) \geq  \lambda_{\text{min}} \left(\Diag(\nabla^2 f(\mathbf{x}_{t}))\right) -  L_H \norm{\mathbf{x}_t - \mathbf{x}_{t+1}}_2,
\end{equation} 
which when combined with~(\ref{eq:globalSolDiag12}) gives
\begin{equation}\label{eq:globalSolDiag14}
    \lambda_{\text{min}} \left(\Diag(\nabla^2 f(\mathbf{x}_{t+1}))\right)   \geq  - \left(\frac{M}{2}+  L_H\right) \norm{\mathbf{x}_t - \mathbf{x}_{t+1}}_2,
\end{equation} 
and, in turn, implies
\allowdisplaybreaks
\begin{equation}\label{eq:globalSolDiag15}
    - \frac{2}{2 L_H + M} \lambda_{\text{min}} \left(\Diag(\nabla^2 f(\mathbf{x}_{t+1}))\right)   \leq    \norm{\mathbf{x}_t - \mathbf{x}_{t+1}}_2,
\end{equation} 
In order to obtain an $(\epsilon_g, \epsilon_H)$-stationary point, we need $\norm{\mathbf{x}_t - \mathbf{x}_{t+1}}_2 \leq \epsilon_g$ and $\norm{\mathbf{x}_t - \mathbf{x}_{t+1}}_2 \leq \epsilon_H$ in~(\ref{eq:lemma3Nesterov2006Iter}) and~(\ref{eq:globalSolDiag15}), respectively, which yields~(\ref{eq:localOptimalityCriterion}), and the proof is complete. 
\end{proof}
\label{lemma:localOptimalityCriterion}
\end{restatable}

\begin{restatable}{lem}{lemma5Nesterov2006} 
For any $\mathbf{x} \in \pazocal{F}$ we have 
\begin{equation}\label{eq:mumleqrm}
    \mu_M (T_M(\mathbf{x})) \leq r_M (\mathbf{x}).
\end{equation}
\begin{proof} 
Adapting \citet[Corollary 1.2.2]{nesterov2018} in our context yields
\begin{equation}\label{eq:mumleqrm1}
    \Diag(\nabla^2 f(T_M(\mathbf{x}))) \succeq  \Diag(\nabla^2 f(\mathbf{x})) -  r_M(\mathbf{x}) L_H \mathbf{I}.
\end{equation} 
Combining~(\ref{eq:mumleqrm1}) and~(\ref{eq:globalSolDiag1}) gives
\begin{equation}
    \Diag(\nabla^2 f(T_M(\mathbf{x}))) \succeq  -\left(\frac{1}{2}M + L_H \right) r_M(\mathbf{x})\mathbf{I},
\end{equation} 
which when combined with Lemma~\ref{lemma:localOptimalityCriterion} yields~(\ref{eq:mumleqrm}) and the proof is complete.
\end{proof} 
\label{lemma:lemma5Nesterov2006}
\end{restatable}

\subsection{Proof of Theorem~\ref{thm:diagmodelmin}}\label{supplement:diagmodelmin}
Let $(\mathbf{s}_{i+1}, \nu_{i+1})$ be the output of Algorithm~\ref{alg:newton} for $\mathbf{B}= \Diag(\nabla^2 f (\mathbf{x}_i))$, $\mathbf{g}=\nabla f (\mathbf{x}_i)$. Recall that $(\mathbf{s}_{i+1}, \nu_{i+1})$ is a minimizer of~(\ref{prob:cubicOursEquiv}) and according to Theorem~\ref{thm:probequiv} it is also a minimizer of problem~(\ref{prob:cubicNesterov2}) where $M=\nu_{i+1}$. Recall also that $\mathbf{x}_{i+1} = \mathbf{x}_{i} + \mathbf{s}_{i+1}$ and let the sequence $\{ \mathbf{x}_i \}_{i \geq 1}$ be generated by Algorithm~\ref{alg:mainAlg}.

Next, suppose that Assumption~\ref{assum:continuityAssum} holds, i.e., the objective function $f(\mathbf{x})$ is bounded from bellow, $f(\mathbf{x}) \geq f^{\text{low}}$ for all $\mathbf{x}\in \pazocal{F}$. Then, we continue with the proof of the main result in Theorem~\ref{thm:diagmodelmin}. From~(\ref{eq:lemma2Nesterov1}), we have
\begingroup
\allowdisplaybreaks
\begin{equation}\label{eq:lemma4Nesterov}
\begin{gathered}
f(\mathbf{x}_0) - \bar{f}_{M_0} (\mathbf{x}_0) \geq \frac{M_0}{12} r_{M_0}^3 (\mathbf{x}_0) \\
f(\mathbf{x}_1) - \bar{f}_{M_1} (\mathbf{x}_1) \geq \frac{M_1}{12} r_{M_1}^3 (\mathbf{x}_1) \\
  \vdots  \\
f(\mathbf{x}_{k-1}) - \bar{f}_{M_{k-1}} (\mathbf{x}_{k-1}) \geq  \frac{M_{k-1}}{12} r_{M_{k-1}}^3 (\mathbf{x}_{k-1}), 
\end{gathered}
\end{equation}
\endgroup
where $r_{M_{i}} (\mathbf{x}_i) = \norm{\mathbf{x}_i - \mathbf{x}_{i+1}}_2$ and
\allowdisplaybreaks
\begin{equation}\label{prob:cubicNesterov3}
 \bar{f}_{M_i} (\mathbf{x}) \stackrel{\Delta}{=}  \underset{\mathbf{s} \in \mathbb{R}^d}{\min}~m_{M_i}(\mathbf{s}).
\end{equation} 
Summing over~(\ref{eq:lemma4Nesterov}) we get
\begin{equation}
\sum_{i=0}^{k-1} \left(f(\mathbf{x}_i) -  \bar{f}_{M_i} (\mathbf{x}_i) \right) \geq \sum_{i=0}^{k-1} \frac{M_i}{12} r^3_{M_i} (\mathbf{x}_i).
\end{equation} 
Then applying $f(\mathbf{x}_{i+1}) \leq \bar{f}_{M_i} (\mathbf{x}_i)$ (Lemma~\ref{lemma:lemma4Nesterov2006}), we get
\allowdisplaybreaks
\begin{equation}
\sum_{i=0}^{k-1} \left(f(\mathbf{x}_i) -  f(\mathbf{x}_{i+1}) \right) \geq \sum_{i=0}^{k-1} \frac{M_i}{12} r^3_{M_i} (\mathbf{x}_i),
\end{equation} 
which, by applying the telescoping sum, yields
\begin{equation}
    f(\mathbf{x}_0) -  f(\mathbf{x}_{k}) \geq \sum_{i=0}^{k-1} \frac{M_i}{12} r^3_{M_i} (\mathbf{x}_i).
\end{equation} 
Next, using the lower bound of $f(\mathbf{x}_k)$, i.e., $f^{\text{low}}$ by Assumption~\ref{assum:continuityAssum} and Remark~\ref{rmrk:continuityAssum1}, we obtain
\begin{equation}\label{eq:sumBound}
    f(\mathbf{x}_0)- f^{\text{low}} \geq \sum_{i=0}^{k-1} \frac{M_i}{12} r^3_{M_i} (\mathbf{x}_i),
\end{equation} which implies
\allowdisplaybreaks
\begin{equation}\label{eq:sumBound1}
    f(\mathbf{x}_0)- f^{\text{low}} \geq k \frac{L_0}{12} r^3_{L_0} (\mathbf{x}_i) \Leftrightarrow  
\mu_{M_i}(\mathbf{x}_{i+1}) \leq r_{M_i} (\mathbf{x}_i) \leq 12^{1/3} \left(\frac{f(\mathbf{x}_0)- f^{\text{low}}}{k\:M_i}\right)^{1/3},
\end{equation} where Lemma~\ref{lemma:lemma5Nesterov2006} (Appendix~\ref{supplement:diagmodelminRelatedLemmata}) is applied to get $\mu_{M_i}(\mathbf{x}_{i+1}) \leq r_{M_i} (\mathbf{x}_i)$. For $ L_H = M_i$ and applying the trick $12=3\cdot 4 \Leftrightarrow 12^{1/3}=3^{1/3} \cdot (8/2)^{1/3} \Leftrightarrow 12^{1/3} = (3/2)^{1/3}\cdot 8^{1/3} \Leftrightarrow 12^{1/3} = (3/2)^{1/3} \cdot2 \Leftrightarrow 12^{1/3} = (3/2)^{1/3}\cdot 8 / 4 \leq (3/2)^{1/3} \cdot 8/3$, we arrive at  
\begin{equation}\label{eq:sumBound4}
    \mu_{L_H} (\mathbf{x}_{i+1}) \leq \frac{8}{3} \left( \frac{3}{2}\frac{f(\mathbf{x}_0) - f^{\text{low}}}{L_H k \: } \right)^{1/3}.
\end{equation} As in~\citet[Theorem 3]{nesterov2006}, here it is assumed that $\nabla^2 f (\mathbf{x}_i)$ is positive definite for some $i \geq 0$. The latter assumption implies that $\Diag(\nabla^2 f (\mathbf{x}_i))$ is also positive definite. Then for some $i \geq 0$, from~(\ref{eq:localOptimalityCriterion}), we restrict our study to
\begin{equation}
    \mu_{ L_H} (\mathbf{x}_{i+1}) = \sqrt{\frac{1}{L_H}\norm{\nabla f(\mathbf{x}_{i+1})}_2},
\end{equation} 
which combined with~(\ref{eq:sumBound4}), yields
\begin{equation}\label{eq:gradNesterovupperBound}
    \min_{0\leq i \leq k-1} \norm{\nabla f(\mathbf{x}_{i+1})}_2 \leq  L_H^{1/3} \left(\frac{8}{3}\right)^2 \left( \frac{3}{2} \frac{f(\mathbf{x}_0) - f^{\text{low}}}{k \: } \right)^{2/3},
\end{equation} 
which implies~(\ref{eq:normgradfUpperExactDiag}). The convergence rate in~(\ref{eq:normgradfUpperExactDiag}) is used to establish the local convergence rate when the approximate $\mathbf{g}_i$ and $\mathbf{B}_i$ are used instead of $\nabla f(\mathbf{x}_i)$ and $\nabla^2 f(\mathbf{x}_i)$, respectively. The latter argument is strengthened by Corollary~\ref{cor:gradienthessianSampling}, and the proof is complete.

\subsection{Vector and Matrix Bernstein Inequalities}\label{supplement:vectormatrixBernstein}
For completeness, we restate \citet[Lemma 18]{kohler2017}, incorporating corrections for minor typographical errors. Lemma~\ref{lemma:vectorBernstein} is utilized by Lemma~\ref{lemma:gradientDeviationBound}. Next, Lemma~\ref{lemma:basedDefOnTwoIndRandVars} is introduced and utilized by Lemma~\ref{lemma:matrixBernsteinTwoSources}, which in turn is utilized by Lemma~\ref{lemma:hessianDeviationBound}. Lemma~\ref{lem:5.4.10.alt} is also used by Lemma~\ref{lemma:matrixBernsteinTwoSources}.
\begin{restatable}[Vector Bernstein Inequality]{lem}{vectorBernstein} 
Let $\mathbf{x}_1, \mathbf{x}_2, \dots \mathbf{x}_n$ be independent random vectors of common dimension $d$ and assume that each one is centered, uniformly bounded, and also the variance is bounded from above, i.e.,
\begin{equation}
    \mathbb{E}[\mathbf{x}_i] = \mathbf{0} \text{ and } \norm{\mathbf{x}_i}_2 \leq \vartheta \text{ as well as } \mathbb{E}[\norm{\mathbf{x}_i}_2^2] \leq \sigma^2.
\end{equation} 
Let $\mathbf{z} = \frac{1}{n}\sum_{i=1}^{n} \mathbf{x}_i$.
Then we have 
\begin{equation}\label{eq:vectorBernstein}
    \Pr (\norm{\mathbf{z}}_2 \geq \epsilon) \leq \exp \left(-n \frac{\epsilon^2}{8\sigma^2} + \frac{1}{4}\right),
\end{equation} 
with $0 <  \epsilon < \sigma^2 / \vartheta + \sigma$.
\label{lemma:vectorBernstein}
\begin{proof}
A proof can be found in~\citep[Lemma 18]{kohler2017}. However, some typographical errors were identified, leading us to reproduce the proof for clarity.

The Vector Bernstein inequality for independent, zero-mean random vectors ~\citet[Theorem 12]{gross2011} states
\begin{equation}\label{eq:vectorBernsteinGrosss}
    \Pr \left( \frac{1}{n}\norm{\sum_{i=1}^{n} \mathbf{x}_i}_2 \geq \frac{1}{n} (t + \sqrt{V}) \right) \leq \exp \left( -\frac{t^2}{4V}\right),
\end{equation} 
where $V = \sum_{i=1}^{n} \mathbb{E}[\norm{\mathbf{x}_i}_2^2]$ is the sum of the traces of the covariance matrices of the centered vectors $\mathbf{x}_i$. Using $\mathbb{E}[\norm{\mathbf{x}_i}_2^2] \leq \sigma^2$ yields $V \leq n \sigma^2$.

Note that in~(\ref{eq:vectorBernsteinGrosss}), the probability condition is scaled by a factor of $1/n$ to align with the subsequent analysis involving $\mathbf{z}$. Let $\epsilon = (t + \sqrt{V})/n \Leftrightarrow t = n \epsilon -\sqrt{V}$. Using~(\ref{eq:vectorBernsteinGrosss}) we get
\allowdisplaybreaks
\begin{equation}\label{eq:vectorBernsteinGrosss1}
    \Pr \left( \norm{\mathbf{z}}_2 \geq \epsilon \right) \leq \exp \left( -\frac{1}{4} \left(\frac{n\epsilon - \sqrt{V}}{\sqrt{V}} \right)^2\right) =  \exp \left( -\frac{1}{4} \left(\frac{n\epsilon}{\sqrt{V}} - 1 \right)^2\right).
\end{equation} 
We claim that
\allowdisplaybreaks
\begin{equation}\label{eq:vectorBernsteinGrosss2}
    \begin{aligned}
        & -\frac{1}{4}\left(\frac{n\epsilon}{\sqrt{V}}-1\right)^2 \leq -\frac{1}{4}\frac{n^2\epsilon^2}{2V} + \frac{1}{4}  
    \end{aligned}
\end{equation}
Indeed, if (\ref{eq:vectorBernsteinGrosss2}) holds we arrive at a valid inequality
\begin{equation}
 \begin{aligned}
         & -\frac{n^2\epsilon^2}{V} + 2 \frac{n\epsilon}{\sqrt{V}} - 1 \leq -\frac{1}{2}\frac{n^2\epsilon^2}{V} + 1 \\
        \Leftrightarrow & \left(\frac{n\epsilon}{\sqrt{2V}} - \sqrt{2}\right)^2 \geq 0.
    \end{aligned}
\end{equation} 
Using~(\ref{eq:vectorBernsteinGrosss2}) in~(\ref{eq:vectorBernsteinGrosss1}) gives
\begin{equation}\label{eq:vectorBernsteinGrosss3}
    \Pr \left( \norm{\mathbf{z}}_2 \geq \epsilon \right) \leq \exp \left(-n\frac{\epsilon^2}{8\sigma^2} + \frac{1}{4}\right),
\end{equation} 
where $V \leq n \sigma^2$ is used. 
According to~\citet[Theorem 12]{gross2011}, $t < V/\max_i \norm{\mathbf{x}_i}_2$. For $V \leq n \sigma^2$ and $\norm{\mathbf{x}_i}_2\leq \vartheta$ gives $t < n \sigma^2/ \vartheta$. Given $V \leq n \sigma^2$ and $t < n \sigma^2/\vartheta$, we arrive at
\begin{equation}\label{eq:vectorBernsteinGrosss4}
    n \epsilon = t + \sqrt{V} \leq \frac{n \sigma^2}{\vartheta} + \sqrt{n} \sigma \Leftrightarrow \epsilon \leq \frac{\sigma^2}{\vartheta} + \sigma,
\end{equation} 
where $\sqrt{x} < x$ with $x > 1$ is used. In addition, it can be shown that $\mathrm{Var} (\mathbf{z}) \leq \sigma^2/n$~\citet[Theorem 12]{gross2011} establishing~(\ref{eq:vectorBernstein}), which concluded the proof.
\end{proof}
\end{restatable} 

\begin{restatable}{lem}{basedDefOnTwoIndRandVars} 
Let \( \mathbf{u}_i: \Omega_{\mathbf{u}_i} \to \mathbb{R}^n \) and \( \mathbf{v}_j: \Omega_{\mathbf{v}_j} \to \mathbb{R}^m \) be independent random vectors for each \( i, j \). Let \( g: \mathbb{R}^n \times \mathbb{R}^m \to \mathbb{R}^{d \times d} \) be a function that produces random matrices. Then, for any indices \( (i, j) \neq (k, l) \), the matrices \( g(\mathbf{u}_i, \mathbf{v}_j) \) and \( g(\mathbf{u}_k, \mathbf{v}_l) \) are independent, regardless of whether \( \mathbf{u}_i \) and \( \mathbf{v}_j \) come from the same or different distributions.
\label{lemma:basedDefOnTwoIndRandVars}
\begin{proof}
Let two threshold matrices \( \mathbf{M} \) and \( \mathbf{M}' \) (which are symmetric \( d \times d \) matrices), and consider the probability
\begin{equation}\label{eq:prob1}
   P ( g(\mathbf{u}_i, \mathbf{v}_j) \ledot \mathbf{M} \cap g(\mathbf{u}_k, \mathbf{v}_l) \ledot \mathbf{M}' ),
\end{equation} 
using the element-wise comparison operator $\ledot$. Given that \( \mathbf{u}_i: \Omega_{\mathbf{u}_i} \to \mathbb{R}^n \) and \( \mathbf{v}_j: \Omega_{\mathbf{v}_j} \to \mathbb{R}^m \) are independent random vectors, and \( g: \mathbb{R}^n \times \mathbb{R}^m \to \mathbb{R}^{d \times d} \) is a function generating random matrices, we rewrite the event as
\begin{multline}\label{eq:equivA}
\{ g(\mathbf{u}_i, \mathbf{v}_j) \ledot \mathbf{M} \} 
\equiv  \{ (\omega_{\mathbf{u}_i}, \omega_{\mathbf{v}_j}) \in \Omega_{\mathbf{u}_i} \times \Omega_{\mathbf{v}_j} : g(\mathbf{u}_i(\omega_{\mathbf{u}_i}), \mathbf{v}_j(\omega_{\mathbf{v}_j})) \ledot \mathbf{M} \} \\ 
\equiv \{ (\mathbf{u}_i, \mathbf{v}_j) \in \mathbb{R}^{n} \times \mathbb{R}^{m} : g (\mathbf{u}_i, \mathbf{v}_j) \ledot \mathbf{M} \}.
\end{multline} 
Similarly, we have
\begin{multline} \label{eq:equivB}
\{ g(\mathbf{u}_k, \mathbf{v}_l) \ledot \mathbf{M}' \}  \equiv  \{ (\omega_{\mathbf{u}_k}, \omega_{\mathbf{v}_l}) \in \Omega_{\mathbf{u}_k} \times \Omega_{\mathbf{v}_l} : g(\mathbf{u}_k(\omega_{\mathbf{u}_k}), \mathbf{v}_l(\omega_{\mathbf{v}_l})) \ledot \mathbf{M}' \} \\
\equiv  \{ (\mathbf{u}_k, \mathbf{v}_l) \in \mathbb{R}^{n} \times \mathbb{R}^{m} : g (\mathbf{u}_k, \mathbf{v}_l) \ledot \mathbf{M}' \}.
\end{multline}
Let
\begin{equation}
    A = \{ (\mathbf{u}_i, \mathbf{v}_j) \in \mathbb{R}^{n} \times \mathbb{R}^{m} : g (\mathbf{u}_i, \mathbf{v}_j) \ledot \mathbf{M} \}
\end{equation} 
and
\begin{equation}
    B= \{ (\mathbf{u}_k, \mathbf{v}_l) \in \mathbb{R}^{n} \times \mathbb{R}^{m} : g (\mathbf{u}_k, \mathbf{v}_l) \ledot \mathbf{M}' \}.
\end{equation} 
Then using~(\ref{eq:equivA}) and~(\ref{eq:equivB}) in~(\ref{eq:prob1}), we write
\begin{equation}
    P( \{ g(\mathbf{u}_i, \mathbf{v}_j) \ledot  \mathbf{M} \} \cap \{ g(\mathbf{u}_k, \mathbf{v}_l) \ledot \mathbf{M}' \} ) =  P( (\mathbf{u}_i, \mathbf{v}_j) \in A \cap (\mathbf{u}_k, \mathbf{v}_l) \in B ).
\end{equation} 
Recall that the sequences \( \{ \mathbf{u}_i \} \) and \( \{ \mathbf{v}_j \} \) are independent families of random variables which implies that the pairs \( (\mathbf{u}_i, \mathbf{v}_j) \) are formed by drawing independently from these families. Thus, since \( \mathbf{u}_i \) and \( \mathbf{v}_j \) are independent for each \( (i, j) \), and \( (\mathbf{u}_k, \mathbf{v}_l) \) are also independent, we have
\begin{equation}
    P( \{ g(\mathbf{u}_i, \mathbf{v}_j) \ledot \mathbf{M} \} \cap \{ g(\mathbf{u}_k, \mathbf{v}_l) \ledot \mathbf{M}' \} ) =  P( (\mathbf{u}_i, \mathbf{v}_j) \in A ) P( (\mathbf{u}_k, \mathbf{v}_l) \in B ),
\end{equation} 
which implies
\begin{equation}
    P( \{ g(\mathbf{u}_i, \mathbf{v}_j) \ledot \mathbf{M} \} \cap \{ g(\mathbf{u}_k, \mathbf{v}_l) \ledot \mathbf{M}' \} ) =  P ( g(\mathbf{u}_i, \mathbf{v}_j) \ledot \mathbf{M} ) P (g(\mathbf{u}_k, \mathbf{v}_l) \ledot \mathbf{M}' ),
\end{equation} 
where~(\ref{eq:equivA}) and~(\ref{eq:equivB}) are used. Since the joint probability factorizes, this proves that \( g(\mathbf{u}_i, \mathbf{v}_j) \) and \( g(\mathbf{u}_k, \mathbf{v}_l) \) are independent whenever \( (i, j) \neq (k, l) \), regardless of whether \( \mathbf{u}_i \) and \( \mathbf{v}_j \) come from the same or different distributions.
\end{proof}
\end{restatable}

\begin{restatable}{lem}{vershynin2018high}\label{lem:5.4.10.alt}
Let \( \mathbf{X} \in \mathbb{R}^{d \times d} \) be a symmetric mean-zero matrix with $\norm{\mathbf{X}} \leq 1$ almost surely. Then,
\begin{equation}
    \mathbb{E} \left[ \exp(\lambda \mathbf{X} ) \right] \preceq \exp \left( g(\lambda) \mathbb{E} [\mathbf{X}^2] \right), 
\end{equation} 
where $g(\lambda) = e^{\lambda} - \lambda - 1$.
\begin{proof}
    We refer the reader to~\citep{vershynin2018high}.
\end{proof}
\end{restatable}

\begin{restatable}[Matrix Bernstein Inequality]{lem}{matrixBernsteinTwoSources} 
Let \( \mathbf{X}_{ij} \stackrel{\Delta}{=} g(\mathbf{u}_i, \mathbf{v}_j) \) be \( d \times d \) zero-mean random matrices with two independent sources of randomness, \( \mathbf{u}_i \) and \( \mathbf{v}_j \). Also, let \( \{ \mathbf{X}_{ij} \}_{i,j=1}^{N,M} \) be a set of independent random matrices of common dimension \( d \times d \), such that \( \norm{\mathbf{X}_{ij}}_2 \leq K \) almost surely for all \( i, j \). Then, for every \( t \geq 0 \), we have  
\allowdisplaybreaks
\begin{equation}
    \Pr \left( \norm{\sum_{i=1}^{N} \sum_{j=1}^{M} \mathbf{X}_{ij}}_2 \geq t \right) 
    \leq 2d \exp \left( -\frac{t^2 / 2}{\sigma^2 + Kt / 3} \right).  
\end{equation} Here, the matrix variance is given by  
\begin{equation}
    \sigma^2 = \norm{\sum_{i=1}^{N} \sum_{j=1}^{M} \mathbb{E} \left[ \mathbf{X}_{ij}^2 \right]}_2.
\end{equation} 
In particular, we can express this bound as a mixture of sub-Gaussian and sub-exponential tails, just like in the scalar Bernstein’s inequality:
\begin{equation}
    \Pr \left( \norm{\sum_{i=1}^{N} \sum_{j=1}^{M} \mathbf{X}_{ij} }_2 \geq t \right) 
    \leq 2d \exp \left( - \frac{3}{8} \min \left\{ \frac{t^2}{\sigma^2}, \frac{t}{K} \right\}\right).    
\end{equation}
\begin{proof} 
The following analysis is based on~\citep[Theorem 5.4.1]{vershynin2018high}.

\noindent \textbf{Reduction of MGF.} To bound the norm of the sum
\begin{equation}
    \mathbf{S} \stackrel{\Delta}{=} \sum_{i=1}^{N} \sum_{j=1}^{M}\mathbf{X}_{ij},    
\end{equation} 
we need to control the largest and smallest eigenvalues of \( \mathbf{S} \). We can do this separately. To put this formally, consider the largest eigenvalue
\begin{equation}
    \lambda_{\max}(\mathbf{S}) \stackrel{\Delta}{=} \max_i \lambda_i(\mathbf{S})    
\end{equation} 
and note that
\begin{equation}
    \norm{\mathbf{S}}_2 = \max |\lambda_i(\mathbf{S})| = \max \{\lambda_{\max}(\mathbf{S}), \lambda_{\max}(-\mathbf{S}) \}
\end{equation} 
and
\begin{equation}
    \Pr(|\lambda_{\max}(\mathbf{S})| \geq t) =  \Pr(\lambda_{\max}(\mathbf{S}) \geq t) +  \Pr(\lambda_{\max}(-\mathbf{S}) \geq t) -  \Pr(\lambda_{\max}(\mathbf{S}) \geq t \text{ and } \lambda_{\max}(-\mathbf{S}) \geq t),
\end{equation} 
which implies
\begin{equation}\label{eq:splitPrAbsMaxLambda}
   \Pr(|\lambda_{\max}(\mathbf{S})| \geq t) \leq  \Pr(\lambda_{\max}(\mathbf{S}) \geq t) + \Pr(\lambda_{\max}(-\mathbf{S}) \geq t).
\end{equation}

To bound \( \lambda_{\max}(\mathbf{S}) \), we proceed with computing the moment generating function. We fix \( \lambda \geq 0 \) and use Markov’s inequality to obtain
\begin{equation}\label{eq:UpperBoundmaxEigen}
    \Pr( \lambda_{\max} (\mathbf{S}) \geq t ) = \Pr( e^{\lambda \lambda_{\max} (\mathbf{S})} \geq e^{\lambda t} ) \leq e^{-\lambda t} \: \mathbb{E} [e^{\lambda \lambda_{\max} (\mathbf{S})} ].
\end{equation} 
Since by \citet[Definition 5.4.2]{vershynin2018high} the eigenvalues of \( e^{\lambda \mathbf{S}} \) are \( e^{\lambda \lambda_i(\mathbf{S})} \), we have
\begin{equation}
    {E} \stackrel{\Delta}{=} \mathbb{E} [ e^{\lambda  \lambda_{\max} (\mathbf{S})} ] = \mathbb{E} [ \lambda_{\max} (e^{\lambda \mathbf{S}}) ].
\end{equation} 
Since the eigenvalues of \( e^{\lambda \mathbf{S}} \) are all positive, the maximum eigenvalue of \( e^{\lambda \mathbf{S}} \) is bounded by the sum of all eigenvalues, the trace of \( e^{\lambda \mathbf{S}} \), which leads to
\begin{equation}
    {E} \leq \mathbb{E} [\Tr (e^{\lambda \mathbf{S}}) ].    
\end{equation}

\noindent \textbf{Application of Lieb's inequality.} First note that
\begin{equation}
\mathbf{S} = \sum_{i=1}^{N-1} \sum_{j=1}^{M-1} \mathbf{X}_{ij}  + \sum_{i=1}^{N-1} \mathbf{X}_{iM} + \sum_{j=1}^{M-1} \mathbf{X}_{Nj} + \mathbf{X}_{NM}.
\end{equation}

To prepare the application of Lieb’s inequality in~\citet[Lemma 5.4.9]{vershynin2018high}, let us separate the last term from the sum \( \mathbf{S} \)
\begin{equation}
    {E} \Biggl[ \leq \mathbb{E} \, \Tr \Biggl( \exp \Biggl(\sum_{i=1}^{N-1} \sum_{j=1}^{M-1} \lambda \mathbf{X}_{ij}  +  \sum_{i=1}^{N-1} \lambda \mathbf{X}_{iM} + \sum_{j=1}^{M-1} \lambda \mathbf{X}_{Nj} + \lambda \mathbf{X}_{NM} \Biggr)\Biggr)\Biggr] .
\end{equation} 
Conditioning on $\{\mathbf{X}_{ij}\}_{i,j=1}^{N-1,M-1}$ and applying~\citet[Lemma 5.4.9]{vershynin2018high} for the fixed matrix 
\begin{equation}
    \mathbf{H} \stackrel{\Delta}{=} \sum_{i=1}^{N-1} \sum_{j=1}^{M-1} \lambda \mathbf{X}_{ij}  + \sum_{i=1}^{N-1} \lambda \mathbf{X}_{iM} + \sum_{j=1}^{M-1} \lambda \mathbf{X}_{Nj}
\end{equation} 
and the random matrix \( \mathbf{Z} \stackrel{\Delta}{=} \lambda \mathbf{X}_{NM} \), we obtain
\begin{multline}
    {E} \leq  \mathbb{E}_{\{\mathbf{X}_{ij}\}_{i,j=1}^{N,M}}\Biggl[ \Tr \Biggl( \exp \Biggl( \sum_{i=1}^{N-1} \sum_{j=1}^{M-1} \lambda \mathbf{X}_{ij}  +  
    \sum_{i=1}^{N-1} \lambda \mathbf{X}_{iM} + \sum_{j=1}^{M-1} \lambda \mathbf{X}_{Nj}  + \lambda \mathbf{X}_{NM} \Biggr)\Biggr) \Biggr] \\
    \leq  \mathbb{E}_{\{\mathbf{X}_{ij}\}_{i,j=1}^{N-1,M-1}}\Biggl[ \mathbb{E}_{\mathbf{X}_{NM}}\Biggl[ \Tr \Biggl( \exp \Biggl( \sum_{i=1}^{N-1} \sum_{j=1}^{M-1} \lambda \mathbf{X}_{ij}  +  \sum_{i=1}^{N-1} \lambda \mathbf{X}_{iM} + \sum_{j=1}^{M-1} \lambda \mathbf{X}_{Nj}  + \lambda \mathbf{X}_{NM} \Biggr) \Biggr) \Biggr] \Biggr]  \\ 
    \leq  \mathbb{E}_{\{\mathbf{X}_{ij}\}_{i,j=1}^{N-1,M-1}} \Biggl[  \Tr\Biggl(\exp \Biggl(\sum_{i=1}^{N-1} \sum_{j=1}^{M-1} \lambda \mathbf{X}_{ij}  + \sum_{i=1}^{N-1} \lambda \mathbf{X}_{iM} +  \sum_{j=1}^{M-1} \lambda \mathbf{X}_{Nj} + \log \mathbb{E}_{\mathbf{X}_{NM}}  e^{\lambda \mathbf{X}_{NM}}\Biggr)\Biggr)\Biggr] .
\end{multline} 
We continue similarly: separate the next term \( \lambda \mathbf{X}_{N-1,M-1} \) from the remaining sum and apply~\citet[Lemma 5.4.9]{vershynin2018high} again for \( \mathbf{Z} = \lambda \mathbf{X}_{N-1, M-1} \). Repeating this process \( NM \) times, we obtain
\begin{equation}\label{eq:eUpperbound}
    \Pr( \lambda_{\max} (\mathbf{S}) \geq t ) \leq  \Tr \left(e^{-\lambda t} \exp \left( \sum_{i=1}^{N} \sum_{j=1}^{M} \log \mathbb{E} \, \exp{\lambda \mathbf{X}_{ij}} \right)\right).    
\end{equation}
\noindent \textbf{MGF of the individual terms.} 
It remains to bound the matrix-valued moment generating function  $\mathbb{E} \, e^{\lambda \mathbf{X}_{ij}}$ for each term \(\mathbf{X}_{ij}\). We now use Lemma~\ref{lem:5.4.10.alt}.

\noindent \textbf{Completion of the proof.} Using Lemma~\ref{lem:5.4.10.alt}, we obtain
\allowdisplaybreaks
\begin{multline}
\mathbb{E} \left[\exp \left(\lambda \mathbf{X}_{ij} /K \right)\right] \preceq  \exp \left(g(\lambda)\: \mathbb{E} [\mathbf{X}_{ij}^2]/K^2\right)  \Leftrightarrow  
\prod_{i,j=1}^{N,M} \mathbb{E}\left[ \exp \left(\lambda \mathbf{X}_{ij} /K \right)\right] \preceq   \prod_{i,j=1}^{N,M} \exp \left(g(\lambda)\: \mathbb{E} [\mathbf{X}_{ij}^2 ] /K^2\right) \\ \Leftrightarrow  
\prod_{i,j=1}^{N,M} \mathbb{E} \left[\exp \left(\lambda \mathbf{X}_{ij} /K \right) \right]  \preceq  \exp \left(g(\lambda)\: \sum_{i,j=1}^{N,M} \mathbb{E} [\mathbf{X}_{ij}^2 ] /K^2\right),
\end{multline} which implies
\begin{equation}\label{eq:lem5410Application}
   \prod_{i,j=1}^{N,M} \mathbb{E} \left[\exp \left(\lambda \mathbf{X}_{ij} /K\right)\right]  \preceq  \exp \left(g(\lambda)\: \sum_{i,j=1}^{N,M} \mathbb{E} [\mathbf{X}_{ij}^2] /K^2\right). 
\end{equation} 
Also, given that $\mathbf{X}_{ij}$ are independent, we have
\begin{equation}
    \prod_{i,j=1}^{N,M} \mathbb{E}\left[ \exp \left(\lambda \mathbf{X}_{ij} /K\right) \right]=   
    \exp \left(\log\left( \prod_{i,j=1}^{N,M} \mathbb{E} \exp \left(\lambda  \mathbf{X}_{ij} /K \right)\right)\right) = \\ 
    \exp \left( \sum_{i,j=1}^{N,M} \log \mathbb{E} \exp \left(\lambda   \mathbf{X}_{ij} /K\right)\right),
\end{equation} 
which combined with~(\ref{eq:lem5410Application})
\begin{equation}
    \exp \left( \sum_{i,j=1}^{N,M} \log \mathbb{E} \exp \left(\lambda   \mathbf{X}_{ij} /K \right)\right) \preceq  \exp \left(g(\lambda)\: \sum_{i,j=1}^{N,M} \mathbb{E} [\mathbf{X}_{ij}^2]/K^2\right),
\end{equation} 
and applying the trace to both sides yields
\begin{equation}
    \Tr \left(\exp \left( \sum_{i,j=1}^{N,M} \log \mathbb{E} \exp \left(\lambda   \mathbf{X}_{ij} /K\right)\right)\right)  \leq \Tr \left( \exp \left(g(\lambda)\: \tilde{\mathbf{Z}} / K^2\right)\right), 
\end{equation} 
where \( \tilde{\mathbf{Z}} \stackrel{\Delta}{=} \mathbb{E} \left[\sum_{i=1}^N\sum_{j=1}^M\mathbf{X}_{ij}^2\right] \). Since the trace of \( \exp(g(\lambda) \tilde{\mathbf{Z}} / K^2 )  \) is a sum of \( d \) positive eigenvalues, it is bounded by \( d \) times the maximum eigenvalue and using~\citet[Definition 5.4.2]{vershynin2018high}, we obtain
\begin{multline}\label{eq:eq:lem5410Application1}
    \Tr \left(\exp \left(g(\lambda)\: \tilde{\mathbf{Z}} / K^2 \right)\right) \leq d \: \lambda_{\max}\left( \exp \left( g(\lambda) \: \tilde{\mathbf{Z}} / K^2 \right)\right)     = d \: \exp \left( g(\lambda) \: \lambda_{\max}(\tilde{\mathbf{Z}} / K^2) \right)  
    \\  =  d \: \exp \left( g(\lambda) \: \| \tilde{\mathbf{Z}}  \|  / K^2 \right)  
     =  d \: \exp \left( g(\lambda) \: \sigma^2 / K^2\right).
\end{multline} 
Combining~(\ref{eq:eUpperbound}) and~(\ref{eq:eq:lem5410Application1}) we get 
\begin{equation}\label{eq:eUpperbound1}
    \Pr( \lambda_{\max} (\mathbf{S}) \geq K t ) \leq  e^{-\lambda t} \Tr \left(\exp \left( \sum_{i=1}^{N} \sum_{j=1}^{M} \log \mathbb{E} \, \exp{\lambda \mathbf{X}_{ij} / K} \right) \right) \leq   d \: \exp \left( -\lambda t + g(\lambda) \: \sigma^2 / K^2\right),   
\end{equation} 
which implies
\begin{equation}\label{eq:eUpperbound2}
    \Pr( \lambda_{\max} (\mathbf{S}) \geq  t )  \leq   d \: \exp \left( -\frac{\lambda }{K}t + \frac{g(\lambda)}{ K^2}\sigma^2\right).   
\end{equation} 
Minimizing over \( \lambda > 0 \), the minimum occurs at 
\begin{equation}
    \lambda = \log\left(1 + \frac{K t}{\sigma^2} \right), \quad t \geq 0.
\end{equation} 
Plugging this into the bound, we get
\begin{equation}
    \Pr \left( \lambda_{\max} \left( \mathbf{S} \right) \geq t \right)
\leq d \: \exp \left( - \frac{\sigma^2}{K^2} \: h\left( \frac{K t}{\sigma^2} \right) \right),
\end{equation} 
where 
\begin{equation}
    h(u) = (1 + u) \log(1 + u) - u, \quad \text{for } u > 0.    
\end{equation}
We know that~\citet[Exercise 2.8]{LugosiBook}
\begin{equation}
    h(u) \geq \frac{u^2}{2(1 + u/3)},    
\end{equation} 
with $u > 0$ and thus
\begin{equation}\label{eq:eUpperbound3}
    \Pr \left( \lambda_{\max} \left( \mathbf{S} \right) \geq t \right)
\leq d \: \exp \left( - \frac{\sigma^2}{K^2} \: \frac{u^2}{2(1 + u/3)} \right),    
\end{equation} 
where \( u = \frac{Kt}{\sigma^2} \). Substituting $u$ in~(\ref{eq:eUpperbound3}), we obtain
\begin{equation}
    \Pr \left( \lambda_{\max} \left( \mathbf{S} \right) \geq t \right)
\leq d \: \exp \left( - \frac{t^2 / 2}{\sigma^2 + K t/3} \right).    
\end{equation} 
Following similar steps with $-\mathbf{S}$ instead of $\mathbf{S}$ and using~(\ref{eq:splitPrAbsMaxLambda}), yields
\begin{equation}
    P(|\lambda_{\max} (\mathbf{S})| \geq t) \leq
    \begin{cases}
        2d \exp \left( \frac{-3t^2}{8\sigma^2} \right), \quad t \leq \sigma^2 / K \\
        2d \exp \left( \frac{-3t}{8K} \right), \quad t > \sigma^2 / K.
    \end{cases}
\end{equation} 
Intuitively, for small $t$, i.e., $t \leq \sigma^2 / K$, we have a sub-Gaussian bound, while for large $t$, i.e., $t > \sigma^2 / K$, we have a sub-exponential bound. Looking for the tightest bound, we may write 
\begin{equation}
    P(|\lambda_{\max} (\mathbf{S})| \geq t) \leq  2d \exp \left( - \frac{3}{8} \min \left\{ \frac{t^2}{\sigma^2}, \frac{t}{K} \right\}\right).
\end{equation}
\end{proof}
\label{lemma:matrixBernsteinTwoSources}
\end{restatable}

\end{document}